\documentclass[12pt]{article}
\usepackage[usenames,dvipsnames]{color}

\usepackage{graphicx}
\usepackage{amsthm}
\usepackage{graphics}
\usepackage{amsmath,,amssymb}
\usepackage{algorithmic}
\usepackage{algorithm}
\usepackage{caption}
\usepackage{subcaption}
\usepackage{pstricks, pst-node}
\usepackage{tikz}
\usepackage[linewidth=1pt]{mdframed}
\usepackage{multirow}
\usepackage[round,colon,authoryear]{natbib}
\usepackage{pgfplots}
\usepackage{hyperref}
\usepackage{bm}
\usepackage[title]{appendix}
\usepackage{enumitem}
\usepackage{makecell}

\pgfplotsset{compat=1.10}
\usepgfplotslibrary{fillbetween}
\usetikzlibrary{patterns}
\usetikzlibrary{shapes,arrows}
\usetikzlibrary{arrows.meta}

\newtheorem{theorem}{Theorem}
\newtheorem{proposition}{Proposition}
\newtheorem{lemma}{Lemma}
\newtheorem{corollary}{Corollary}

\newtheorem{assumption}{Assumption}

\def\EE{\mathbb{E}}
\def\PP{\mathbb{P}}
\def\RR{\mathbb{R}}

\def\Var{{\rm Var}}

\def\Ecal{\mathcal E}

\def\Bcal{\mathcal B}
\def\Acal{\mathcal A}

\def\Fcal{\mathcal F}

\def\Xcal{\mathcal X}

\def\Ycal{\mathcal Y}
\def\argmax{\mathop{\rm argmax}}

\def\bB{\mathbf B}

\def\bI{\mathbf I}

\def\argmin{\mathop{\rm argmin}}
\def\tr{\mathop{\rm tr}}

\def\bX{\mathbf X}
\def\bY{\mathbf Y}

\tikzstyle{block} = [rectangle, draw, fill=white!80!black, line width=2pt,
    text width=15em, text centered, rounded corners, minimum height=3em]
\tikzstyle{line} = [draw, -latex',line width=2pt]

\begin{document}
\title{Self-Supervised Metric Learning in Multi-View Data:\\ A Downstream Task Perspective}
\author{Shulei Wang\\ University of Illinois at Urbana-Champaign}
\date{(\today)}
%When and Why Does It Work?
\maketitle

\footnotetext[1]{
	Address for Correspondence: Department of Statistics, University of Illinois at Urbana-Champaign, 725 South Wright Street, 
	Champaign, IL 61820 (Email: shuleiw@illinois.edu).}

%\footnotetext[1]{
%Address for Correspondence: .}=

\begin{abstract}
Self-supervised metric learning has been a successful approach for learning a distance from an unlabeled dataset. The resulting distance is broadly useful for improving various distance-based downstream tasks, even when no information from downstream tasks is utilized in the metric learning stage. To gain insights into this approach, we develop a statistical framework to theoretically study how self-supervised metric learning can benefit downstream tasks in the context of multi-view data. Under this framework, we show that the target distance of metric learning satisfies several desired properties for the downstream tasks. On the other hand, our investigation suggests the target distance can be further improved by moderating each direction's weights. In addition, our analysis precisely characterizes the improvement by self-supervised metric learning on four commonly used downstream tasks: sample identification, two-sample testing, $k$-means clustering, and $k$-nearest neighbor classification. When the distance is estimated from an unlabeled dataset, we establish the upper bound on distance estimation's accuracy and the number of samples sufficient for downstream task improvement. Finally, numerical experiments are presented to support the theoretical results in the paper. 
\end{abstract}

%\noindent{\bf Keywords:} 

%\noindent{\bf AMS 2000 Subject Classification:} 

\newpage

%%%%%%%%%%%%%%%%%%%%%%%%%%%%%%%%%%%%%%%%%%%%%%%
\section{Introduction}
\label{sc:intro}
%%%%%%%%%%%%%%%%%%%%%%%%%%%%%%%%%%%%%%%%%%%%%%%

%%%%%%%%%%%%%%%%%%%%%%%%%%%%%%%%%%%%%%%%%%%%%%%
\subsection{Self-Supervised Metric Learning in Multi-View Data}
\label{sc:metriclearning}
%%%%%%%%%%%%%%%%%%%%%%%%%%%%%%%%%%%%%%%%%%%%%%%

Measuring distance is the first step to understand relationships between the data points and also one of the most key components in many distance-based statistics and machine learning methods, such as the $k$-means clustering algorithm and $k$-nearest neighbor method. The performance of these distance-based methods usually depends in large part on the choice of distance. Although various distances have been proposed to quantify the difference between data points in different applications, e.g., Euclidean distance, Wasserstein distance, and Manhattan distance, it is still unclear which distance the researcher should use to quantify the dissimilarity between the data for a given task at hand. One promising solution for such a problem is metric
learning, which has already been used in a wide range of applications, including face identification \citep{guillaumin2009you,liao2015person,li2014deepreid,yi2014deep}, remote sensing \citep{zhang2018coarse,ji2018earthquake} and neuroscience \citep{ktena2018metric,ma2019deep}.

Most metric learning methods require access to similar and dissimilar data pairs since they aim to preserve the closeness between similar data pairs and push dissimilar data points far from each other. A commonly-used strategy is to construct similar and dissimilar data pairs based on the labels' value in a supervised setting. For example, when the label is binary, the data points within the same class are regarded as similar ones, and those from different classes are dissimilar ones. Despite of the popularity in practice, such a strategy usually needs a large amount of labeled data, which can sometimes be expensive or difficult to collect. To overcome this challenge, a self-supervised learning framework is proposed to leverage the unlabeled data \citep{zhang2016colorful,oord2018representation,tian2019contrastive,chen2020simple}. The pseudo labels are generated from the unlabeled dataset itself, and then the statistics or machine learning model is trained by these pseudo labels. Specifically, when it comes to self-supervised metric learning, similar and dissimilar data pairs are constructed in an unsupervised fashion from the unlabeled dataset to train a better distance. 

It is generally difficult to distinguish similar and dissimilar data pairs from unstructured data as we usually do not have insights on which data points are closer than which. However, it can be much easier to construct similar pairs in an unsupervised way when there is some structure information in the dataset. In particular, multi-view data is a typical class of such datasets, where several different views from each sample are observed. More concretely, multi-view data refers to a dataset of $m$ samples, in which $n$ different views of each sample $(X_{i,1},\ldots, X_{i,n})\in \RR^{d\times n}$, $i=1,\ldots, m$, are recorded. Multi-view data is very common in real applications, for instance:
\begin{itemize}
	\item In face recognition, the images of the same face with different illumination or viewpoints are collected, such as the Extended Yale Face Database B \citep{georghiades2001few}.
	\item In the microbiome studies, the microbial samples of the same individual are usually collected at multiple time points \citep{gajer2012temporal,flores2014temporal}. 
	\item In robotics, the videos of the same scenario from multiple viewpoints are recorded \citep{sermanet2017time,dwibedi2018learning}. 
	\item Data augmentation is a popular technique to help increase the amount of data and generate extra views for each sample. For example, many different ways are used to synthesize imaging data, such as flipping, rotation, colorization, and cropping \citep{gidaris2018unsupervised,shorten2019survey}. By the data augmentation technique, a multi-view dataset can be generated from a single-view dataset. 
\end{itemize}
In these multi-view datasets, one can naturally label data points from two different views of the same sample, $X_{i,j}$ and $X_{i,j'}$ for some $j\ne j'$, as similar pair and data points from different samples, $X_{i,j}$ and $X_{i',j'}$ for some $i\ne i'$, as dissimilar pair. Therefore, it is a popular strategy to use multi-view data for self-supervised metric learning, which has been very successful in practice \citep{sohn2016improved,movshovitz2017no, sermanet2017time,duan2018deep,tian2019contrastive,roth2020revisiting, deng2021query}. 

Given the similar and dissimilar data pairs, a common principle of most existing metric learning methods is to look for a distance that can better predict whether a pair of data points is similar or not. If similar and dissimilar data pairs come from the multi-view data, it is equivalent to find a distance that can distinguish if a pair of data points comes from the same sample or not. To achieve this goal, different loss functions have been proposed to compare data pairs in metric learning \citep{xing2002distance,weinberger2009distance,kulis2012metric,bellet2013survey,bellet2015metric,musgrave2020metric}. Despite the difference in these loss functions, the ideal distance in metric learning methods aims to have a much larger value for dissimilar data pairs than similar ones. 

%%%%%%%%%%%%%%%%%%%%%%%%%%%%%%%%%%%%%%%%%%%%%%%
\subsection{Self-Supervised Metric Learning and Downstream Task}
%%%%%%%%%%%%%%%%%%%%%%%%%%%%%%%%%%%%%%%%%%%%%%%

Learning a distance from multi-view data is never the end of story, and the ultimate goal of self-supervised metric learning is to improve various downstream distance-based methods, be it $k$-means clustering algorithm or $k$-nearest neighbor method. In the supervised setting, where similarity is determined based on the actual labels, it is natural to believe that the resulting distance from metric learning can benefit the downstream tasks since similar and dissimilar data pairs are directly related to the labels in the downstream analysis \citep{weinberger2009distance}. On the other hand, different from the supervised setting, the self-supervised metric learning only has access to the fact whether two data points come from the same sample or not. At first sight, the self-supervised metric learning seems impossible to improve the performance of downstream distance-based methods since it does not utilize any label information. However there is considerable empirical evidence showing that self-supervised metric learning can indeed improve the efficiency of downstream analysis \citep{schroff2015facenet,sermanet2017time,tian2019contrastive}. These phenomena raise several natural questions: why does self-supervised metric learning benefit the downstream tasks? What kind of distance is a reasonable distance from an angle of downstream analysis? To what extent can the downstream tasks be improved by self-supervised metric learning? How much unlabeled multi-view data is sufficient to help improve the downstream tasks?

The theoretical properties of metric learning are mainly studied from the angle of generalization rates under a supervised setting in the literature \citep{jin2009regularized,bellet2015metric,cao2016generalization,jain2017learning,ye2019fast}. These results could help us understand how fast the empirical loss function converges but do not connect the resulting distance with downstream tasks. On the other hand, the self-supervised metric learning we study here is closely connected with self-supervised representation learning, which aims to find a transformation of the data that makes it easier to build an efficient classifier \citep{bengio2013representation,tschannen2019mutual}. Instead of distance, some recent works study how the representation learned from the data is helpful for the downstream tasks under a self-supervised setting \citep{arora2019theoretical,lee2020predicting,tian2020understanding,tosh2021contrastive,wei2020theoretical,tsai2020self}. Although these results provide theoretical insights of self-supervised representation learning, the analysis cannot be directly applied to the investigation of metric learning and the downstream distance-based task, such as $k$-means clustering algorithm and $k$-nearest neighbor method. Therefore, there is a clear need for a comprehensive theoretical study for self-supervised metric learning from a perspective of the downstream task. 

%%%%%%%%%%%%%%%%%%%%%%%%%%%%%%%%%%%%%%%%%%%%%%%
\subsection{A Downstream Task Perspective}
%%%%%%%%%%%%%%%%%%%%%%%%%%%%%%%%%%%%%%%%%%%%%%%

This paper's main goal is to understand how self-supervised metric learning works from the perspective of the downstream task. To demystify the effectiveness of self-supervised metric learning, we focus on learning a Mahalanobis distance, which has the form $D_M(X_1,X_2)=(X_1-X_2)^TM(X_1-X_2)$ for some positive semi-definite matrix $M$, and assume the multi-view data $(X_{i,1},\ldots, X_{i,n})$ is drawn from a latent factor model  
$$
X_{i,j}=BZ_i+\epsilon_{i,j},\qquad j=1,\ldots,n,\ i=1,\ldots, m
$$
where $Z_i\in \RR^K$ is $i$th sample's unobserved latent variable and $B=(b_1,\ldots,b_K)$ is the collection of factors such that $B^TB=\Lambda$, where $\Lambda={\rm diag}(\lambda_1,\ldots,\lambda_K)$ is a diagonal matrix. Here, $\epsilon_{i,j}$ is some view-specific random variable independent from $Z_i$. Under this latent factor model, the intrinsic structure of data lies in a $K$-dimensional subspace, where $K$ is usually much smaller than $d$. Our investigation shows that the target distances of metric learning under the latent factor model can be seen as the following distance
$$
D^\ast(X_1,X_2)=(X_1-X_2)^TBB^T(X_1-X_2). 
$$
Roughly speaking, the target distance $D^\ast$ measures the difference between data within the  $K$-dimensional subspace spanned by $b_1,\ldots,b_K$ and puts more weights in the directions that can better distinguish the similar and dissimilar data pairs. Thus, the distance can help reduce the data dimension, but is this distance a reasonable distance for downstream analysis?

The target distance $D^\ast$ seems only related to the latent factor model of multi-view data and has nothing to do with downstream tasks. However, our analysis shows that, perhaps surprisingly, $D^\ast$ has several desired properties for the downstream tasks if we further assume the latent variable includes all the label information in the downstream analysis, i.e.,
$$
Y_i\perp (X_{i,1},\ldots,X_{i,n})|Z_i,
$$
where $Y_i\in \{-1,1\}$ is the binary label in the downstream analysis. Here, no assumption is made for the relationship between label $Y$ and latent variable $Z$. Specifically, the distance $D^\ast$ has the following properties: 1) $D^\ast$ is a distance between a sufficient statistic for $Y$, so no information on the label is lost; 2) $D^\ast$ is robust to a collection of spurious features in data; 3) $D^\ast$ only keeps minimally sufficient information for $Y$. In a word, the distance that self-supervised metric learning aims for can help remove nuisance factors and keep necessary information even when no label is utilized. On the other hand, our further analysis suggests that the directions that can better capture the difference between the similar and dissimilar data pairs are not necessarily more useful in the downstream tasks than the one that cannot capture the difference very well. Motivated by this observation, we argue that target distance $D^\ast$ can be improved by an isotropic version of target distance, that is, we put equal weights in all directions 
$$
D^{\ast\ast}(X_1,X_2)=(X_1-X_2)^TB\Lambda^{-1}B^T(X_1-X_2).
$$
In particular, our results indicate that the distance $D^{\ast\ast}$ is a better choice than $D^\ast$ when the condition number of factor model is large where condition number is defined as $\kappa=\lambda_1/\lambda_K$. 

\begin{table}[h!]
	\centering
	\renewcommand{\arraystretch}{1.6}
	\begin{tabular}{cccc}
		\hline\hline
		Downstream Task & Measure & Euclidean Distance & Metric Learning \\ 
		\hline
		sample identification & detection radius & $\begin{aligned}{d^{1/4}\sigma\over \sqrt{\lambda}}\end{aligned}$ & $\begin{aligned} {K^{1/4}\sigma\over \sqrt{\lambda}}\end{aligned}$ \\
		two-sample test & detection radius  & $\begin{aligned}\left(\sqrt{K}\lambda+\sqrt{d}\sigma^2\over s\right)^{1/2}\end{aligned}$ & $\begin{aligned}\left(\sqrt{K}(\lambda+\sigma^2)\over s\right)^{1/2}\end{aligned}$ \\ 
		\multirow{2}{*}{$k$-means} & mis-cluster rate & $\begin{aligned}\exp\left(-{\|\mu\|^2\over 8(\lambda+\sigma^2)}\right) \end{aligned}$ & $\begin{aligned}\exp\left(-{\|\mu\|^2\over 8(\lambda+\sigma^2)}\right) \end{aligned}$  \\ 
		& required signal & $\begin{aligned}\left(1+{K\over s}\right)\lambda+\left(1+{d\over s}\right)\sigma^2\end{aligned}$& $\begin{aligned}\left(1+{K\over s}\right)(\lambda+\sigma^2)\end{aligned}$  \\ 
		$k$-nearest neighbor & excess risk & $\begin{aligned}s^{-\alpha(1+\beta)/(2\alpha+d)}\end{aligned}$ & $\begin{aligned}s^{-\alpha(1+\beta)/(2\alpha+K)}\end{aligned}$  \\ 
		\hline\hline
	\end{tabular}
	\caption{Performance comparisons between Euclidean distance and resulting distance from self-supervised metric learning. $d$ is the dimension of the data, $K$ is the number of factors, $s$ is the sample size in the downstream task, $\sigma^2$ measures the variation of different views, $\lambda$ measures the variation of sample difference, and $\mu$ is the expected difference between class.}
	\label{tb:downstream}
\end{table}

To further investigate the benefits of self-supervised metric learning, we compare the performance of Euclidean distance and target distances from metric learning, both $D^\ast$ and $D^{\ast\ast}$, on four commonly used distance-based methods: distance-based sample identification, distance-based two-sample testing, $k$-means clustering, and $k$-nearest neighbor ($k$-NN)  classification algorithm. The informal results are summarized in Table~\ref{tb:downstream} if we assume $\lambda=\lambda_1=\ldots=\lambda_K$ and the covariance matrix of $\epsilon_{i,j}$ is $\sigma^2I$. The formal results of a general setup, including both upper and lower bound, are discussed in Section~\ref{sc:downstream}. Table~\ref{tb:downstream} suggests that the performance of downstream tasks can be improved in different ways. In particular, the curse of dimensionality can be much alleviated by self-supervised metric learning as the performance only relies on the number of factors $K$ rather than the dimension of data $d$ when self-supervised metric learning is applied. For example, the nonparametric method $k$-NN behaves just like on a $K$-dimensional space as the target distance $D^\ast$ and $D^{\ast\ast}$ fits the geometry of the Bayes classification rule in a better way.

\begin{table}[h!]
	\centering
	\renewcommand{\arraystretch}{1.6}
	\begin{tabular}{cccc}
		\hline\hline
		Downstream Task & Distance & Accuracy $\Delta$ & Sample Size $m$ \\ 
		\hline
		\multirow{2}{*}{\makecell{two-sample test\\ $k$-means\\ sample identification}}  & $D^\ast$ & $o(\lambda)$ & $\begin{aligned} K+{d\sigma^2\over n\lambda}+{d\sigma^4\over n^2\lambda^2}\end{aligned}$ \\
		& $D^{\ast\ast}$  & $o(1)$ & $\begin{aligned}{d\sigma^2\over n\lambda}+{d\sigma^4\over n^2\lambda^2}\end{aligned}$ \\ 
		\hline
		\multirow{2}{*}{$k$-nearest neighbor} & $D^\ast$ & $\lambda s^{-1/(2\alpha+K)}$ & $\begin{aligned}s^{1/(2\alpha+K)}\left(K+{d\sigma^2\over n\lambda}+{d\sigma^4\over n^2\lambda^2}\right)\end{aligned}$  \\ 
		& $D^{\ast\ast}$ & $s^{-1/(2\alpha+K)}$ & $\begin{aligned}s^{1/(2\alpha+K)}\left({d\sigma^2\over n\lambda}+{d\sigma^4\over n^2\lambda^2}\right)\end{aligned}$  \\ 		
		\hline\hline
	\end{tabular}
	\caption{Distance estimation's accuracy and number of samples sufficient for downstream task improvement in self-supervised metric learning.}
	\label{tb:datadrivendistance}
\end{table}

In practice, we still need to estimate the target distances $D^\ast$ and $D^{\ast\ast}$ from the unlabeled multi-view data when they are unknown in advance. Our investigation shows that the estimated distances from self-supervised metric learning can also help improve above four distance-based methods provided the distance estimation is accurate enough. Specifically, if we quantify the distance estimation's accuracy by their largest discrepancy 
$$
\Delta(D,\hat{D})=\sup_{\|X_1-X_2\|\le 1}\left|D(X_1,X_2)-\hat{D}(X_1,X_2)\right|,
$$
the sufficient accuracy to achieve results in Table~\ref{tb:downstream} is summarized in Table~\ref{tb:datadrivendistance}. To estimate an accurate distance for downstream tasks, we consider a spectral metric learning method and study its theoretical properties in this paper. We show that the spectral method can help achieve minimax optimality in estimating target distances. Moreover, the analysis can help precisely characterize the number of samples $m$ sufficient for downstream tasks improvement, which is also summarized in Table~\ref{tb:datadrivendistance}. Table~\ref{tb:datadrivendistance} shows that it is easier to estimate $D^{\ast\ast}$ than $D^\ast$ from the unlabeled multi-view data.

The rest of the paper is organized as follows. We first introduce the multi-view model and discuss the main assumptions of the model in Section~\ref{sc:model}. Next, Section~\ref {sc:metric} studies the target distance of metric learning methods and its properties from a perspective of downstream analysis. In Section~\ref{sc:downstream}, the benefits of self-supervised learning are systematically investigated on several specific downstream distance-based tasks. Then, we study target distance estimation and characterize the sample complexity for downstream tasks improvement in Section~\ref{sc:spectral}. Finally, we analyze both the simulated and real data sets in Section~\ref{sc:numerical} to verify the theoretical results in this paper. All proofs are relegated to online Supplemental Materials.

%%%%%%%%%%%%%%%%%%%%%%%%%%%%%%%%%%%%%%%%%%%%%%%
\section{A Model for Multi-View Data}
\label{sc:model}
%%%%%%%%%%%%%%%%%%%%%%%%%%%%%%%%%%%%%%%%%%%%%%%

In this paper, we consider the following model of multi-view data for $m$ different samples
$$
(X_{i,1},\ldots,X_{i,n},Z_i,Y_i),\qquad i=1,\ldots m,
$$
where $n$ is the number of views observed for each sample.  We assume each $(Z_i,Y_i)$ is independently drawn from a distribution $\pi(Z,Y)$, where $Z\in \RR^K$ represents the sample's latent variable, and $Y$ is the label of interest. For simplicity, we always assume the label of interest is binary, i.e., $Y\in \{-1,1\}$. We also assume the conditional distribution of $Z$ given $Y$ is a continuous distribution, that is, the probability density function $\pi(Z|Y)$ exists. Given the latent variable $Z_i$, we assume the data of $n$ different views $X_{i,j}\in \RR^d$, $j=1,\ldots,n$, are independently drawn from a continuous conditional distribution $f(X|Z)$. In self-supervised metric learning, instead of observing the full data, we only observe the unlabeled multi-view data,  i.e.,
$$
(X_{i,1},\ldots,X_{i,n}),\qquad i=1,\ldots m.
$$
In the downstream analysis, depending on the task, we assume the observed data is a collection of single-view data with or without labels, i.e.,
$$
(X_1,Y_1),\ldots, (X_s,Y_s)\qquad {\rm or}\qquad X_1,\ldots, X_s.
$$
Here, $X_i$ refers to the single-view data in downstream analysis, and $X_{i,j}$ refers to the multi-view data in metric learning. We assume the data used in metric learning and downstream analysis are drawn from the same distribution, but different parts of the data are observed. In a typical self-supervised learning setting, we can expect the sample size in unlabeled multi-view data $m$ is much larger than the sample size in the downstream analysis $s$.

The latent variable $Z$ plays a vital role in the structure of multi-view data, characterizing the information shared by different views of the same sample. We assume $X_{i,j}$ connects with $Z_i$ through a factor model \citep{fan2020statistical}, i.e.,
\begin{equation}
	\label{eq:factormodel}
	X_{i,j}=\sum_{k=1}^K b_k Z_{i,k}+\epsilon_{i,j}
\end{equation}
where $\epsilon_{i,j}$ is a mean zero random variable independent from $Z_i$. $\epsilon_{i,j}$ are independent for different $i$ and $j$. If we write $B=(b_1,\ldots,b_K)$, we further assume 
$$
B^TB=\Lambda\qquad {\rm and}\qquad {\rm Var}(Z)=I_K,
$$
where $\Lambda={\rm diag}(\lambda_1,\ldots,\lambda_K)$ is a diagonal matrix with $\lambda_1\ge \ldots\ge \lambda_K$ and $I_K$ is an identity matrix. In addition, we assume $(I_d-B(B^TB)^{-1}B^T)\epsilon_{i,j}$ is independent from $B^T\epsilon_{i,j}$.  This latent factor model assumes that the intrinsic structure of data lies in a $K$-dimensional subspace. In the rest of the paper, we write $U=B\Lambda^{-1/2}$ as normalized projection matrix and $u_k=b_k/\sqrt{\lambda_k}$. Besides, we also assume the latent variable $Z$ includes all information about the sample which is invariant from different views, and thus
\begin{equation}
	\label{eq:condindep}
	Y_i\perp (X_{i,1},\ldots,X_{i,n})|Z_i .
\end{equation}
In other words, the observed multi-view data is connected with the label of interest only through the latent variable. 

%%%%%%%%%%%%%%%%%%%%%%%%%%%%%%%%%%%%%%%%%%%%%%%
\section{Self-Supervised Metric Learning}
\label{sc:metric}
%%%%%%%%%%%%%%%%%%%%%%%%%%%%%%%%%%%%%%%%%%%%%%%

%%%%%%%%%%%%%%%%%%%%%%%%%%%%%%%%%%%%%%%%%%%%%%%
\subsection{Metric Learning}
%%%%%%%%%%%%%%%%%%%%%%%%%%%%%%%%%%%%%%%%%%%%%%%

Given the multi-view data, metric learning aims to learn a distance $D$ that can help improve the downstream tasks. In particular, many different loss functions have been proposed to separate similar and dissimilar data pairs in the literature of metric learning \cite{kulis2012metric,musgrave2020metric}, including contrastive loss \citep{xing2002distance,chopra2005learning,hadsell2006dimensionality}, the triplet loss \citep{weinberger2009distance,chechik2010large, schroff2015facenet}, and $N$-pair loss\citep{sohn2016improved}. These loss functions have been widely used in various applications and lead to good performance in practice.

We now study how metric learning can extract information from the similar and dissimilar data pairs. The common goal of different metric learning methods is to find a distance that can distinguish dissimilar and similar data pairs. This goal can be naturally achieved by maximizing the following expected distance difference between dissimilar and similar data pairs  in multi-view data
$$
\EE\left(D(X_{i,j},X_{i',j'})-D(X_{i,j},X_{i,j'})\right),
$$
where $X_{i,j}$ and $X_{i',j'}$ are from different samples, and $X_{i,j}$ and $X_{i,j'}$ are different views of the same sample. If we are interested in learning a Mahalanobis distance, we can show that 
\begin{equation}
	\label{eq:metriclearning}
	M^\ast:=\argmax_{M\in \mathbb{S}_+^{d\times d}, \|M\|_F\le 1}\EE\left(D_M(X_{i,j},X_{i',j'})-D_M(X_{i,j},X_{i,j'})\right)=BB^T/\|BB^T\|_F,
\end{equation}
where $\mathbb{S}_+^{d\times d}$ is the collection of symmetric and positive semi-definite matrix and the Frobenius norm of a matrix $M$ is defined as $\|M\|_F=\sqrt{\sum_{i=1}^d\sigma^2_i(M)}$ where $\sigma_i(M)$ are the singular values of $M$. The main purpose of constraint for the Frobenius norm of $M$ is to avoid the scaling issue of Mahalanobis distance. For example, we always have $\EE\left(D_{cM}(X_{i,j},X_{i',j'})-D_{cM}(X_{i,j},X_{i,j'})\right)>\EE\left(D_M(X_{i,j},X_{i',j'})-D_M(X_{i,j},X_{i,j'})\right)$ for any constant $c>1$. When we observe infinite samples, the target Mahalanobis distance in above metric learning formulation is
$$
D^\ast(X_1,X_2)=(X_1-X_2)^TBB^T(X_1-X_2)=(X_1-X_2)^TU\Lambda U^T(X_1-X_2).
$$
Compared with the Euclidean distance, the target distance $D^\ast$ makes two main modifications: (i) $D^\ast$ measures the difference between data points in $K$ directions spanned by the column space of $B$; (ii) $D^\ast$ puts different weights in different directions. Is this distance $D^\ast$ a reasonable distance for the downstream analysis?

%%%%%%%%%%%%%%%%%%%%%%%%%%%%%%%%%%%%%%%%%%%%%%%
\subsection{Distance for Downstream Task}
%%%%%%%%%%%%%%%%%%%%%%%%%%%%%%%%%%%%%%%%%%%%%%%

The self-supervised metric learning aim to learn a distance $D^\ast$ by the unlabeled multi-view data. However, it is still unclear how the target distance $D^\ast$ is linked with the downstream tasks. In this section, we will see that the distance $D^\ast$ has several good properties desired for the downstream tasks, but may not honestly reflect the information needed for the downstream analysis. To see this, we need the following theorem.
\begin{theorem}
	\label{thm:idealdist}
	Suppose all the assumptions for multi-view data model in Section~\ref{sc:model} hold. Then there exists a function $g$ and a vector $\theta\in \RR^K$ with $\|\theta\|<2$ such that 
	$$
	{\pi(X|Y=1)\over \pi(X|Y=-1)}=g(U^TX)\qquad {\rm and}\qquad \EE(X|Y=1)-\EE(X|Y=-1)=B\theta,
	$$
	where $U$ is the normalized projection matrix in factor model and $\pi(X|Y)$ is the probability density function of $X$ given $Y$. Moreover, for any given $\theta\in \RR^K$ with $\|\theta\|<2$, there exists a joint distribution of $(X, Z,Y)$ satisfying assumptions in Section~\ref{sc:model} such that 
	$$
	\EE(X|Y=1)-\EE(X|Y=-1)=B\theta.
	$$
\end{theorem}
Theorem~\ref{thm:idealdist} shows that $D^\ast$ has the following good properties for downstream tasks:
\begin{itemize}
	\item In Theorem~\ref{thm:idealdist}, it is shown that $U^TX$ is a sufficient statistic for $Y$. Thus, from a prediction view, no information on $Y$ is lost when $D^\ast$ is used. This property is also a gold standard of many other problems, including approximate Bayesian computation \citep{fearnhead2012constructing}, representation learning \citep{cvitkovic2019minimal}, and dimension reduction \citep{adragni2009sufficient}.
	\item Theorem~\ref{thm:idealdist} suggests the mean difference between classes lies in the column space of $B$. If we write $U_\perp$ as an orthogonal matrix of $U$, then $U_\perp^TX$ is a collection of spurious features. $D^\ast$ is robust to these spurious features.	
	\item As suggested by the second part of Theorem~\ref{thm:idealdist}, all $u_k^TX$, $k=1,\ldots, K$, are potentially useful when we do not have access to $Y$ in the metric learning stage. In other words, the distance $D^\ast$ only keeps minimally sufficient information of $X$ for $Y$.  
\end{itemize}
In a word, the distance $D^\ast$ can keep all necessary information for $Y$ and remove nuisance factors from the data $X$, although label information is not utilized in the metric learning stage.

Unlike Euclidean distance, the target distance $D^\ast$ puts more weights in the directions that can reflect more difference between similar and dissimilar data pairs. More concretely, if we project the data to the direction $u_k$, the difference between similar and dissimilar data pairs is $\lambda_k$
$$
\lambda_k=\EE\left(\left[u_k^T(X_{i,j}-X_{i',j'})\right]^2-\left[u_k^T(X_{i,j}-X_{i,j'})\right]^2\right),\qquad k=1,\ldots, K.
$$
Along direction $u_k$, the average distance between dissimilar data pair is more significant than that between similar data pair when $\lambda_k$ is larger. So $u_k$ can better distinguish similar and dissimilar data pairs than $u_{k+1}$ as $\lambda_k\ge \lambda_{k+1}$. It seems reasonable to put more weights on $u_k$ over $u_{k+1}$ since it is usually believed that a feature that can better distinguish similar and dissimilar data pairs is more useful for the downstream analysis. However, the second part of Theorem~\ref{thm:idealdist} suggests that it is possible that $u_{k+1}$ is more useful than $u_{k}$ in the downstream analysis. For example, if we assume $Z|Y\sim N(\theta Y/2,I_K-\theta\theta^T/4)$ with $\theta$ such that $\theta_k=0$ but $\theta_{k+1}\ne 0$, then $u_k^TX|Y=1$ and $u_k^TX|Y=-1$ follow the same distribution while $u_{k+1}^TX|Y=1$ and $u_{k+1}^TX|Y=-1$ follow different ones. Motivated by this observation, we consider a moderated target distance 
$$
D^{\ast\ast}(X_1,X_2)=(X_1-X_2)^TU U^T(X_1-X_2),
$$
which puts equal weights in all directions $u_k$, $k=1,\ldots,K$. Similar to $D^{\ast}$, $D^{\ast\ast}$ also has the same good properties for the downstream tasks. As we can see in the next section, $D^{\ast\ast}$ is a better choice than $D^{\ast}$ when the conditional number $\kappa=\lambda_1/\lambda_K$ is large. 

%%%%%%%%%%%%%%%%%%%%%%%%%%%%%%%%%%%%%%%%%%%%%%%
\section{Target Distance on Specific Tasks}
\label{sc:downstream}
%%%%%%%%%%%%%%%%%%%%%%%%%%%%%%%%%%%%%%%%%%%%%%%

The ultimate goal of self-supervised metric learning is to improve various downstream distance-based statistical and machine learning methods. But it is still unclear to what extent the performance of the specific downstream task can be improved. In order to fill this gap, we investigate the benefits of self-supervised metric learning on some specific tasks when we observe infinite unlabeled multi-view samples, that is, $D^\ast$ and $D^{\ast\ast}$ are known. We consider four of the most commonly used distance-based methods: $k$-nearest neighbor classification algorithm, distance-based two-sample testing, $k$-means clustering (discussed in Supplemental Materials), and distance-based sample identification (discussed in Supplemental Materials). 

%%%%%%%%%%%%%%%%%%%%%%%%%%%%%%%%%%%%%%%%%%%%%%%
\subsection{$k$-Nearest Neighbor Classification}
\label{sc:knn}
%%%%%%%%%%%%%%%%%%%%%%%%%%%%%%%%%%%%%%%%%%%%%%%

Classification is the first problem we consider in this section. The observed data in classification includes the label of each sample, i.e., $(X_1,Y_1),\ldots, (X_s,Y_s)$. In classification, our goal is to build a decision rule $f: \RR^d\to \{-1,1\}$ to predict the label $Y$ for any given input of $X$. A long list of classification methods has been proposed to predict the labels. One of the most simple, intuitive, and efficient ones is probably the $k$-nearest neighbor ($k$-NN) classification method \citep{fix1985discriminatory,altman1992introduction,biau2015lectures}. Given the choice of  distance $D$ and a fixed point $x$,  $k$-NN is defined as following: $(X_{(1)},Y_{(1)}),\ldots, (X_{(s)},Y_{(s)})$ is a permutation of  $(X_1,Y_1),\ldots, (X_s,Y_s)$ such that 
$$
D(X_{(1)},x)\le \ldots \le D(X_{(s)},x),
$$
and then the decision rule of $k$-NN is the majority vote of its neighbors
$$
\hat{f}_D(x)=\begin{cases}
	1,\qquad & \sum_{i=1}^k \bI(Y_{(i)}=1)\ge k/2,\\
	-1,\qquad & {\rm otherwise}.
\end{cases}
$$
The $k$-NN classification rule is a plug-in estimator of the Bayes classification rule, which is given by 
$$
f^\ast(x)=\begin{cases}
	1,\qquad & \eta(x)\ge 1/2,\\
	-1,\qquad & {\rm otherwise},
\end{cases}
$$
where $\eta(x)=\PP(Y=1|X=x)$ is the regression function. The Bayesian rule is considered as the optimal decision rule since it minimizes misclassification error $R(f)=\PP(Y\ne f(X))$. To compare the performances of different distances on $k$-NN, we use the excess risk of misclassification error as the measure
$$
r(D)=\EE\left(\PP(Y\ne \hat{f}_D(X))\right)-\PP(Y\ne f^\ast(X)).
$$
Before characterizing the performance of $k$-NN, we can show that both the Bayes classification rule and the regression function can be written as a function of $U^TX$. A toy example of regression function is shown in Figure~\ref{fg:3dsurf} to illustrate the idea. The form of the regression function is closely connected to the multiple index model in statistical literature \citep{li1991sliced,lin2021optimality}.

\begin{figure}[h!]
	\centering
	\includegraphics[width=0.45\textwidth]{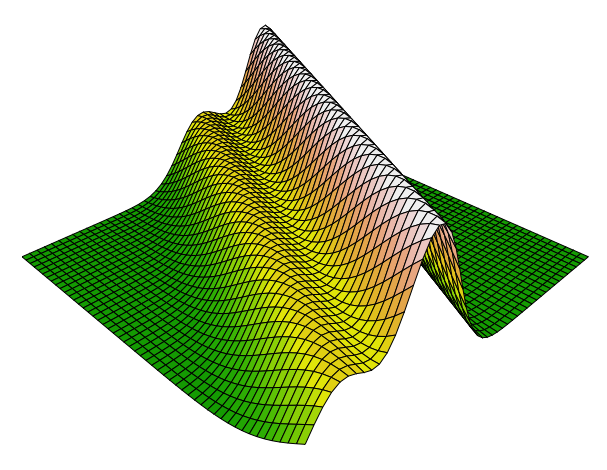}
	\caption{A toy example of regression function in two dimensional space. The regression function only changes along one direction.}\label{fg:3dsurf}
\end{figure}

\begin{proposition}
	\label{prop:bayesrule}
	If the assumptions in Section~\ref{sc:model} hold, there exists a function $\tilde{\eta}$ and $\tilde{f}^\ast$ such that 
	$$
	\eta(x)=\tilde{\eta}(U^Tx)\qquad {\rm and}\qquad f^\ast(x)=\tilde{f}^\ast(U^Tx).
	$$
\end{proposition}
We omit the proof of Proposition~\ref{prop:bayesrule} since it is an immediate result of Theorem~\ref{thm:idealdist}. Proposition~\ref{prop:bayesrule} suggests that we can make assumptions for $\tilde{\eta}$ and $\tilde{f}^\ast$ rather than $\eta$ and $f^\ast$. Specifically, we consider the following assumptions. 
\begin{assumption}
	\label{ap:knn}
	It holds that
	\begin{enumerate}[label=(\alph*)]
		\item $\tilde{\eta}(y)$ is $\alpha$-H\"older continuous, i.e., $|\tilde{\eta}(y)-\tilde{\eta}(y')|\le L\|y-y'\|^\alpha$, where $y,y'\in \RR^K$;
		\item the distribution of $X$ satisfies $\beta$-marginal assumption, i.e., $\PP(0<|\tilde{\eta}(U^TX)-1/2|\le t)\le C_0t^\beta$ for some constant $C_0$;
		\item the support of $X$ is a compact set and the probability density function $\mu(x)$ exists. The probability density function $\mu(x)$ is bounded away from 0 on the support of $X$, i.e., $\mu(x)\ge \mu_{\rm min}$ for some small constant $\mu_{\rm min}$.
	\end{enumerate}
\end{assumption}
These assumptions in Assumption~\ref{ap:knn} are commonly used conditions for analyzing nonparametric classification methods such as $k$-NN \citep{audibert2007fast,samworth2012optimal}. With these conditions, the following theorem characterizes the convergence rate of $k$-NN when different distances are used. 
\begin{theorem}
	\label{thm:knn}
	Suppose assumptions in Section~\ref{sc:model} and Assumption~\ref{ap:knn} hold. If we choose $k=cs^{2\alpha/(2\alpha+d)}$ for some constant $c$, then
	$$
	r(\|\cdot\|^2) \lesssim  s^{-\alpha(1+\beta)/(2\alpha+d)}.
	$$
	On the other hand, if $k=c(s/\kappa^{K-1})^{2\alpha/(2\alpha+K)}$ or $k=cs^{2\alpha/(2\alpha+K)}$ for some constant $c$, then
	$$
	r(D^\ast) \lesssim  (s/\kappa^{K-1})^{-\alpha(1+\beta)/(2\alpha+K)}\qquad {\rm and}\qquad r(D^{\ast\ast}) \lesssim s^{-\alpha(1+\beta)/(2\alpha+K)}.
	$$
	
	Let $\Fcal$ be the collection of regression function $\eta(x)$ and probability density function $\mu(x)$ satisfying Assumption~\ref{ap:knn}. We have 
	$$
	\min_k\sup_{(\eta,\mu)\in \Fcal}r(\|\cdot\|^2) \gtrsim s^{-\alpha(1+\beta)/(2\alpha+d)},
	$$
	$$
	\min_k\sup_{(\eta,\mu)\in \Fcal}r(D^\ast) \gtrsim (s/\kappa^{K-1})^{-\alpha(1+\beta)/(2\alpha+K)}\quad {\rm and}\quad \min_k\sup_{(\eta,\mu)\in \Fcal}r(D^{\ast\ast}) \gtrsim s^{-\alpha(1+\beta)/(2\alpha+K)}.
	$$
\end{theorem}
We write $a\lesssim b$ for two sequences $a$ and $b$ if there exists a constant $C$ such that $a\le C b$, and $a\gtrsim b$ for two sequences $a$ and $b$ if there exists a constant $c$ such that $a\ge c b$. The two parts in Theorem~\ref{thm:knn} show that the convergence rates are tight. Theorem~\ref{thm:knn} suggests that when the target distances $D^\ast$ and $D^{\ast\ast}$ are used, the curse of dimensionality is alleviated and the convergence rate of $k$-NN can be much improved. The reason for the improvement is that the neighborhood defined by target distance $D^\ast$ and $D^{\ast\ast}$ can better fit the geometry of the Bayes classification rule than that defined by Euclidean distance. To illustrate this point, we compare balls defined by Euclidean distance and target distance, respectively,  denoted by $\Bcal_{\|\cdot\|^2}(x,r)$ and $\Bcal_{D^\ast}(x,r)$. The shapes of the two neighborhoods are quite different: $\Bcal_{\|\cdot\|^2}(x,r)$  is a standard sphere, while $\Bcal_{D^\ast}(x,r)$ is a cylinder, of which axis is in the orthogonal complement of $B$. One toy example in $\RR^2$ is illustrated in Figure~\ref{fg:knn}, where the red area is $\Bcal_{D^\ast}(x,r)$, and the yellow area is $\Bcal_{\|\cdot\|^2}(x,r)$. As pointed out by Proposition~\ref{prop:bayesrule}, the value of $\eta(x)$ only changes along with the directions in the column subspace of $U$, so we can expect values of $\eta(x)$ is more similar in $\Bcal_{D^\ast}(x,r)$ than in $\Bcal_{\|\cdot\|^2}(x,r)$ and thus $\Bcal_{D^\ast}(x,r)$ can lead to a smaller bias than $\Bcal_{\|\cdot\|^2}(x,r)$.

\begin{figure}[h!]
	\begin{center}
		\begin{tikzpicture}[scale=1]
			
			\fill[red!50!white,fill opacity=0.5] (-0.5,-0.5) -- (5,5) -- (6,4) -- (0.5,-1.5) --cycle ; 
			\fill[yellow!50!white,fill opacity=0.5] (2,1) circle (1.5);
			
			\path[-latex,draw,line width=1pt] (-3,0) -- (7,0);
			\path[-latex,draw,line width=1pt] (-1,-1.5) -- (-1,5.5);
			
			\draw[fill=black] (2,1) circle (0.06);
			\draw[line width=1.5pt] (2,1) circle (1.5);
			
			\path[draw,line width=1.5pt,black!80!white,dashed] (-0.5,-0.5) -- (5,5);
			\path[draw,line width=1.5pt,black!80!white,dashed] (0.5,-1.5) -- (6,4);
			
			\path[-latex,draw,line width=1pt] (2,1) -- (0.5,2.5);
			
			\draw[thick,black](0.56,2.9)node{Direction of $U$};
			\draw[thick,black](2,0.7)node{$x$};
			
			\draw[-{Latex[length=1mm, width=1.5mm]},line width=1pt] (4.1,0.3) node[right]{$\Bcal_{\|\cdot\|^2}(x,r)$}--(3.5,0.3);
			
			\draw[-{Latex[length=1mm, width=1.5mm]},line width=1pt] (6.1,3) node[right]{$\Bcal_{D^\ast}(x,r)$ or $\Bcal_{D^{\ast\ast}}(x,r)$ }--(5.5,3);
			
		\end{tikzpicture}
	\end{center}
	\caption{An illustrative example for the neighborhoods defined by Euclidean distance $\|\cdot\|^2$ and target distance $D^\ast$ or $D^{\ast\ast}$.}
	\label{fg:knn}
\end{figure}
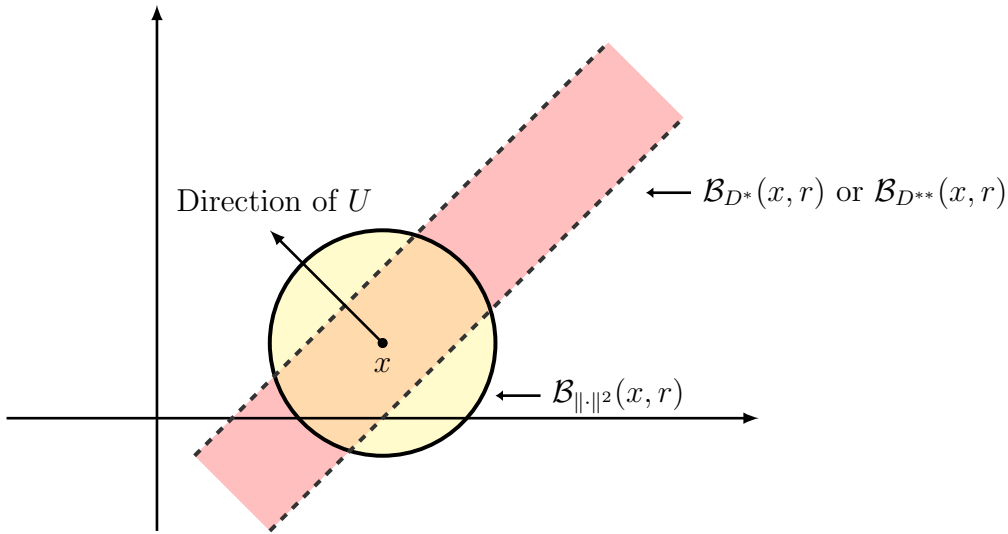

%%%%%%%%%%%%%%%%%%%%%%%%%%%%%%%%%%%%%%%%%%%%%%%
\subsection{Two-Sample Testing}
\label{sc:twosample}
%%%%%%%%%%%%%%%%%%%%%%%%%%%%%%%%%%%%%%%%%%%%%%%

Two-sample testing is central to statistical inferences and an important tool in many applications. Unlike the multi-view data used for metric learning, we observe only one view but with labels for each sample in the standard two-sample testing setting. Specifically, the data we observe in two-sample testing is $(X_1,Y_1),\ldots, (X_s,Y_s)$ and we are interested in the following hypothesis
$$
H_0: \EE(X|Y=-1)=\EE(X|Y=1)\qquad {\rm and}\qquad H_1: \EE(X|Y=-1)\ne \EE(X|Y=1).
$$
In order to test such a hypothesis, many different tests have been proposed. One of the most widely used test families is the distance-based method, including the energy distance test \citep{szekely2005new,sejdinovic2013equivalence}, permutational multivariate analysis of variance (PERMANOVA) \citep{mcardle2001fitting,anderson2014permutational,wang2021hypothesis}, and graph-based test \citep{friedman1979multivariate,chen2017new}. The idea of a distance-based test is that the pairwise distances between samples are first evaluated, and then the test is then constructed based on the distance matrix. The distance-based two-sample test is also closely related to the kernel-based two-sample test, such as the maximum mean discrepancy (MMD) \citep{gretton2012kernel}. In particular, \cite{sejdinovic2013equivalence} shows the equivalence between the energy distance test and the MMD test when the distance is a metric of negative type.  

In this section, we mainly focus on the energy distance test
$$
E(D)={2\over s_+s_{-}}\sum_{Y_i\ne Y_{i'}}D(X_i,X_{i'})-{1\over s_+(s_+-1)}\sum_{Y_i=Y_{i'}=1}D(X_i,X_{i'})-{1\over s_-(s_--1)}\sum_{Y_i=Y_{i'}=-1}D(X_i,X_{i'}),
$$
where $D$ is a given distance, $s_+=|\{i:Y_i=1\}|$, and $s_-=|\{i:Y_i=-1\}|$. The energy distance test compares the average within-group distance and the one across groups and can fully characterize the distribution homogeneity between groups when the distance is a metric of negative type \citep{sejdinovic2013equivalence}. Euclidean distance is a metric of negative type, but neither $D^\ast$ nor $D^{\ast\ast}$ is since they measure the difference only along with $K$ directions. This suggests that the target distances in self-supervised metric learning cannot fully capture the difference between two general distributions but are particularly suitable for the multi-view data, as we show in this section. To make decisions, we still need to choose a critical value for $E(D)$ or transform $E(D)$ to a $p$-value. Here, we consider two different ways to make decisions based on $E(D)$. The first one we consider here is the permutation test. Specifically, let $\Phi_s$ be the set of permutations on $\{1,\ldots, s\}$, i.e., $\Phi_s=\{\phi:\{1,\ldots, s\}\to \{1,\ldots, s\}|\phi(i)\ne \phi(j)\ {\rm if}\ i\ne j\}$. Given a permutation $\phi$, we write $\phi E(D)$ as the energy distance test statistic calculated on $(X_1,Y_{\phi(1)}),\ldots, (X_s,Y_{\phi(s)})$. Let $\phi_1,\ldots, \phi_B$ be $B$ permutations drawn from $\Phi_s$ randomly. Then, the $p$-value can be calculated by
$$
\hat{P}={1+\sum_{b=1}^B\bI_{(\phi_b E(D)\ge E(D))}\over 1+B}.
$$
We reject the null hypothesis when $\hat{P}\le \alpha$. The second way to make the decision is based on asymptotic distribution. We show that under the null hypothesis, $E(D)/{\rm sd}_{H_0}(E(D))\to N(0,1)$, where ${\rm sd}_{H_0}(E(D))$ is the standard deviation of $E(D)$ under the null hypothesis. So we can reject the null hypothesis when $E(D)>z_\alpha {\rm sd}_{H_0}(E(D))$ where $z_\alpha$ is the upper $\alpha$-quantile of standard normal distribution. ${\rm sd}_{H_0}(E(D))$ is usually a function of the covariance matrix and thus can be estimated consistently in practice \citep{chen2010two}.

The energy distance test's performance depends largely on the choice of distance and the difference between distributions in two groups. Here, we mainly study the tests' performance when the means between groups, $\mu=\EE(X|Y=-1)\ne \EE(X|Y=1)$, are different. We consider detection radius for the two-sample testing problem to compare the performance of different distances
$$
r(D,\epsilon)=\inf\left\{r:\underbrace{\PP(\phi_D=1|H_0)}_{type\  I\  error}+\underbrace{\PP(\phi_D=0|H_1(r))}_{type\  II\  error}\le \epsilon\right\},
$$
where $\phi_D$ is the test defined above by permutation test or asymptotic distribution and $H_1(r)=\{\|\mu\|\ge r\}$. Intuitively, the detection radius  $r(D,\epsilon)$ represents the smallest distance to separate the null and alternative hypothesis reliably. Thus, the test is more powerful to distinguish similar samples when $r(D,\epsilon)$ is smaller. To characterize the performance of energy distance test, we make the following assumptions. 
\begin{assumption}
	\label{ap:twosample}
	It holds that
	\begin{enumerate}[label=(\alph*)]
		\item we choose $\alpha=\epsilon/2$;
		\item assume $\PP(Y=1)=\PP(Y=-1)=1/2$;
		\item assume the covariance matrix of $\epsilon_{i,j}$ is $\Sigma$
		\item if  we write the covariance matrix of $X$ given $Y=1$ as $\Sigma_+$ and the covariance matrix of $X$ given $Y=-1$ as $\Sigma_-$, then we assume ${\rm Tr}(\Sigma_{i_1}\Sigma_{i_2}\Sigma_{i_3}\Sigma_{i_4})=o\left(\|\Sigma_++\Sigma_-\|_F^4\right)$ for $i_1,i_2,i_3,i_4=+$ or $-$. We assume it still hold when we replace $\Sigma_+$ and $\Sigma_-$ by $B^T\Sigma_+B$ and $B^T\Sigma_-B$ ($U^T\Sigma_+U$ and $U^T\Sigma_-U$).
		\item for any $1\le i<j\le s$, we assume $\EE(X_i^TX_j)^4=o\left(s\|\Sigma_++\Sigma_-\|_F^4\right)$, $\EE(X_i^TBB^TX_j)^4=o\left(s\|B^T(\Sigma_++\Sigma_-)B\|_F^4\right)$ and $\EE(X_i^TUU^TX_j)^4=o\left(s\|U^T(\Sigma_++\Sigma_-)U\|_F^4\right)$.
	\end{enumerate}
\end{assumption}
The first three assumptions in Assumption~\ref{ap:twosample} are fairly weak conditions, and the last two are moment conditions used for the central limit theorem of $U$-statistics. Similar assumptions also appear in \cite{hall1984central,chen2010two,li2019optimality}. If we use Euclidean distance and the distance $D^\ast$ and $D^{\ast\ast}$ in the energy distance test $E(D)$, the detection radius can be characterized by the following theorem. 

\begin{theorem}
	\label{thm:twosample}
	Suppose assumptions in Section~\ref{sc:model} and Assumption~\ref{ap:twosample} hold. If the test $\phi_D$ is defined by permutation test (permutation test does not need (d) and (e) in Assumption~\ref{ap:twosample}) or asymptotic distribution, then
	$$
	r(\|\cdot\|^2,\epsilon)\lesssim {\|BB^T+\Sigma\|_F^{1/2}\over \sqrt{s}},\qquad r(D^\ast,\epsilon)\lesssim {\|\Lambda^2+B^T\Sigma B\|_F^{1/2}\over \sqrt{s\lambda_K}}
	$$
	and
	$$
	r(D^{\ast\ast},\epsilon)\lesssim {\|\Lambda+U^T\Sigma U\|_F^{1/2}\over \sqrt{s}}.
	$$
	
	Consider  the energy distance test defined by permutation test or asymptotic distribution and the following local alternative hypothesis $\tilde{H}_1(r)=\left\{\mu=ru_K\right\}$.
	If $r=o(\|BB^T+\Sigma\|_F^{1/2}/\sqrt{s})$, then 
	$$
	\PP(\phi_{\|\cdot\|^2}=0|\tilde{H}_1(r))\to 1-\alpha.
	$$
	Similarly, if $r=o(\|\Lambda^2+B^T\Sigma B\|_F^{1/2}/\sqrt{s\lambda_K})$ or $r=o(\|\Lambda+U^T\Sigma U\|_F^{1/2}/\sqrt{s})$, then
	$$
	\PP(\phi_{D^\ast}=0|\tilde{H}_1(r))\to 1-\alpha\qquad {\rm and}\qquad \PP(\phi_{D^{\ast\ast}}=0|\tilde{H}_1(r))\to 1-\alpha.
	$$
\end{theorem}
Together with the first and second part of Theorem~\ref{thm:twosample}, the detection radius for Euclidean distance and the target distances of self-supervised metric learning are sharp. Theorem~\ref{thm:twosample} suggests that the detection radius of the energy distance test is mainly determined by the variation of $X$, which can be decomposed into two parts: the first part corresponds to the difference between samples and the second part is due to the variation between different views of the same sample. If we assume $\Sigma=\sigma^2I$ in Theorem~\ref{thm:twosample}, we can have 
$$
r(\|\cdot\|^2,\epsilon)\lesssim {(\sqrt{K}\lambda_1+\sqrt{d}\sigma^2)^{1/2}\over \sqrt{s}},\qquad r(D^\ast,\epsilon)\lesssim \sqrt{\kappa}{(\sqrt{K}\lambda_1+\sqrt{K}\sigma^2)^{1/2}\over \sqrt{s}}
$$
and 
$$
r(D^{\ast\ast},\epsilon)\lesssim {(\sqrt{K}\lambda_1+\sqrt{K}\sigma^2)^{1/2}\over \sqrt{s}}.
$$
When self-supervised metric learning is used, variation between different views can be reduced from $\sqrt{d}\sigma^2$ to $\sqrt{K}\sigma^2$. It implies that the energy distance test can be improved by self-supervised metric learning when the variation between different views dominates, i.e., $\sqrt{K}\lambda_1\ll \sqrt{d}\sigma^2$.

%%%%%%%%%%%%%%%%%%%%%%%%%%%%%%%%%%%%%%%%%%%%%%%
\section{Self-Supervised Metric Learning in Multi-View Data}
\label{sc:spectral}
%%%%%%%%%%%%%%%%%%%%%%%%%%%%%%%%%%%%%%%%%%%%%%%

%%%%%%%%%%%%%%%%%%%%%%%%%%%%%%%%%%%%%%%%%%%%%%%
\subsection{Data-Driven Distance on Downstream Tasks}
%%%%%%%%%%%%%%%%%%%%%%%%%%%%%%%%%%%%%%%%%%%%%%%

In the previous section, we show that target distances $D^\ast$ and $D^{\ast\ast}$ in self-supervised metric learning are good distances for downstream analysis. However, we cannot directly adopt target distances in each downstream task as they are usually unknown in advance. In practice, we still need to estimate $D^\ast$ and $D^{\ast\ast}$ from the unlabeled multi-view data. One may wonder if the data-driven distances estimated from unlabeled multi-view data can also improve the downstream tasks similarly to target distances. Our investigation in this section confirms that the data-driven distance can benefit the downstream analysis when the target distances can be estimated accurately. It is sufficient to  estimate the following matrices to estimate the target distances
$$
M^\ast=BB^T\qquad {\rm and}\qquad M^{\ast\ast}=UU^T.
$$
Let $\hat{M}^\ast$ and $\hat{M}^{\ast\ast}$ be some estimators for $M^\ast$ and $M^{\ast\ast}$, and $D_{\hat{M}^\ast}$ and $D_{\hat{M}^{\ast\ast}}$ be the distances defined by them. The measure $\Delta(D,D')$ can be rewritten as the spectral norm of matrix difference, $\Delta(D,D')=\|M-M'\|$, where $D(X_1,X_2)=(X_1-X_2)^TM(X_1-X_2)$ and $D'(X_1,X_2)=(X_1-X_2)^TM'(X_1-X_2)$. The following theorem shows that the estimated distances can still improve downstream analysis. 

\begin{theorem}
	\label{thm:downstream}
	Suppose the data in self-supervised metric learning is independent from the data in downstream tasks and assumptions in Section~\ref{sc:model} hold and $\kappa$ is bounded. Let $\hat{M}^\ast$ and $\hat{M}^{\ast\ast}$ be some estimators of $M^\ast$ and $M^{\ast\ast}$ such that 
	$$
	\Delta(D^\ast,D_{\hat{M}^\ast})\le \delta^\ast\qquad {\rm and}\qquad \Delta(D^{\ast\ast},D_{\hat{M}^{\ast\ast}})\le \delta^{\ast\ast}.
	$$
	\begin{itemize}
		\item ($k$-nearest neighbor classification) Suppose Assumption~\ref{ap:knn} holds and let $c$ be some constant. If $k=c(s/\kappa^{K-1})^{2\alpha/(2\alpha+K)}$, $\delta^\ast\lesssim \lambda_K(s/\kappa^{K-1})^{-1/(2\alpha+K)}$ in $D_{\hat{M}^\ast}$ or $k=cs^{2\alpha/(2\alpha+K)}$, $\delta^{\ast\ast}\lesssim s^{-1/(2\alpha+K)}$ in $D_{\hat{M}^{\ast\ast}}$, then
		$$
		r(D_{\hat{M}^\ast}) \lesssim  (s/\kappa^{K-1})^{-\alpha(1+\beta)/(2\alpha+K)}\qquad {\rm and}\qquad r(D_{\hat{M}^{\ast\ast}}) \lesssim s^{-\alpha(1+\beta)/(2\alpha+K)}.
		$$
		\item (two-sample testing) Suppose Assumption~\ref{ap:twosample} and $\|\Sigma\|\lesssim \lambda_1$ hold and let $c$ be a large enough constant.  If $\delta^\ast=o(\lambda_K)$ in $D_{\hat{M}^\ast}$ or $\delta^{\ast\ast}=o(1)$ in $D_{\hat{M}^{\ast\ast}}$, then
		$$
		r(D_{\hat{M}^\ast},\epsilon)\lesssim {\|\Lambda^2+B^T\Sigma B\|_F^{1/2}\over \sqrt{s}}\qquad {\rm and}\qquad r(D_{\hat{M}^{\ast\ast}},\epsilon)\lesssim {\|\Lambda+U^T\Sigma U\|_F^{1/2}\over \sqrt{s}}.
		$$
		\item ($k$-means clustering) Suppose Assumption~S1 holds  and $t>\log s$. If $\|B^T\mu\|\gg \Psi(\Lambda^2+B^T\Sigma_\pm B)$, $\delta^\ast=o(\lambda_K)$ in $D_{\hat{M}^\ast}$ or  $\|\mu\|\gg \Psi(\Lambda+U^T\Sigma_\pm U)$, $\delta^{\ast\ast}=o(1)$ in $D_{\hat{M}^{\ast\ast}}$, then
		$$
		r(D_{\hat{M}^\ast})\le \Gamma(1+o(1), B^T\mu, B^T\Sigma_\pm B)\qquad {\rm and}\qquad r(D_{\hat{M}^{\ast\ast}}) \le \Gamma(1+o(1), \mu, U^T\Sigma_\pm U)
		$$
		with probability at least $1-s^5-\exp(-\sqrt{v}\|\mu\|)$ where $v\to \infty$.
		\item (sample identification) 	Suppose Assumption~S2 holds and $\lambda_d(\Sigma)\ge c\|\Sigma\|$ where $\lambda_d(\Sigma)$ is the smallest eigenvalue of $\Sigma$. If $\delta^\ast=o(\lambda_K)$ in $D_{\hat{M}^\ast}$ or $\delta^{\ast\ast}=o(1)$ in $D_{\hat{M}^{\ast\ast}}$, then
		$$
		r(D_{\hat{M}^\ast},\epsilon)\lesssim {\|B^T\Sigma B\|_F^{1/2}\over \lambda_K}\quad {\rm and }\quad  r(D_{\hat{M}^{\ast\ast}},\epsilon)\lesssim {\|U^T\Sigma U\|_F^{1/2}\over \sqrt{\lambda_K}}.
		$$
	\end{itemize}
\end{theorem}

Theorem~\ref{thm:downstream} suggests that the estimated distance $D_{\hat{M}^\ast}$ and $D_{\hat{M}^{\ast\ast}}$ from the self-supervised metric learning could help achieve a similar performance as $D^\ast$ and $D^{\ast\ast}$ when the target distances can be estimated accurately. Self-supervised learning can help improve two-sample testing, $k$-means clustering, and sample identification as long as we have enough unlabeled multi-view data to estimate the target distance consistently, i.e., $\Delta(D^\ast,D_{\hat{M}^\ast})=o(\lambda_K)$ or $\Delta(D^{\ast\ast},D_{\hat{M}^{\ast\ast}})=o(1)$. Unlike these three downstream tasks, the improvement of $k$-nearest neighbor classification needs a more accurate estimation of target distance. Theorem~\ref{thm:downstream} assumes the independence between data in metric learning and downstream tasks for the simplicity of analysis. This is a reasonable assumption when we have many unlabeled multi-view data in a typical self-supervised learning setting. If the metric learning and downstream tasks use the same data set, the results in Theorem~\ref{thm:downstream} might still hold, but the analysis can be much more involved. 

%%%%%%%%%%%%%%%%%%%%%%%%%%%%%%%%%%%%%%%%%%%%%%%
\subsection{Spectral Self-Supervised Metric Learning}
\label{sc:metrixfinitesample}
%%%%%%%%%%%%%%%%%%%%%%%%%%%%%%%%%%%%%%%%%%%%%%%

The previous section shows that the downstream task can be improved when the target distances can be estimated accurately. Two questions naturally arise: how shall we estimate the target distances? how much unlabeled multi-view data is sufficient to improve the downstream analysis? To answer these questions, we consider a spectral method to estimate $D^\ast$ and $D^{\ast\ast}$ in this section. Since $M^\ast$ is the optimal solution of \eqref{eq:metriclearning}, a natural idea of estimating $M^\ast$ is to replace $\EE\left(D_M(X_{i,j},X_{i',j'})-D_M(X_{i,j},X_{i,j'})\right)$ with its empirical version. More concretely, its empirical version can be written as
$$
{1\over m(m-1)n^2}\sum_{i\ne i',j,j'}D_M(X_{i,j},X_{i',j'})-{1\over mn(n-1)}\sum_{i,j\ne j'}D_M(X_{i,j},X_{i,j'}).
$$
Here, we consider all pairs of dissimilar and similar data and use $U$-statistics as the estimator. After plugging in the empirical version of distance difference and some calculation, $M^\ast$ can be estimated by the following optimization problem
$$
\max_{M} {\rm Tr}\left(\hat{R}M\right),\qquad {\rm s.t.}\quad \|M\|_F\le 1\quad {\rm and}\quad {\rm rank}(M)\le K.
$$
where $\hat{R}$ is a $d\times d$ matrix 
\begin{align*}
	\hat{R}&={1\over mn(n-1)}\sum_{i,j\ne j'}\left(X_{i,j}X_{i,j'}^T+X_{i,j'}X_{i,j}^T\right)-{1\over m(m-1)}\sum_{i\ne i'}\left(\bar{X}_{i}\bar{X}_{i'}^T+\bar{X}_{i'}\bar{X}_{i}^T\right)
\end{align*}
Here, $\hat{R}$ is an unbiased estimator of $BB^T$ regardless of the $\epsilon_{i,j}$'s distribution. The reason for having unbiased estimator is that we observe several views of each sample. This is different from the classical factor model, where we only observe a single view for each sample \citep{fan2020statistical}. In the above optimization problem, we also add a constraint for the rank of $M$ since $BB^T$ is a low-rank matrix. This optimization problem's form can then naturally lead to a simple spectral algorithm to estimate $M^\ast$, summarized in Algorithm~\ref{ag:metric}. The spectral method in can also be easily adjusted to estimate $M^{\ast\ast}$ when we change the last step, which is also included in Algorithm~\ref{ag:metric}.

\begin{algorithm}[h!]
	\caption{Spectral Metric Learning in Multi-view Data}
	\label{ag:metric}
	\begin{algorithmic}
		\REQUIRE Multi-view data $(X_{i,1},\ldots,X_{i,n_i})$ for $i=1,\ldots m$.
		\ENSURE A matrix $\hat{M}^\ast$ or $\hat{M}^{\ast\ast}$.
		\STATE Evaluate $\hat{R}$.
		\STATE Find the first $K$ eigenvalues and eigenvectors of $\hat{R}$, i.e., $(\hat{\lambda}_1,\ldots,\hat{\lambda}_K)$ and $(\hat{u}_1,\ldots,\hat{u}_K)$.
		\STATE Estimate $\hat{M}^\ast$ or $\hat{M}^{\ast\ast}$ by
		$$
		\hat{M}^\ast=\sum_{k=1}^K\hat{\lambda}_k\hat{u}_k\hat{u}_k^T\qquad {\rm or}\qquad  \hat{M}^{\ast\ast}=\sum_{k=1}^K\hat{u}_k\hat{u}_k^T.
		$$
	\end{algorithmic}
\end{algorithm}

The Algorithm~\ref{ag:metric} seems computationally expensive at first sight since the definition of $\hat{R}$ involves $U$-statistics, which usually requires quadratic time complexity. However, thanks to the special structure of empirical covariance matrix $\hat{R}$, it can be rewritten as the following equivalent form
\begin{align*}
	\hat{R}=\left({n\over n-1}+{1\over m-1}\right){1\over m}\sum_{i}\bar{X}_{i}\bar{X}_{i}^T-{1\over mn(n-1)}\sum_{i,j}X_{i,j}X_{i,j}^T-{m\over m-1}\bar{\bar{X}}\bar{\bar{X}}^T
\end{align*}
where $\bar{X}_i=n^{-1}\sum_j X_{i,j}$ and $\bar{\bar{X}}=m^{-1}\sum_i \bar{X}_i$. Thus, $\hat{R}$ can be computed in a linear time. 

We now investigate the theoretical properties of $\hat{M}^\ast$ or $\hat{M}^{\ast\ast}$ in Algorithm~\ref{ag:metric}. To the end, we make the following assumptions.
\begin{assumption}
	\label{ap:metriclearning}
	It holds that
	\begin{enumerate}[label=(\alph*)]
		\item $\epsilon_{i,j}$ and $Z_i$ follow sub-Gaussian distributions, that is, for any $a\in \RR^d$ and $b\in \RR^K$
		$$
		\EE\left(e^{\langle a,\epsilon_{i,j}\rangle}\right)\le e^{\sigma^2 \|a\|^2/2}\qquad {\rm and}\qquad \EE\left(e^{\langle b,Z_i-\EE(Z_i)\rangle}\right)\le e^{\|b\|^2/2};
		$$
		\item conditional number $\kappa$ is bounded;
		\item assume $K$ is known.
	\end{enumerate}
\end{assumption}
The assumption on sub-Gaussian distributions is the key assumption in Assumption~\ref{ap:metriclearning}, which is commonly used in the study of eigenspace estimation \citep{zhang2018heteroskedastic,chen2020spectral}. Since we observe multi-view data of each sample, we do not assume diagonal or sparse covariance matrix as literature \citep{yao2015sample, zhang2018heteroskedastic}. The following theorem characterizes the convergence rate of $\hat{M}^\ast$ and $\hat{M}^{\ast\ast}$.

\begin{theorem}
	\label{thm:metricupper}
	Suppose assumptions in Section~\ref{sc:model} and Assumption~\ref{ap:metriclearning} hold. If $m\ge c\log(d+m)(\kappa^2K\log(d+m)+d\kappa\sigma^2/n\lambda_K+d\sigma^4/n^2\lambda^2_K)$ for a large enough constant $c$, then, with probability at least $1-{6/(d+m)^5}$, we have
	$$
	\Delta(D^\ast,D_{\hat{M}^\ast})\lesssim {\sqrt{\log(d+m)}\over \sqrt{m}}\left[\sqrt{K}\lambda_1+\sigma{\sqrt{d\lambda_1}\over \sqrt{n}}+\sigma^2{\sqrt{d}\over  n}\right].
	$$
	
	In addition, if $m\ge c\log(d+m)(K+d\kappa\sigma^2/n\lambda_K+d\sigma^4/n^2\lambda^2_K)$ for a large enough constant $c$, we have similar results for $\hat{M}^{\ast\ast}$, that is 
	$$
	\Delta(D^{\ast\ast},D_{\hat{M}^{\ast\ast}})\lesssim {\sqrt{\log(d+m)}\over \sqrt{m}}\left[\sigma{\sqrt{d}\over \sqrt{n\lambda_K}}+\sigma^2{\sqrt{d}\over n\lambda_K}\right]
	$$
	with probability at least $1-{6/(d+m)^5}$.
\end{theorem}

Naturally, one may wonder whether the bound for spectral method is tight, and if there are some other methods that can help learn distance $D^\ast$ or $D^{\ast\ast}$ better. To answer these questions, we develop the information-theoretic lower bound that matches the upper bound in Theorem~\ref{thm:metricupper}. To develop the lower bound, we focus on the following Gaussian noise model
$X_{i,j}=BZ_i+\epsilon_{i,j}$, where $Z_i\sim N(0,I)$ and $\epsilon_{i,j}\sim N(0,\sigma^2I)$ and consider the collection of matrix $B$
$$
\Bcal(\nu)=\left\{B\in \RR^{d\times K}:\lambda_1(B)/\lambda_K(B)\le \sqrt{\kappa}, \lambda_K(B)\ge \sqrt{\nu}\right\},
$$
where $\lambda_1(B)$ and $\lambda_K(B)$ are the largest and smallest singular value of $B$.
\begin{theorem}
	\label{thm:metriclower}
	Suppose $\kappa> 1$ is bounded, $m>K$ and $4K\le d$.  Then
	$$
	\inf_{\hat{M}^\ast}\sup_{B\in \Bcal(\nu)}\EE\left(\Delta(D^\ast,D_{\hat{M}^\ast})\right)\gtrsim {1\over \sqrt{m}}\left[\sqrt{K}\nu+\sigma{\sqrt{d\nu}\over \sqrt{n}}+\sigma^2{\sqrt{d}\over n}\right].
	$$
	We also have similar results for $\hat{M}^{\ast\ast}$, that is 
	$$
	\inf_{\hat{M}^{\ast\ast}}\sup_{B\in \Bcal(\nu)}\EE\left(\Delta(D^{\ast\ast},D_{\hat{M}^{\ast\ast}})\right)\gtrsim {1\over \sqrt{m}}\left[\sigma{\sqrt{d}\over \sqrt{n\nu}}+\sigma^2{\sqrt{d}\over  n\nu}\right].
	$$
\end{theorem}
Through comparing Theorem~\ref{thm:metricupper} and \ref{thm:metriclower}, we can know the results in Theorem~\ref{thm:metricupper} are indeed sharp up to a logarithm factor. As shown in these two theorems, estimating $D_{\hat{M}^{\ast\ast}}$ is easier than $D_{\hat{M}^\ast}$ since there is no need for estimating the eigenvalues $\hat{\lambda}_1,\ldots, \hat{\lambda}_K$. Theorem~\ref{thm:metricupper} also suggests $D_{\hat{M}^\ast}$ and $D_{\hat{M}^{\ast\ast}}$ can improve the downstream analysis provided the sample size of unlabeled multi-view data is large enough. By combining Theorem~\ref{thm:downstream} and \ref{thm:metricupper}, we have the following corollary which precisely characterizes the sample size needed for downstream tasks improvement.

\begin{corollary} \label{cor:complexity} Suppose assumptions in Theorem~\ref{thm:downstream} and \ref{thm:metricupper} hold.
	If $m\gg (s/\kappa^{K-1})^{1/(2\alpha+K)}\log(d+m)(K+d\sigma^2/n\lambda_K+d\sigma^4/n^2\lambda^2_K)$ in $D_{\hat{M}^\ast}$ or $m\gg s^{1/(2\alpha+K)}\log(d+m)(d\sigma^2/n\lambda_K+d\sigma^4/n^2\lambda^2_K)$ in $D_{\hat{M}^{\ast\ast}}$, $k$-NN can achieve the same convergence rate in Theorem~\ref{thm:downstream}. If $m\gg \log(d+m)(K+d\sigma^2/n\lambda_K+d\sigma^4/n^2\lambda^2_K)$ in $D_{\hat{M}^\ast}$ or $m\gg \log(d+m)(d\sigma^2/n\lambda_K+d\sigma^4/n^2\lambda^2_K)$ in $D_{\hat{M}^{\ast\ast}}$, two-sample testing, $k$-means clustering and sample identification can also achieve the same convergence rate in Theorem~\ref{thm:downstream}.
\end{corollary}

%%%%%%%%%%%%%%%%%%%%%%%%%%%%%%%%%%%%%%%%%%%%%%%
\section{Numerical Experiments}
\label{sc:numerical}
%%%%%%%%%%%%%%%%%%%%%%%%%%%%%%%%%%%%%%%%%%%%%%%

In this section, we conduct several numerical experiments to complement our theoretical developments. In particular, we compare the performance of the four downstream tasks in Section~\ref{sc:downstream} when Euclidean distance and resulting distance from metric learning are used.

%%%%%%%%%%%%%%%%%%%%%%%%%%%%%%%%%%%%%%%%%%%%%%%
\subsection{Simulated Data}
%%%%%%%%%%%%%%%%%%%%%%%%%%%%%%%%%%%%%%%%%%%%%%%

To simulate the data, we consider the Gaussian model $X_{i,j}=BZ_i+\epsilon_{i,j}$, where $\epsilon_{i,j}\sim N(0,\sigma^2I)$. Here, we choose $\lambda_k=\lambda(K-k+1)/K$ for some $\lambda$ and the directions of $B$, $\{b_1/\|b_1\|,\ldots,b_K/\|b_K\|\}$, are obtained  from the first $K$ left-singular vectors of randomly generated $d\times d$ standard Gaussian matrix. We generate $Z_i$ from a mixture model, $0.5N(\alpha,I-\alpha\alpha^T)+0.5N(-\alpha,I-\alpha\alpha^T)$, for some $\alpha\in \RR^K$ with $\|\alpha\|< 1$. We let $Y_i=1$ if $Z_i$ is drawn from $N(\alpha,I-\alpha\alpha^T)$ and $Y_i=-1$ otherwise. 

\paragraph{Sample identification} To study the effect of $\|Z_1-Z_2\|$ and $K$, we vary $\|Z_1-Z_2\|=1,2,3,4,5$ and $K=10, 50$. Specifically, we set the first $K/2$ elements in $Z_1-Z_2$ as zero and the last $K/2$ elements in $Z_1-Z_2$ as the same non-zero constant. We consider 7 distances: Euclidean distance, target distance $D^\ast$ and $D^{\ast\ast}$, estimated distance $D^\ast$ and $D^{\ast\ast}$ by spectral method with $m=1000, 5000$ samples. We choose $d=100$, $\lambda=4$, $\sigma^2=1$ and $n=10$ and repeat the simulation 500 times. We compare the performance of sample identification by power, which is estimated by the number of rejecting null hypothesis. The results are summarized in Table~\ref{tb:sampleiden}. Table~\ref{tb:sampleiden} suggests that self-supervised metric learning is indeed helpful for sample identification, and the helps shrinkage when $K$ becomes larger, which is consistent with the theoretical results. 

\begin{table}[h!]
	\centering
	\begin{tabular}{cccccccccccc}
		\hline\hline
		& \multicolumn{5}{c}{$K=10$} & & \multicolumn{5}{c}{$K=50$} \\ 
		\cline{2-6}\cline{8-12}
		$\|Z_1-Z_2\|$ & $1$ & $2$ &$ 3$ & $4$ & $5$ & & $1$ & $2$ &$ 3$ & $4$ & $5$ \\ 
		\hline
		$\|\cdot\|^2$ & 0.08 & 0.21 & 0.42 & 0.77 & 0.96 & & 0.07 & 0.13 & 0.34 & 0.61 & 0.91\\ 
		$\hat{D}^\ast\ (1000)$ & 0.04 & 0.24 & 0.64 & 0.95 & 1.00 & & 0.08 & 0.14 & 0.28 & 0.53 & 0.85\\ 
		$\hat{D}^\ast\ (5000)$ & 0.05 & 0.27 & 0.65 & 0.96 & 1.00 & & 0.08 & 0.13 & 0.30 & 0.56 & 0.87\\ 
		$D^\ast$ & 0.06 & 0.27 & 0.67 & 0.97 & 1.00 & & 0.08 & 0.13 & 0.30 & 0.56 & 0.87\\ 
		$\hat{D}^{\ast\ast}\ (1000)$ & 0.09 & 0.47 & 0.90 & 0.99 & 1.00 & & 0.09 & 0.22 & 0.50 & 0.83 & 0.99\\ 
		$\hat{D}^{\ast\ast}\ (5000)$ & 0.09 & 0.48 & 0.90 & 1.00 & 1.00 & & 0.08 & 0.20 & 0.49 & 0.82 & 0.99\\ 
		$D^{\ast\ast}$ & 0.09 & 0.47 & 0.90 & 1.00 & 1.00 & & 0.08 & 0.20 & 0.50 & 0.82 & 0.99\\ 
		\hline\hline
	\end{tabular}
	\caption{Comparisons of different distances on sample identification.}
	\label{tb:sampleiden}
\end{table}

\paragraph{Two-sample testing} We now move to the simulation experiment for two-sample testing. Similar to sample identification, we still compare the same 7 distances and choose $d=100$, $\sigma^2=1$, $K=10$, $n=10$ and $s=500$. Let $\alpha$ be a vector such that $\alpha_1=\ldots=\alpha_4=0$ and $\alpha_5=\ldots=\alpha_{10}=r/\sqrt{6}$ for some $r$. We study the effect of $\|\mu\|$ and $\lambda$ by considering the following two experiment settings: 1) $\lambda=1$ and $r=0,0.05,\ldots,0.5$ 2) $\lambda=0.5,1,\ldots, 5$ and $r=0.3/\lambda$ so that $\|\mu\|$ is fixed. To evaluate the power of different methods, we still repeat the simulation 500 times. The results are summarized in Figure~\ref{fg:twosample}. Through Figure~\ref{fg:twosample}, we can conclude that self-supervised metric learning is helpful when  $\lambda/\sigma^2$ is moderate, while all distances perform similarly when $\lambda/\sigma^2$ is large. These results help verify the theoretical conclusion in Theorem~\ref{thm:twosample}.

\begin{figure}[h!]
	\centering
	\includegraphics[width=0.8\textwidth]{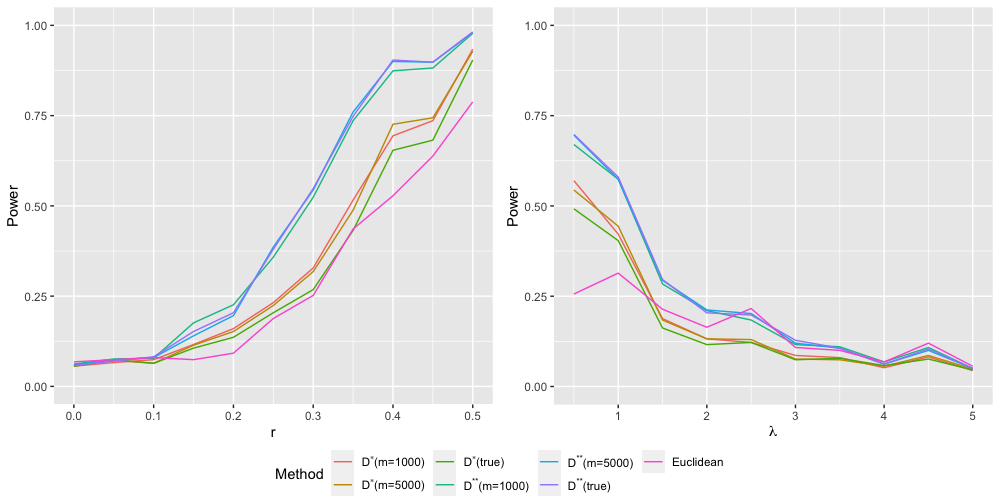}
	\caption{Comparisons of different distances on two-sample testing: left is setting 1 and right is setting 2.}\label{fg:twosample}
\end{figure}

\paragraph{$k$-means clustering} We then consider the simulation experiment for $k$-means clustering. We adopt the same setting in two-sample testing and set $\lambda=2$. We choose $\alpha$ as a vector such that $\alpha_1=\ldots=\alpha_4=r/\sqrt{4}$ for some $r$ and $\alpha_5=\ldots=\alpha_{10}=0$. To compare the required signal, we vary $r=0.4,0.6,0.8,1$ and use the mis-clustering rate as the measure of performance, which is defined in Section~S1.1. We consider two ways to choose the initial estimator of mean in $k$-means: 1) we randomly choose the two points as initial points 2) we use the true mean in each class as initial points. The results based on the 500 times simulation are summarized in Table~\ref{tb:kmeans}. In Table~\ref{tb:kmeans}, even when the starting point is perfect, the performances of $D^\ast$ is not as good as Euclidean distance and $D^{\ast\ast}$ due to the anisotropic transformation. Moreover, the distance $D^{\ast\ast}$ is slightly helpful when random initial points are used. This is again consistent with the theoretical results.  

\begin{table}[h!]
	\centering
	\begin{tabular}{ccccccccccc}
		\hline\hline
		&& \multicolumn{4}{c}{Random Start} & & \multicolumn{4}{c}{Perfect Start} \\ 
		\cline{3-6}\cline{8-11}
		& & $r=0.4$ & $r=0.6$ &$ r=0.8$ & $r=1$  & & $r=0.4$ & $r=0.6$ &$ r=0.8$ & $r=1$\\ 
		\hline
		$\|\cdot\|^2$ && 0.43 & 0.39 & 0.34 & 0.14  & & 0.38 & 0.31 & 0.21 & 0.05 \\ 
		$\hat{D}^\ast\ (1000)$ && 0.43 & 0.40 & 0.36 & 0.23  & & 0.41 & 0.37 & 0.29 & 0.12\\ 
		$\hat{D}^\ast\ (5000)$ && 0.43 & 0.39 & 0.34 & 0.23  & & 0.41 & 0.37 & 0.31 & 0.14 \\ 
		$D^\ast$ && 0.43 & 0.39 & 0.34 & 0.24  & & 0.41 & 0.37 & 0.31 & 0.15 \\ 
		$\hat{D}^{\ast\ast}\ (1000)$ && 0.43 & 0.39 & 0.34 & 0.12  & & 0.40 & 0.34 & 0.24 & 0.05\\ 
		$\hat{D}^{\ast\ast}\ (5000)$ && 0.43 & 0.39 & 0.34 & 0.12  & & 0.39 & 0.34 & 0.24 & 0.05 \\ 
		$D^{\ast\ast}$ && 0.43 & 0.39 & 0.34 & 0.13  & & 0.40 & 0.34 & 0.24 & 0.05 \\ 
		\hline\hline
	\end{tabular}
	\caption{Comparisons of different distances on $k$-means clustering.}
	\label{tb:kmeans}
\end{table}

\paragraph{$k$-nearest neighbor classification} In the last simulation experiment, we compare the performance of  $k$-nearest neighbor classification when it works with different distances. We use the same setting in $k$-means clustering and vary $s$ and $r$ in $\alpha$, where $\alpha_1=\ldots=\alpha_5=0$ and $\alpha_6=\ldots=\alpha_{10}=r/\sqrt{5}$. Specifically, we consider the following two experiment settings: $r=0.9$ and the sample size is different $s=500,1000,\ldots, 5000$; sample size is $s=2000$ and $r=0.1,\ldots,1$. The misclassification error defined in Section~\ref{sc:knn} is used as the measure for performance of different distances. The results are summarized in Figure~\ref{fg:knnsimu}, showing the self-supervised metric learning is helpful for $k$-NN, and the error decreases when the sample size or the difference between populations increases (large $r$ implies large $\beta$ in marginal assumption). 

\begin{figure}[h!]
	\centering
	\includegraphics[width=0.8\textwidth]{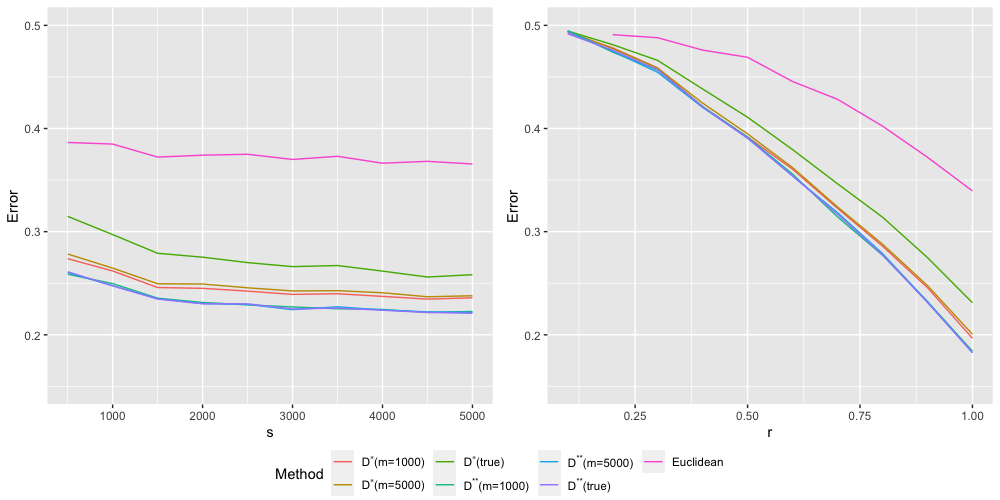}
	\caption{Comparisons of different distances on $k$-nearest neighbor classification.}\label{fg:knnsimu}
\end{figure}

All the numerical results in these four simulation experiments are consistent with theoretical conclusion in Section~\ref{sc:downstream}. Compared with target distance $D^\ast$, the isotropic target distance $D^{\ast\ast}$ is a better choice for all four downstream tasks we consider here. In addition, distance estimated from self-supervised metric learning performs almost as well as the true target distance in these simulation experiments. 

%%%%%%%%%%%%%%%%%%%%%%%%%%%%%%%%%%%%%%%%%%%%%%%
\subsection{Computer Vision Task}
%%%%%%%%%%%%%%%%%%%%%%%%%%%%%%%%%%%%%%%%%%%%%%%
We further compare Euclidean distance and resulting distance from self-supervised metric learning on some computer vision tasks. Specifically, we consider two datasets: MNIST \citep{lecun1998gradient} and Fashion-MNIST \citep{xiao2017fashion}. Both datasets contain $6\times 10^4$ training images and $10^4$ testing images, which are all $28\times 28$ gray-scale images from 10 classes. The difference between the two datasets is that MNIST is a collection of handwritten digits while Fashion-MNIST is a collection of clothing. MNIST and Fashion-MNIST do not contain multi-view data, but we can generate a multi-view dataset by shifting the images. Specifically, we shift the image in 4 different directions (left, right, upper and lower) to generate the multi-view dataset. A toy example of image shifting can be found in Figure~\ref{fg:mnist}.

\begin{figure}[h!]
	\centering
	\includegraphics[width=0.17\textwidth]{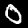}
	\includegraphics[width=0.17\textwidth]{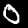}
	\includegraphics[width=0.17\textwidth]{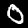}
	\includegraphics[width=0.17\textwidth]{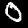}
	\includegraphics[width=0.17\textwidth]{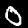}
	\caption{Multi-view data generated from MNIST dataset: from left to right are original, left shift, right shift, upper shift and lower shift.}\label{fg:mnist}
\end{figure}

In each dataset, we consider applying $k$-NN to classify the images. In this numerical experiment, a large unlabeled multi-view dataset ($m=10^4$ and $n=5$) and a small labeled dataset ($s=10^3, 2\times 10^3, 5\times 10^3$) are randomly drawn from training images and then used to train a $k$-NN classifier. We consider the following three ways to train $k$-NN classifier: 1) Euclidean distance is used to train $k$-NN directly on the small labeled dataset; 2) the anisotropic distance $D^\ast$ is estimated by the spectral method from the unlabeled multi-view dataset, and then the estimated distance is used to train $k$-NN; 3) the isotropic distance $D^{\ast\ast}$ is estimated from the unlabeled multi-view dataset and then used to train $k$-NN. To measure the performances, we adopt the misclassification errors, which can be estimated on $10^3$ images randomly drawn from testing images. The misclassification errors are reported in Table~\ref{tb:mnist}. It suggests that the self-supervised metric learning on the dataset from simple image shifting is helpful for the downstream classification task. 

\begin{table}[h!]
	\centering
	\begin{tabular}{ccccccccc}
		\hline\hline
		&& \multicolumn{3}{c}{MNIST} & & \multicolumn{3}{c}{Fashion-MNIST} \\ 
		\cline{3-5}\cline{7-9}
		& & $\|\cdot\|^2$ & $D^\ast$ & $D^{\ast\ast}$   & & $\|\cdot\|^2$ & $D^\ast$ & $D^{\ast\ast}$ \\ 
		\hline
		$s=1000$ && 0.115 & 0.268 & 0.094  & & 0.254 & 0.380 & 0.254\\ 
		$s=2000$ && 0.086 & 0.222 & 0.079  & &0.240 & 0.352 & 0.233 \\ 
		$s=5000$ && 0.062 & 0.169 & 0.059 & & 0.208 & 0.318 & 0.204\\ 
		\hline\hline
	\end{tabular}
	\caption{Comparisons of different distances on computer vision task.}
	\label{tb:mnist}
\end{table}

%%%%%%%%%%%%%%%%%%%%%%%%%%%%%%%%%%%%%%%%%%%%%%%
\section{Conclusion}
\label{sc:conclusion}
%%%%%%%%%%%%%%%%%%%%%%%%%%%%%%%%%%%%%%%%%%%%%%%

This paper conducts a systematic investigation of self-supervised metric learning in unlabeled multi-view data from a downstream task perspective. Building on a latent factor model for multi-view data, we provide theoretical justification for the success of this popular approach. Our analysis precisely characterizes the improvement by self-supervised metric learning on several downstream tasks, including sample identification, two-sample testing, $k$-means clustering, and $k$-nearest neighbor classification. Furthermore, we also establish the upper bound on distance estimation's accuracy and the number of samples sufficient for downstream task improvement. We assume that the number of factors $K$ is known in the analysis. In practice, some data-driven methods can help choose $K$, like Kaiser criterion and scree plot, when it is unknown. See more discussion in Chapter 10 of \cite{fan2020statistical}. The results in this paper rely on the assumption of the latent factor model and are designed for Mahalanobis distance. It could also be interesting to explore if the results can be extended to the deep neural network-based metric learning methods. 

\bibliographystyle{plainnat}
\bibliography{MetricLearning}

\newpage
\appendixpageoff
\appendixtitleoff
\renewcommand{\appendixtocname}{Supplementary Material}
\begin{appendices}
	\setcounter{section}{0}
	\setcounter{equation}{0}
	\setcounter{theorem}{0}
	\setcounter{assumption}{0}
	\setcounter{figure}{0}
	\setcounter{table}{0}
	\def\theequation{S\arabic{section}.\arabic{equation}}
	\def\thesection{S\arabic{section}}
	\def\thetheorem{S\arabic{theorem}}
	\def\theassumption{S\arabic{assumption}}
	\def\thefigure{S\arabic{figure}}
	\def\thetable{S\arabic{table}}

\begin{center}
	\textbf{\LARGE Supplementary Material}
\end{center}
In this supplementary material, we provide some extra results, the proof for the main results and all the technical lemmas. 
	
%%%%%%%%%%%%%%%%%%%%%%%%%%%%%%%%%%%%%%%%%%%%%%%
\section{More Specific Tasks for Self-Supervised Metric Learning}
\label{sc:extra}
%%%%%%%%%%%%%%%%%%%%%%%%%%%%%%%%%%%%%%%%%%%%%%%

%%%%%%%%%%%%%%%%%%%%%%%%%%%%%%%%%%%%%%%%%%%%%%%
\subsection{$k$-Means Clustering}
\label{sc:cluster}
%%%%%%%%%%%%%%%%%%%%%%%%%%%%%%%%%%%%%%%%%%%%%%%

Clustering is another fundamental problem in statistic inference and machine learning. In clustering, we observe one view for each sample $X_1,\ldots, X_s$ and do not observe each sample's latent variable $Z_i$ and label $Y_i$. Our goal in clustering is to recover the samples' labels by putting similar samples together, i.e., find an estimator $\hat{Y}_i$ that is as close to the true labels as possible. One of the most popular clustering algorithms is perhaps $k$-means clustering \citep{macqueen1967some,lloyd1982least,lu2016statistical}. In $k$-means clustering, we need to choose a distance $D$ and then minimize the objective function
$$
\min_{\mu_+,\mu_-}\sum_{i=1}^s \min \left(D(X_i,\mu_+),D(X_i,\mu_-)\right).
$$
We can then assign $\hat{Y}_i=1$ if $D(X_i,\mu_+)<D(X_i,\mu_-)$, and $\hat{Y}_i=-1$ otherwise. $k$-means clustering adopts the following iterative two steps to minimize the objective function: 
\begin{enumerate}
	\item Update centroid of each cluster
	$$
	\hat{\mu}_+^{(t)}=\argmin_{\mu_+}\sum_{\hat{Y}_i^{(t)}=1}D(X_i,\mu_+)\qquad {\rm and}\qquad \hat{\mu}_-^{(t)}=\argmin_{\mu_-}\sum_{\hat{Y}_i^{(t)}=-1}D(X_i,\mu_-).
	$$
	\item Assign each sample to its closest centroid
	$$
	\hat{Y}_i^{(t+1)}=\argmin_{+1,-1} \left(D(X_i,\hat{\mu}_+^{(t)}),D(X_i,\hat{\mu}_-^{(t)})\right).
	$$
\end{enumerate} 
When Euclidean distance is used in $k$-means clustering, the centroid update in step 1 is just the sample mean within each group as it is the unique minimal point in the optimization problem. However, the minimal point can be non-unique when the target distance in self-supervised metric learning is used, because $D^\ast$ and $D^{\ast\ast}$ only measure distance along $K$ directions. In such case, we still use the sample mean as the centroid update since it is one of minimal point. 

To compare the performance of different distance $D$ on $k$-means, we consider Gaussian mixture model, which is one of the most widely used and well-studied models for clustering. Specifically, we make the following assumptions. 
\begin{assumption}
	\label{ap:kmeans}
	It holds that
	\begin{enumerate}[label=(\alph*)]
		\item we assume $\PP(Y=1)=\PP(Y=-1)=1/2$;
		\item we assume $\EE(X|Y=1)-\EE(X|Y=-1)=\mu=B\theta$ for some vector $\mu\in \RR^d$, where $\theta\in\RR^K$;
		\item we assume $\epsilon_{i,j}\sim N(0,\Sigma)$ in factor model \eqref{eq:factormodel}, where $\Sigma$ is the covariance matrix;
		\item we assume the latent variable $Z_i|Y_i=1\sim N(\theta/2,I_K-\theta\theta^T/4)$ and $Z_i|Y_i=-1\sim N(-\theta/2,I_K-\theta\theta^T/4)$ in factor model \eqref{eq:factormodel};
		\item we assume the mis-clustering rate of initial assignment is smaller than $h$ for some $h<1/2$, i.e.,
		$$
		{1\over s}\min\left(\sum_{i=1}^s\bI\left(\hat{Y}^{(1)}_i\ne Y_i\right), \sum_{i=1}^s\bI\left(\hat{Y}^{(1)}_i\ne -Y_i\right)\right)<h;
		$$
		\item we assume $s>d$.
	\end{enumerate}
\end{assumption}
The conditions in Assumption~\ref{ap:kmeans} implies $\Var(X_i|Y_i=1)=\Var(X_i|Y_i=-1)=\Sigma+BB^T-\mu\mu^T/4$. For simplicity, we write $\Sigma_\pm=\Sigma+BB^T-\mu\mu^T/4$. To quantify the performance of $k$-means, we adopt the mis-clustering rate of $\hat{Y}_1,\ldots, \hat{Y}_s$ as our measure
$$
r(D)={1\over s}\min\left(\sum_{i=1}^s\bI\left(\hat{Y}_i\ne Y_i\right), \sum_{i=1}^s\bI\left(\hat{Y}_i\ne -Y_i\right)\right).
$$
To quantify the mis-clustering rate of $k$-means, we define the following quantities for any vector $\mu$ and symmetrical matrix $\Sigma_0$
$$
\Psi(\Sigma_0)=\|\Sigma_0\|+{{\rm Tr}(\Sigma_0)\over s}\qquad {\rm and }\qquad \Gamma(l,\mu,\Sigma_{0})=\PP\left(l\left\|\mu\right\|^2 \le 2\langle \xi_1,\mu+ \xi_2-\xi_3\rangle\right),
$$
where $\|\Sigma_0\|$ is spectral norm of $\Sigma_0$, ${\rm Tr}(\Sigma_0)$ is trace of $\Sigma_0$, $0\le l\le 1$, and $\xi_1$, $\xi_2$ and $\xi_3$ are independent random variable such that $\xi_1\sim N(0,\Sigma_0)$, $\xi_2\sim N(0,\Sigma_0/s_+)$, and $\xi_3\sim N(0,\Sigma_0/s_-)$, where $s_+=|\{i:Y_i=1\}|$ and $s_-=|\{i:Y_i=-1\}|$. The following theorem characterizes the performance of $k$-means clustering algorithm when Euclidean distance, target distances $D^\ast$ and $D^{\ast\ast}$ are used.  

%\max\left(\mu^T\Sigma_{0}\mu,{\|\Sigma_{0}\|_F^2\over s},{\|\mu\|^2\|\Sigma_{0}\|\over \sqrt{s}}\right)

%We usually expect the mis-clustering rate to be different when different distances are adopted in $k$-means clustering. However, the following theorem shows that $k$-means behaves similarly when Euclidean distance and target distance $D^\ast$ are used. 

\begin{theorem}
	\label{thm:kmeans}
	Suppose assumptions in Section~\ref{sc:model} and Assumption~\ref{ap:kmeans} hold. Let $v$ be a sequence of number going to infinity. If $t>\log s$ and $\|\mu\|^2>v\Psi(BB^T+\Sigma_\pm)$, then 
	$$
	r(\|\cdot\|^2) \le \Gamma(1+o(1),\mu,\Sigma_{\pm})
	$$
	with probability at least $1-s^{-5}-\exp(-\sqrt{v}\|\mu\|)$. Similarly, if $t>\log s$ and $\|B^T\mu\|^2>v\Psi(\Lambda^2+B^T\Sigma_\pm B)$ or $\|\mu\|^2>v\Psi(\Lambda+U^T\Sigma_\pm U)$, then 
	$$
	r(D^\ast) \le \Gamma(1+o(1),B^T\mu,B^T\Sigma_{\pm}B) \qquad {\rm and}\qquad r(D^{\ast\ast}) \le \Gamma(1+o(1),\mu,U^T\Sigma_{\pm}U)
	$$
	with probability at least $1-s^{-5}-\exp(-\sqrt{v}\|\mu\|)$. 
	
	On the other hand, if $\|\mu\|^2>v\Psi(BB^T+\Sigma_\pm)$, we have 
	$$
	\sup_{(Y_1,\ldots,Y_s)\in \{-1,1\}^s }r(\|\cdot\|^2) \ge \Gamma(1+o(1),\mu,\Sigma_{\pm}).
	$$
	In addition, when $\|B^T\mu\|^2>v\Psi(\Lambda^2+B^T\Sigma_\pm B)$ or $\|\mu\|^2>v\Psi(\Lambda+U^T\Sigma_\pm U)$, we have
	$$
	\sup_{(Y_1,\ldots,Y_s)\in \{-1,1\}^s }r(D^\ast) \ge \Gamma(1+o(1),B^T\mu,B^T\Sigma_{\pm}B)
	$$
	and
	$$
	\sup_{(Y_1,\ldots,Y_s)\in \{-1,1\}^s }r(D^{\ast\ast}) \ge \Gamma(1+o(1),\mu,U^T\Sigma_{\pm}U).
	$$
\end{theorem}

Theorem~\ref{thm:kmeans} suggests that the performance of $k$-means is fully characterized by the quantity $\Gamma(1,\mu,\Sigma_{0})$. Depending on $\mu$ and $\Sigma_{0}$, the behavior of $\Gamma(1,\mu,\Sigma_{0})$ can be very different. More concretely, if $\|\Sigma_{0}\|_F^2/s=o(\mu^T\Sigma_0\mu)$, then
$$
\Gamma(1,\mu,\Sigma_0)\le \exp\left(-{\|\mu\|^4\over 8\mu^T\Sigma_0\mu}\right).
$$
On the other hand, when $\mu^T\Sigma_0\mu=o(\|\Sigma_{0}\|_F^2/s)$, Lemma~1 suggests
$$
\Gamma(1,\mu,\Sigma_0)\le \begin{cases}
	\exp\left(-\dfrac{s\|\mu\|^4}{8\|\Sigma_0\|_F^2}\right),& \qquad \|\mu\|^2\|\Sigma_{0}\|=o\left(\|\Sigma_{0}\|_F^2/ \sqrt{s}\right)\\[10pt]
	\exp\left(-\dfrac{\sqrt{s}\|\mu\|^2}{2\|\Sigma_0\|}\right),& \qquad  \|\Sigma_{0}\|_F^2/ \sqrt{s}=o\left(\|\mu\|^2\|\Sigma_{0}\|\right).
\end{cases}
$$
If we consider a special case $\Sigma_0=\sigma^2I$, $\Gamma(1,\mu,\sigma^2I)=\exp\left(-{\|\mu\|^2/8\sigma^2}\right)$  recovers the results in \cite{lu2016statistical}. The intuition behind the lower bound in Theorem~\ref{thm:kmeans} is that even we have perfect initialization, that is, $\hat{Y}^{(1)}_i=Y_i$ $i=1,\ldots, s$, $\Gamma(1,\mu,\Sigma_\pm)$ is the mis-clustering rate we can expect after one iteration in $k$-means.

To compare the results in Theorem~\ref{thm:kmeans}, we assume $\Sigma=\sigma^2 I$. The mis-clustering rate and required signal are summarized in Table~\ref{tb:kmeanssummary} when different distances are used. The main benefit of metric learning in $k$-means clustering is that the required condition for convergence becomes weaker because variation between different views is reduced. Specifically, $k$-means with Euclidean distance requires $\lambda_1+\sigma^2+(\sum_{k=1}^K\lambda_k+d\sigma^2)/s$, which relies on the dimension $d$, while $k$-means with target distance $D^{\ast\ast}$ only requires $\lambda_1+\sigma^2+(\sum_{k=1}^K\lambda_k+K\sigma^2)/s$. 

\begin{table}[h!]
	\centering
	\renewcommand{\arraystretch}{1.6}
	\begin{tabular}{ccc}
		\hline\hline
		Measure & Mis-Clustering Rate  & Required Signal  \\ 
		\hline
		Euclidean Distance & $\begin{aligned}\exp\left(-{\|\mu\|^2\over 8(\lambda_1+\sigma^2)}\right) \end{aligned}$ & $\begin{aligned}\lambda_1+\sigma^2+{\sum_{k=1}^K\lambda_k+d\sigma^2\over s}\end{aligned}$  \\ 
		Distance $D^\ast$ & $\begin{aligned}\exp\left(-{\|\mu\|^2\over 8\kappa^2(\lambda_1+\sigma^2)}\right) \end{aligned}$& $\begin{aligned}\kappa\left(\lambda_1+\sigma^2+{\sum_{k=1}^K\lambda_k+K\sigma^2\over s}\right)\end{aligned}$  \\ 
		Distance $D^{\ast\ast}$ & $\begin{aligned}\exp\left(-{\|\mu\|^2\over 8(\lambda_1+\sigma^2)}\right) \end{aligned}$& $\begin{aligned}\lambda_1+\sigma^2+{\sum_{k=1}^K\lambda_k+K\sigma^2\over s}\end{aligned}$  \\ 
		\hline\hline
	\end{tabular}
	\caption{Performance comparisons on $k$-means clustering.}
	\label{tb:kmeanssummary}
\end{table}

%%%%%%%%%%%%%%%%%%%%%%%%%%%%%%%%%%%%%%%%%%%%%%%
\subsection{Sample Identification}
\label{sc:identification}
%%%%%%%%%%%%%%%%%%%%%%%%%%%%%%%%%%%%%%%%%%%%%%%

In sample identification, we observe two views $X_{1}$ and $X_{2}$, and aim to determine if these two views come from the same sample or not. Different from two-sample testing and classification problem, sample identification does not involve any label information. One popular example of sample identification is face identification, where we would like to know if the faces from two different images are the same person or not \citep{guillaumin2009you,nguyen2010cosine,liao2015person}. Although this problem has different formulation, we study it from a hypothesis testing perspective. More concretely, let $Z_1$ and $Z_2$ be the latent variables of $X_1$ and $X_2$, respectively, and we assume $(X_i,Z_i)$, $i=1,2$ follows the same distribution as multi-view data (think we only observe one view from the sample). $Z_1$ and $Z_2$ are random in multi-view data model, but our investigation are conditioned on the value of $Z_1$ and $Z_2$. In other words, we think that $Z_1$ and $Z_2$ are determined values and $Z_1=Z_2$ if $X_1$ and $X_2$ are two different views of the same sample. Then, the hypothesis of interest in sample identification is 
$$
H_0: Z_1=Z_2\qquad {\rm and}\qquad H_1:Z_1\ne Z_2.
$$
To test such a hypothesis, one of the most popular methods is the distance-based method. Specifically, we choose a distance between different views, $D(X_1,X_2)$, as the statistics and the null hypothesis is rejected when the distance is larger than a given threshold $T_D$. The threshold $T_D$ can be estimated by the unlabeled multi-view data alone. For example, we can choose $T_D$ as upper $\alpha$-quantile of $D(X_{i,1},X_{i,2})$, $i=1,\ldots,m$. As the sample size of unlabeled multi-view data is usually large, we can estimate the threshold accurately or assume it is known. Then, the test used in distance-based sample identification can be written as
$$
\phi_D=\bI(D(X_1,X_2)>T_D).
$$
Different choices of distance can lead to different performances. To quantify the performance, we adopt the detection radius of tests to compare different distances, defined as follows
$$
r(D,\epsilon)=\inf\left\{r:\underbrace{\PP(\phi_D=1|H_0)}_{type\  I\  error}+\underbrace{\PP(\phi_D=0|H_1(r))}_{type\  II\  error}\le \epsilon\right\}
$$
where $H_1(r)=\{\|Z_1-Z_2\|\ge r\}$. We also make the following assumptions.
\begin{assumption}
	\label{ap:identification}
	It holds that
	\begin{enumerate}[label=(\alph*)]
		\item we assume $\epsilon_{i,j}\sim N(0,\Sigma)$ in factor model \eqref{eq:factormodel}, where $\Sigma$ is the covariance matrix;
		\item we choose $T_D$ as upper $\alpha=\epsilon/2$-quantile of $D(X_{i,1},X_{i,2})$.
	\end{enumerate}
\end{assumption}
Condition (a) in Assumption~\ref{ap:identification} is a relatively strong condition in practice, but it can help provide some insights into how metric learning benefits sample identification problem. With Assumption~\ref{ap:identification}, we now compare the detection radius of tests defined by standard Euclidean distance $\|\cdot\|^2$ and target distance in self-supervised metric learning, $D^\ast$ and $D^{\ast\ast}$.

\begin{theorem}
	\label{thm:identification}
	Suppose assumptions in Section~\ref{sc:model} and Assumption~\ref{ap:identification} hold. Given $\epsilon>0$, we have
	$$
	r(\|\cdot\|^2,\epsilon)\lesssim {\|\Sigma \|_F^{1/2}\over \sqrt{\lambda_K}},\quad r(D^\ast,\epsilon)\lesssim {\|B^T\Sigma B\|_F^{1/2}\over \lambda_K}\quad {\rm and }\quad  r(D^{\ast\ast},\epsilon)\lesssim {\|U^T\Sigma U\|_F^{1/2}\over \sqrt{\lambda_K}}.
	$$
	Consider the following local alternative hypothesis
	$$
	\tilde{H}_1(r)=\left\{Z_1-Z_2=re_K\right\},
	$$
	where $e_K=(0,\ldots,0,1)$. If $r=o(\|\Sigma \|_F^{1/2}/\sqrt{\lambda_K})$, then 
	$$
	\PP(\phi_{\|\cdot\|^2}=0|\tilde{H}_1(r))\to 1-\alpha.
	$$
	Similarly, if $r=o(\|B^T\Sigma B\|_F^{1/2}/\lambda_K)$ or $r=o(\|U^T\Sigma U\|_F^{1/2}/ \sqrt{\lambda_K})$, then
	$$
	\PP(\phi_{D^\ast}=0|\tilde{H}_1(r))\to 1-\alpha\qquad {\rm and}\qquad \PP(\phi_{D^{\ast\ast}}=0|\tilde{H}_1(r))\to 1-\alpha.
	$$
\end{theorem}
The two parts in Theorem~\ref{thm:identification} suggests that the detection radius are tight. The results of Theorem~\ref{thm:identification} heavily rely on condition (a) in Assumption~\ref{ap:identification} and could be very different if $\epsilon_{i,j}$ follows different distributions. Despite this, Theorem~\ref{thm:identification} still helps understand and characterize the performance of these different distances on sample identification problem. If $\Sigma=\sigma^2 I$, the results in Theorem~\ref{thm:identification} are reduced to
$$
r(\|\cdot\|^2,\epsilon)\lesssim {d^{1/4}\sigma\over \sqrt{\lambda_K}},\quad r(D^\ast,\epsilon)\lesssim \sqrt{\kappa}{K^{1/4}\sigma\over \sqrt{\lambda_K}}\quad {\rm and}\quad r(D^{\ast\ast},\epsilon)\lesssim {K^{1/4}\sigma\over \sqrt{\lambda_K}}.
$$
In particular, when the target distance $D^\ast$ and $D^{\ast\ast}$ are used, the detection radius can be improved by $O(d/K\kappa^2)^{1/4}$ and $O(d/K)^{1/4}$ folds. With this target distance in self-supervised metric learning, we can detect much more similar sample pairs.  

%%%%%%%%%%%%%%%%%%%%%%%%%%%%%%%%%%%%%%%%%%%%%%%
\section{Proofs}
%%%%%%%%%%%%%%%%%%%%%%%%%%%%%%%%%%%%%%%%%%%%%%%

In this section, $C$ and $c$ refer to some constant, which can be different at different places. 

%%%%%%%%%%%%%%%%%%%%%%%%%%%%%%%%%%%%%%%%%%%%%%%
\subsection{Proof of Theorem~\ref{thm:idealdist}}
%%%%%%%%%%%%%%%%%%%%%%%%%%%%%%%%%%%%%%%%%%%%%%%

We first show $U^TX$ is actually a sufficient statistics of $Z$. Since the conditional distribution of $X$ given $Z$ follow a factor model, we have 
$$
X=BZ+\epsilon,
$$
where $\epsilon$ is independent from $Z$. This observation suggests $X-B(B^TB)^{-1}B^TX$ is only a function of $\epsilon$, which implies that $(X-B(B^TB)^{-1}B^TX)\perp Z$. Note that $X-B(B^TB)^{-1}B^TX=(I-B(B^TB)^{-1}B^T)\epsilon$. Since $(I-B(B^TB)^{-1}B^T)\epsilon$ is independent from $B^T\epsilon$, we can conclude 
$$
X\perp Z|B^TX.
$$
This means $U^TX$ is actually a sufficient statistics of $Z$. 

By the Fisher's factorization theorem, the conditional probability density function $f(x|z)$ can be decomposed as 
$$
f(x|z)=h_1(x)h_2(U^Tx|z).
$$
Moreover, we have $f(x|z,y)=f(x|z)=h_1(x)h_2(U^Tx|z)$ as $X\perp Y|Z$. If we write $\pi(z|y)$ as the probability density function of $z$ given $y$, we can conclude
\begin{align*}
	\pi(X=x|Y=y)&=\int f(x|z,y) \pi(z|y)dz\\
	&=\int h_1(x)h_2(U^Tx|z) \pi(z|y)dz\\
	&=h_1(x) g'(U^Tx|y),
\end{align*}
where $g'(U^Tx|y)=\int h_2(U^Tx|z) \pi(z|y)dz$. Thus, the likelihood ratio $\pi(X|Y=1)/\pi(X|Y=-1)=g(U^TX)$ for some function $g$. 

By the definition,
$$
\EE(X|Y)=\EE(\EE(X|Z,Y)|Y)=\EE(\EE(X|Z)|Y)=\EE(BZ|Y)=B\EE(Z|Y).
$$
Thus, if we choose $\theta=\EE(Z|Y=1)-\EE(Z|Y=-1)$, we can know $\EE(X|Y=1)-\EE(X|Y=-1)=B\theta$. 

Next, we prove the second part of theorem. For any given $B$, we assume $\PP(Y=1)=\PP(Y=-1)=1/2$, $\epsilon$ follows a normal distribution $N(0,\sigma^2I)$ and $Z|Y$ follows a normal distribution $N(Y\theta/2,\Sigma)$, where $\Sigma=I-\theta\theta^T/4$. Clearly, assumptions \eqref{eq:factormodel} and \eqref{eq:condindep} are satisfied and $\EE(X|Y=1)-\EE(X|Y=-1)=B\theta$. 

%%%%%%%%%%%%%%%%%%%%%%%%%%%%%%%%%%%%%%%%%%%%%%%
\subsection{Proof of Theorem~\ref{thm:knn}}
%%%%%%%%%%%%%%%%%%%%%%%%%%%%%%%%%%%%%%%%%%%%%%%

\subsubsection{Upper bound}
In this proof, we write $\hat{\eta}(x)=k^{-1}\sum_{i=1}^k \bI(Y_{(i)}=1)$ and $\hat{\eta}^\ast(x)=k^{-1}\sum_{i=1}^k\eta(X_{(i)})$.
\paragraph{Step 1: Euclidean distance}
By Hoeffding's inequality, we can know that for any fixed point $x$, 
$$
\PP\left(\left|\hat{\eta}(x)-\hat{\eta}^\ast(x)\right|>t\right)\le 2e^{-2kt^2},\qquad \forall\ t>0.
$$
Moreover, if we write $\Bcal_{\|\cdot\|^2}(x, r_{2k/s})$ as ball centered at $x$ such that $\mu(\Bcal_{\|\cdot\|^2}(x, r_{2k/s}))=2k/s$, then an application of Chernoff bound suggests 
$$
\PP(X_{(k+1)}>r_{2k/s})=\PP\left(\sum_i\bI(X_i\in \Bcal_{\|\cdot\|^2}(x, r_{2k/s}))\le k\right)\le e^{-k/4}.
$$
When $X_{(k+1)}\le r_{2k/s}$, we can know that 
$$
|\hat{\eta}^\ast(x)-\eta(x)|\le \sup_{\|\delta\|<r_{2k/s}}\left|\eta(x+\delta)-\eta(x)\right|\le Lr_{2k/s}^\alpha
$$
since $\eta(x)$ is $\alpha$-H\"older continuous as $\tilde{\eta}(x)$ is $\alpha$-H\"older continuous. 
The choice of $r_{2k/s}$ suggests that $r_{2k/s}\le C(k/s)^{1/d}$ as the density $\mu(x)$ is bounded away from 0 on the support, i.e. $\mu(x)\ge \mu_{\rm min}$. Therefore, we can know that 
$$
\PP\left( |\hat{\eta}^\ast(x)-\eta(x)|>C\left(k\over s\right)^{\alpha/d}\right)\le e^{-k/4}. 
$$
Putting $\left|\hat{\eta}(x)-\hat{\eta}^\ast(x)\right|$ and $|\hat{\eta}^\ast(x)-\eta(x)|$ together yields 
\begin{equation}
	\label{eq:knncon}
	\PP\left(\left|\hat{\eta}(x)-\eta(x)\right|>t+C\left(k\over s\right)^{\alpha/d}\right)\le 2e^{-2kt^2}+e^{-k/4},\qquad \forall\ t>0.
\end{equation}
We write $\Delta=1/\sqrt{k}+C\left(k/s\right)^{\alpha/d}=Cs^{-\alpha/(2\alpha+d)}$, $A_0=\{x:0<|\eta(x)-1/2|<\Delta\}$ and $A_j=\{x:2^{j-1}\Delta<|\eta(x)-1/2|<2^j\Delta\}$ for $j=1,\ldots, J:=\lceil -\log(\Delta)/\log 2\rceil$. By definition, 
\begin{align*}
	r(\|\cdot\|^2)&=\EE\left(|2\eta(X)-1|\bI(\hat{f}_{\|\cdot\|^2}(X)\ne f^\ast(X))\right)\\
	&=\sum_{j=0}^J\EE\left(|2\eta(X)-1|\bI(\hat{f}_{\|\cdot\|^2}(X)\ne f^\ast(X))\bI(X\in A_j)\right)
\end{align*}
Since $\bI(\hat{f}_{\|\cdot\|^2}(X)\ne f^\ast(X))\le \bI(|\eta(X)-1/2|<\left|\hat{\eta}(X)-\eta(X)\right|)$, we have 
\begin{align*}
	&\EE\left(|2\eta(X)-1|\bI(\hat{f}_{\|\cdot\|^2}(X)\ne f^\ast(X))\bI(X\in A_j)\right)\\
	\le & 2^{j+1}\Delta\EE\left(\bI(|\eta(X)-1/2|<\left|\hat{\eta}(X)-\eta(X)\right|)\bI(X\in A_j)\right)\\
	\le & 2^{j+1}\Delta\EE\left(\bI(2^{j-1}\Delta<\left|\hat{\eta}(X)-\eta(X)\right|)\bI(X\in A_j)\right)\\
	\le & 2^{j+1}\Delta\EE\left(\PP(2^{j-1}\Delta<\left|\hat{\eta}(X)-\eta(X)\right|)\bI(X\in A_j)\right)\\
	\le & 2^{j+1}\Delta\EE\left((2e^{-2k(2^{j-1}\Delta)^2}+e^{-k/4})\bI(X\in A_j)\right)\\
	\le & 2^{j+1}\Delta 2e^{-2k(2^{j-1}\Delta)^2} (2^{j+1}\Delta)^\beta+e^{-k/4}2^{j+1}\Delta\PP(X\in A_j)
\end{align*}
Here we apply the results in \eqref{eq:knncon} and $\beta$-marginal assumption.
Putting these terms together, we can know that 
$$
r(\|\cdot\|^2)\le C\Delta^{\beta+1}\le Cs^{-\alpha(\beta+1)/(2\alpha+d)}.
$$
\paragraph{Step 2: Mahalanobis distance} If we adopt the distance $D^\ast$, the main difference is the shape of the ball $\Bcal_{D^\ast}(x, r_{2k/s})=\{y:D^\ast(x,y)\le r_{2k/s}\}$. 
We can choose $r_{2k/s}\le C\lambda_K(\kappa^{K-1}k/s)^{1/K}$ to make sure $\mu(\Bcal_{D^\ast}(x, r_{2k/s}))=2k/s$, which leads to
$$
|\hat{\eta}^\ast(x)-\eta(x)|\le \sup_{y\in \Bcal_{D^\ast}(x, r_{2k/s})}\left|\eta(y)-\eta(x)\right|\le L(\kappa^{K-1}k/s)^{\alpha/K}.
$$
Then we can adopt the same analysis in Euclidean distance case to show
$$
r(D^\ast) \le  C(s/\kappa^{K-1})^{-\alpha(1+\beta)/(2\alpha+K)}.
$$
With the same analysis, we can know
$$
r(D^{\ast\ast}) \le  Cs^{-\alpha(1+\beta)/(2\alpha+K)}.
$$

\subsubsection{Lower bound}
We now work on the lower bound. The main idea of lower bound proof is to construct difficult instances and then apply Lemma~\ref{lm:knnlowerbayes} and \ref{lm:knnlower}. 
\paragraph{Step 1: Euclidean distance} To construct difficult instances for multi-view data, we need to choose the distribution for $Z$, the conditional distribution of $Y$ given $Z$, the latent factor $B$, and the distribution for $\epsilon$. For simplicity, we choose $\epsilon$ as a uniform distribution on $\{0\}^K\times [0,1]^{d-K}$ and $B$ is the first $K$ basis in $\RR^d$, i.e., $B=[I_K,0^{K\times (d-K)}]^T$, so $X=(Z_1,\ldots,Z_K,u_1,\ldots,u_{d-K})$ where $u_i$ is uniform distribution on $[0,1]$ and $\tilde{\eta}$ is the same with the conditional distribution of $Y$ given $Z$. To choose the distribution for $Z$ and $\tilde{\eta}$, we split $[0,1]^K$ into $q^K$ non-overlap cubes of size $q^{-1}\times\ldots\times q^{-1}$, where $q$ is an integer that will be specified later. We name these small cubes $Q_1,\ldots,Q_{q^K}$. Let $h$ be a nonincreasing infinitely differentiable function defined on $[0,\infty]$
$$
h(w)=\begin{cases}
	\int_{1/4}^{1/2} \exp(-1/(1/2-t)(t-1/4))dt,& 0\le w\le {1/4}\\
	\int_{w}^{1/2} \exp(-1/(1/2-t)(t-1/4))dt,& 1/4\le w\le {1/2}\\
	0,& w\ge {1/2}\\
\end{cases}
$$
and $\phi(z)=C_\phi h(\|z\|)$ with a sufficient small constant $C_\phi$ such that $|\phi(z)-\phi(z')|\le L\|z-z'\|^\alpha$ for any $z,z'\in \RR^K$. It is clear that $\phi(z)=C_M$ when $\|z\|\le 1/4$ where $C_M=C_\phi \int_{1/4}^{1/2} \exp(-1/(1/2-t)(t-1/4))dt$. We pick the first $M$ cubes $Q_1,\ldots,Q_{M}$ where $M< q^K$ and insert a scaled and shifted version of $\phi(z)$ to $\tilde{\eta}(z)$ at each cube. Specifically, we define 
$$
\tilde{\eta}_\sigma(z)={1\over 2}+\sum_{i=1}^M q^{-\alpha}\sigma_i\phi\big(q(z-cen(Q_i))\big),
$$
where $cen(\cdot)$ is center point of the cube and $\sigma=(\sigma_1,\ldots,\sigma_m)$ is a sequence taking value from $\{-1,1\}^M$. Clearly, $\tilde{\eta}_\sigma(z)$ is an $\alpha$-H\"older continuous function. Here, we write the corresponding probability density function of $Z$ as $\lambda$. We define $\lambda$ as 
$$
\lambda(z)={v\over C_q}\sum_{i=1}^M\bI\big(\|z-cen(Q_i)\|\le 1/4q \big)+{1-Mv\over C_q(q^K-M)}\sum_{i=M+1}^{q^K}\bI\big(\|z-cen(Q_i)\|\le 1/4q \big),
$$
where $C_q$ is Lebesgue measure of the ball with radius $1/4q$. By these construction, we know that  
$$
\mu(0<|\eta(X)-1/2|\le t)={Mv}\bI(t\ge C_M/q^\alpha).
$$
So $\beta$-marginal assumption is satisfied as long as $Mv\le C_0(C_M/q^\alpha)^\beta$. 

Now we are going to choose $q$, $\sigma$, $M$ and $v$ in above construction according to different $k$ and $s$ and then apply Lemma~\ref{lm:knnlowerbayes} and \ref{lm:knnlower}. We consider two cases. In the first case, given $k$ and $s$, we can choose $q=\lceil (C_M\sqrt{k})^{1/\alpha} \rceil$ so that $C_M/q^\alpha\le 1/\sqrt{k}$. We then choose $\sigma_1=\ldots=\sigma_m=1$, $v=\mu_{\rm min}C_q$ and $M=\min(\lceil C_0(C_M/q^\alpha)^\beta/(\mu_{\rm min}C_q)\rceil,q^K)$. By these choices, we can know that 
$$
\Bcal_{\|\cdot\|^2}(cen(Q_i),1/4q)\subset \Ecal^+\left(\delta,{1\over \sqrt{k}},{k\over s}\right),\qquad i=1,\ldots, M.
$$
This leads to
$$
\mu\left(\Ecal^+\left(\delta,{1\over \sqrt{k}},{k\over s}\right)\cup \Ecal^-\left(\delta,{1\over \sqrt{k}},{k\over s}\right)\right)\ge Mv,
$$
where $\Ecal^+$ and $\Ecal^-$ are defined in Lemma~\ref{lm:knnlower}.
Therefore, an application of Lemma~\ref{lm:knnlower} suggests
$$
r(\|\cdot\|^2)\ge {c_0Mv\over \sqrt{k}}\ge {C\over k^{(\beta+1)/2}}.
$$
In the second case, we consider $\sigma$ is drawn from a uniform distribution on $\{-1,1\}^M$. We denote the corresponding distribution of $\eta(x)$ by $f_\eta$ when $\sigma$ is chosen in above way. Given $k$ and $s$, we choose $v=q^{-K}$ and $q=\lceil 3/2(\Gamma_ds/k)^{1/d}\rceil$, where $\Gamma_d$ is volume of unit ball in $\RR^d$. This choice of $q$ can ensure that there are at least $3^K$ small cubes $Q_i$ in $\Bcal_{D}(x,r)$ when $r_{k/s}\le r\le r_{2{k/s}}$ and $x\in \Bcal_{\|\cdot\|^2}(cen(Q_i),1/4q)$. Since each $\sigma_i$ is independent from each other in $f_\eta$, we know that there exists a constant $c_K$ relying $K$ such that
$$
\PP_{\eta\sim f_\eta}\left(\eta(x)\ge{1\over 2}+{C_M\over q^\alpha}; \eta(\Bcal_{D}(x,r))\le {1\over 2},\ \forall\ r_\nu\le r\le r_{2\nu} \right)>c_K
$$
for any $x\in \Bcal_{\|\cdot\|^2}(cen(Q_i),1/4q)$. This means 
$$
\mu\left(\Ecal_{f_\eta}^+\left(\delta,{k\over s}\right)\cup \Ecal_{f_\eta}^-\left(\delta,{k\over s}\right)\right)\ge Mv.
$$
If we choose $M=\min(\lceil C_0(C_M/q^\alpha)^\beta/v\rceil,q^K)$, an application of Lemma~\ref{lm:knnlowerbayes} suggests 
$$
\EE_{\eta\sim f_\eta}\left(r(\|\cdot\|^2)\right)\ge {2c_0C_0C_M^{\beta+1}\over q^{\alpha(\beta+1)}}\ge C\left(k\over s\right)^{\alpha(\beta+1)/d}.
$$
Putting the results of two cases together yields
$$
\sup_{\eta\sim f_\eta}r(\|\cdot\|^2)\ge c\left({1\over k^{(\beta+1)/2}}+\left(k\over s\right)^{\alpha(\beta+1)/d}\right)
$$
and 
$$
\min_{k} \sup_{\eta\sim f_\eta}r(\|\cdot\|^2)\ge cs^{-\alpha(1+\beta)/(2\alpha+d)}.
$$
\paragraph{Step 2: Mahalanobis distance} We can conduct the similar analysis as in the case of Euclidean distance if we adopt $D^\ast$ and $D^{\ast\ast}$. We first work on $D^\ast$. The main difference from the case of Euclidean distance is that we choose $B=[\Lambda,0^{K\times (d-K)}]^T$, where $\Lambda={\rm diag}(\lambda_1,\ldots, \lambda_K)$, and $\tilde{\eta}(U^Tx)=g(\Lambda z)$ where $g$ is the conditional distribution of $Y$ given $Z$. With these new choices, the marginal distribution of $(X,Y)$ is still the same as the case of Euclidean distance, but $D^\ast$ put different weights to different directions. In particular, we choose $\lambda_1=\kappa$ and $\lambda_{2}=\ldots=\lambda_{K}=1$. We still consider the two cases as we did in the case of Euclidean distance. For the first case, we can choose the same $q$, $\sigma$, $M$ and $v$ and obtain 
$$
r(D^\ast)\ge {C\over k^{(\beta+1)/2}}.
$$
In the second case, we still consider $\sigma$ is drawn from a uniform distribution on $\{-1,1\}^M$. The main difference is that we choose $q=\lceil 3/2(\Gamma_Ks/\kappa^{K-1}k)^{1/K}\rceil$ due to the shape of neighbor is different. Then we apply the similar analysis to obtain 
$$
\EE_{\eta\sim f_\eta}\left(r(D^\ast)\right)\ge C\left(\kappa^{K-1}k\over s\right)^{\alpha(\beta+1)/K}.
$$
Therefore, we can conclude
$$
\sup_{\eta\sim f_\eta}r(D^\ast)\gtrsim {1\over k^{(\beta+1)/2}}+\left(\kappa^{K-1} k\over s\right)^{\alpha(\beta+1)/K}\quad {\rm and}\quad \min_{k} \sup_{\eta\sim f_\eta}r(D^\ast)\gtrsim \left(s\over \kappa^{K-1}\right)^{-\alpha(1+\beta)/(2\alpha+K)}.
$$

The analysis for $D^{\ast\ast}$ is almost the same with $D^\ast$ since we only need to set $\kappa=1$. So we have 
$$
\sup_{\eta\sim f_\eta}r(D^{\ast\ast})\gtrsim {1\over k^{(\beta+1)/2}}+\left( k\over s\right)^{\alpha(\beta+1)/K}\quad {\rm and}\quad \min_{k} \sup_{\eta\sim f_\eta}r(D^{\ast\ast})\gtrsim s^{-\alpha(1+\beta)/(2\alpha+K)}.
$$

%%%%%%%%%%%%%%%%%%%%%%%%%%%%%%%%%%%%%%%%%%%%%%%
\subsection{Proof of Theorem~\ref{thm:twosample}}
%%%%%%%%%%%%%%%%%%%%%%%%%%%%%%%%%%%%%%%%%%%%%%%
In this proof, we first prove the result for energy distance test equipped with Euclidean distance and then extend the proof to other Mahalanobis distances. We write the marginal covariance matrix of $X$ as $\Sigma_{\pm}$. Without loss of generality, we assume $\EE(X|Y=1)=\mu/2$ and  $\EE(X|Y=-1)=-\mu/2$ in the following proof.

\subsubsection{Upper bound}
\paragraph{Step 1a: Euclidean distance and permutation test} We work on permutation test equipped with Euclidean distance in this step. If we choose $\alpha=\epsilon/2$ in permutation test, the type I error can be controlled at $\epsilon/2$ level nonasymptotically. So we mainly focus on type II error. Applying Markov's inequality suggests
\begin{align*}
	\PP\left(\hat{P}>\alpha\right)&=\PP\left(1+\sum_{b=1}^B\bI_{(\phi_bE(D)\ge E(D))}>(1+B)\alpha\right)\\
	&\le {1+B\PP\left(\phi_1 E(D)\ge E(D)\right)\over (1+B)\alpha}.
\end{align*}
Thus, it is sufficient to show that $\PP\left(\phi_1 E(D)\ge E(D)\right)$ is small when the difference between groups is large enough. When the Euclidean distance is used in $E(D)$, one can verify that
$$
E(D)=\underbrace{{2\over s_+(s_+-1)}\sum_{Y_i=Y_{i'}=1}X_i^TX_{i'}}_{E_1}+\underbrace{{2\over s_-(s_--1)}\sum_{Y_i=Y_{i'}=-1}X_i^TX_{i'}}_{E_2}-\underbrace{{4\over s_+s_{-}}\sum_{Y_i\ne Y_{i'}}X_i^TX_{i'}}_{E_3}.
$$
We also write $\phi_1 E_1$, $\phi_1 E_2$, and $\phi_1 E_3$ as above when labels are permuted by $\phi_1$. We first work on $\phi_1 E(D)$. To the end, we can show that 
$$
{\rm Var}(\phi_1 E_1|s_+,s_-)={4\over s_+^2(s_+-1)^2}\sum_{Y_{\phi_1(i)}=Y_{\phi_1(i')}=1}\EE(X_i^TX_{i'})^2={2\over s_+(s_+-1)}{\rm Tr}(\Sigma_{\pm}^2),
$$
$$
{\rm Var}(\phi_1 E_2|s_+,s_-)={2\over s_-(s_--1)}{\rm Tr}(\Sigma_{\pm}^2)\qquad {\rm and}\qquad {\rm Var}(\phi_1 E_3|s_+,s_-)={4\over s_+s_{-}}{\rm Tr}(\Sigma_{\pm}^2).
$$
Here, $\Sigma_{\pm}=\Sigma_+/2+\Sigma_{-}/2$. It is not hard to  verify that ${\rm Cov}(\phi_1E_1,\phi_1E_2)={\rm Cov}(\phi_1 E_1,\phi_1 E_3)={\rm Cov}(\phi_1 E_3,\phi_1 E_2)=0$. Combining all these terms yields
$$
{\rm Var}(\phi_1 E(D)|s_+,s_-)={2\over s_+(s_+-1)}{\rm Tr}(\Sigma_\pm^2)+{2\over s_-(s_--1)}{\rm Tr}(\Sigma_{\pm}^2)+{4\over s_+s_{-}}{\rm Tr}(\Sigma_\pm^2).
$$
Because $\EE(\phi_1 E(D)|s_+,s_-)=0$, and $s_+=s-s_-$ can seen drawn from binomial distribution ${\rm Bin}(s,1/2)$, law of total variance suggests 
$$
{\rm Var}( \phi_1 E(D))\le {C\over s^2}{\rm Tr}(\Sigma_\pm^2).
$$
An application of Chebyshev's inequality suggests that, for a large enough $C_\epsilon$, we have
$$
\PP\left(\phi_1 E(D)\ge {C_\epsilon\over s}\sqrt{{\rm Tr}(\Sigma_\pm^2)}\right)\le {\epsilon\alpha(1+B)-4\over 8B}.
$$

Next, we work on $E(D)$. Decompose $E_1$ as
$$
E_1=\underbrace{{2\over s_+(s_+-1)}\sum_{Y_i=Y_{i'}=1}\left(X_i-{\mu\over 2}\right)^T\left(X_{i'}-{\mu\over 2}\right)}_{E_{11}}+\underbrace{{2\over s_+}\sum_{Y_i=1}X_i^T\mu}_{E_{12}}-{\mu^T\mu\over 2}.
$$
Similarly we can also decompose $E_2$ and $E_3$
$$
E_2=\underbrace{{2\over s_-(s_--1)}\sum_{Y_i=Y_{i'}=-1}\left(X_i+{\mu\over 2}\right)^T\left(X_{i'}+{\mu\over 2}\right)}_{E_{21}}-\underbrace{{2\over s_-}\sum_{Y_i=-1}X_i^T\mu}_{E_{22}}-{\mu^T\mu\over 2}.
$$
and
$$
E_3=\underbrace{{4\over s_+s_{-}}\sum_{Y_i=1,Y_{i'}=-1}\left(X_i-{\mu\over 2}\right)^T\left(X_{i'}+{\mu\over 2}\right)}_{E_{31}}+\underbrace{{2\over s_-}\sum_{Y_i=-1}X_i^T\mu-{2\over s_+}\sum_{Y_i=1}X_i^T\mu}_{E_{32}}+{\mu^T\mu}
$$
With the similar analysis for $\phi_1 E(D)$, we can know that there exist a large constant $C'_\epsilon$ such that
$$
\PP\left(E_{11}+E_{21}-E_{31}<-{C'_\epsilon\over s}\sqrt{{\rm Tr}\left(\left(\Sigma_++\Sigma_-\right)^2\right)}\right)\le {\epsilon\alpha(1+B)-4\over 16B}.
$$
For $E_{12}+E_{22}-E_{32}$, we have
\begin{align*}
	E_{12}+E_{22}-E_{32}&={4\over s_+}\sum_{Y_i=1}X_i^T\mu-{4\over s_-}\sum_{Y_i=-1}X_i^T\mu\\
	&={4\over s_+}\sum_{Y_i=1}\left(X_i-{\mu\over 2}\right)^T\mu-{4\over s_-}\sum_{Y_i=-1}\left(X_i+{\mu\over 2}\right)^T\mu+4\mu^T\mu.
\end{align*}
The variance of $\left(X_i-{\mu/2}\right)^T\mu$ is $\mu^T\Sigma_+\mu$ when $Y_i=1$ and the variance of $\left(X_i+{\mu/2}\right)^T\mu$ is $\mu^T\Sigma_-\mu$ when $Y_i=-1$.
We can apply Chebyshev's inequality again to obtain
$$
\PP\left(E_{12}+E_{22}-E_{32}<4\mu^T\mu-{C^{''}_\epsilon\over s}\sqrt{\mu^T\left(\Sigma_++\Sigma_-\right)\mu}\right)\le {\epsilon\alpha(1+B)-4\over 16B}
$$
for a large enough constant $C^{''}_\epsilon$. This suggest that we have $\PP\left(\phi_1 E(D)\ge E(D)\right)<{(\epsilon\alpha(1+B)-4)/4B}$ if 
$$
2\mu^T\mu>{C_\epsilon\over s}\sqrt{{\rm Tr}(\Sigma_\pm^2)}+{C^{''}_\epsilon\over s}\sqrt{\mu^T\left(\Sigma_++\Sigma_-\right)\mu}+{C'_\epsilon\over s}\sqrt{{\rm Tr}\left(\left(\Sigma_++\Sigma_-\right)^2\right)}. 
$$
Because $\Sigma_\pm=\Sigma_+/2+\Sigma_-/2+\mu\mu^T/4$, we can know that the sufficient condition for $\PP\left(\hat{P}>\alpha\right)\le \epsilon/2$ is 
$$
\|\mu\|^2\ge {C\over s}\sqrt{{\rm Tr}\left((\Sigma_++\Sigma_{-})^2\right) }= {C\over s}\left\|\Sigma_++\Sigma_{-}\right\|_F
$$
for a large enough constant $C$.

\paragraph{Step 1b: Euclidean distance and asymptotic distribution} Instead of using permutation test, we derive asymptotic distribution for $E(D)$ under null distribution in this step. The idea is to apply central limit theorem for $U$-statistics introduced in \cite{hall2014martingale,hall1984central}. To simplify the analysis, we define $K_{ij}=\gamma_{ij}X_i^TX_j$, where
$$
\gamma_{ij}=\begin{cases}
	1/s_+(s_+-1) & Y_i,Y_j=1\\
	1/s_-(s_--1) & Y_i,Y_j=-1\\\
	-2/s_+s_- & i\le Y_i=-1,Y_j=1\ {\rm or}\ Y_i=1,Y_j=-1,
\end{cases}
$$
$V_j=\sum_{i<j}K_{ij}$, $R_j=\sum_{i=1}^jV_i$ and $\Fcal_j=\{X_1,\ldots,X_j\}$ for $1\le j\le s$. It is clear that $E(D)=2R_s$. Because $\EE(K_{ij}|\Fcal_{j-1})=0$ when $i<j$, $\{R_j;\Fcal_j\}$ is a sequence of zero mean martingale. As discussed in the last step, under null hypothesis, we can know that the variance of $E(D)$ is
$$
\sigma_E^2:={2\over s_+(s_+-1)}{\rm Tr}(\Sigma_+^2)+{2\over s_-(s_--1)}{\rm Tr}(\Sigma_{-}^2)+{4\over s_+s_{-}}{\rm Tr}(\Sigma_+\Sigma_-).
$$
As $s_+=s-s_-$ is drawn from binomial distribution ${\rm Bin}(s,1/2)$, we have
$$
{\sigma_E^2\over 8\|\Sigma_++\Sigma_{-}\|_F^2/s^2}\to 1.
$$
Corollary 3.1 in \cite{hall2014martingale} suggests that 
$$
{E(D)\over \sigma_E}\to N(0,1),\qquad s\to \infty 
$$
provided that
\begin{equation}
	\label{eq:cltcond1}
	{\rm for\ all\ }\epsilon>0,\qquad \sigma_{E}^{-2}\sum_{j=2}^s\EE(V_j^2\bI(|V_j|>\sigma_{E}\epsilon)|\Fcal_{j-1})\to 0
\end{equation}
and
\begin{equation}
	\label{eq:cltcond2}
	{\sum_{j=2}^s \EE(V_j^2|\Fcal_{j-1})\over \sigma_{E}^2}\stackrel{P}{\to} {1\over 4}.
\end{equation}
To show \eqref{eq:cltcond1}, it is sufficient to show that
$$
{ \sum_{j=2}^s\EE(V_j^4) /\sigma_{E}^{4}}\to 0.
$$
Because for any $i_1\ne i_2\ne i_3\ne i_4\ne i_5$, we have
$$
\EE(K_{i_1i_2}K_{i_1i_3}K_{i_1i_4}K_{i_1i_5})=0\qquad {\rm and}\qquad \EE(K_{i_1i_2}K_{i_1i_3}^3)=0,
$$
we can obtain 
\begin{align*}
	\sum_{j=2}^s\EE(V_j^4)&=\sum_{j=2}^s\sum_{i=1}^{j-1}\EE(K^4_{ij})+3\sum_{j=2}^s\sum_{1\le i, i'\le j-1}\EE(K^2_{ij}K^2_{i'j})\\
	&\le \sum_{j=2}^s\sum_{i=1}^{j-1}\EE(K^4_{ij})+3\sum_{j=2}^s\sum_{1\le i, i'\le j-1}\sqrt{\EE(K^4_{ij})\EE(K^4_{i'j})}\\
	&\le Cs^2 {o\left(s\|\Sigma_++\Sigma_{-}\|_F^4\right) \over s^8}+Cs^3{o\left(s\|\Sigma_++\Sigma_{-}\|_F^4\right) \over s^8}\\
	&\le o\left(\|\Sigma_++\Sigma_{-}\|_F^4/s^4  \right)\\
	&\le o\left(\sigma_{E}^{4}\right).
\end{align*}
Here, we use assumption (e) in Assumption~\ref{ap:twosample}. we now complete the proof of \eqref{eq:cltcond1}. Next, we work on \eqref{eq:cltcond2}. Note that
$$
\EE(V_j^2|\Fcal_{j-1})=\sum_{i_1,i_2=1}^{j-1}\EE(K_{i_1,j}K_{i_2,j}|\Fcal_{j-1})=\sum_{i_1,i_2=1}^{j-1}\gamma_{i_1,j}\gamma_{i_2,j}X_{i_1}^T\Sigma_{(j)}X_{i_2}.
$$
where $\Sigma_{(j)}=\Sigma_{+}$ if $Y_j=1$ and $\Sigma_{(j)}=\Sigma_{-}$ if $Y_j=-1$. This suggests 
$$
\EE\left(\sum_{j=2}^s\EE(V_j^2|\Fcal_{j-1})\right)=\sum_{j=2}^s\sum_{i=1}^{j-1}\gamma_{i,j}^2\tr(\Sigma_{(i)}\Sigma_{(j)})={\sigma_{E}^2\over 4}. 
$$
If $i\le j$, then
\begin{align*}
	\EE\left\{\EE(V_i^2|\Fcal_{i-1})\EE(V_j^2|\Fcal_{j-1})\right\}=&\EE\left(\sum_{i_1,i_2=1}^{j-1}\sum_{i_3,i_4=1}^{i-1}\gamma_{i_1,j}\gamma_{i_2,j}\gamma_{i_3,i}\gamma_{i_4,i}X_{i_1}^T\Sigma_{(j)}X_{i_2}X_{i_3}^T\Sigma_{(i)}X_{i_4}\right)\\
	=&\underbrace{4\sum_{1\le i_1<i_2\le i-1}\gamma_{i_1,j}\gamma_{i_2,j}\gamma_{i_1,i}\gamma_{i_2,i}\tr(\Sigma_{(j)}\Sigma_{(i_1)}\Sigma_{(i)}\Sigma_{(i_2)})}_{F_{1i}}\\
	&+\underbrace{\sum_{i_1=1}^{i-1}\gamma_{i_1,j}^2\gamma_{i_1,i}^2\left(\EE(X_{i_1}^T\Sigma_{(j)}X_{i_1}X_{i_1}^T\Sigma_{(i)}X_{i_1})-\tr(\Sigma_{(i_1)}\Sigma_{(i)})\tr(\Sigma_{(i_1)}\Sigma_{(j)})\right)}_{F_{2i}}\\
	&+\underbrace{\sum_{i_1=1}^{i-1}\sum_{i_2=1}^{j-1}\gamma_{i_1,i}^2\gamma_{i_2,j}^2\tr(\Sigma_{(i_1)}\Sigma_{(i)})\tr(\Sigma_{(i_2)}\Sigma_{(j)})}_{F_{3ij}}.
\end{align*}
By assumption (d) in Assumption~\ref{ap:twosample}, we know that $\tr(\Sigma_{(j)}\Sigma_{(i_1)}\Sigma_{(i)}\Sigma_{(i_2)})=o(\left\|\Sigma_++\Sigma_{-}\right\|^4_F)$, and thus
$$
F_{1i}=o\left(\left\|\Sigma_++\Sigma_{-}\right\|^4_F/s^6\right).
$$
For $F_{2i}$, note that 
\begin{align*}
	\EE(X_{i_1}^T\Sigma_{(j)}X_{i_1}X_{i_1}^T\Sigma_{(i)}X_{i_1})&\le \sqrt{\EE(X_{i_1}^T\Sigma_{(j)}X_{i_1})^2\EE(X_{i_1}^T\Sigma_{(i)}X_{i_1})^2}\\
	&\le \sqrt{\EE(X_{i_1}^TX_{j})^4\EE(X_{i_1}^TX_{i})^4}\\
	&=o(s\left\|\Sigma_++\Sigma_{-}\right\|^4_F).
\end{align*}
Here, we use assumption (e) in Assumption~\ref{ap:twosample}.
This leads to 
$$
F_{2i}=o\left(\left\|\Sigma_++\Sigma_{-}\right\|^4_F/s^6\right).
$$
For $F_{3ij}$, one can verify that
$$
F_{3ij}=o\left(\left\|\Sigma_++\Sigma_{-}\right\|^4_F/s^5\right)\qquad {\rm and}\qquad \sum_{1\le i<j\le m}F_{3ij}={\sigma_{E}^4\over 16},
$$
Consequently, putting $F_{1i}$, $F_{2i}$ and $F_{3ij}$ together suggests 
\begin{align*}
	\Var\left(\sum_{j=2}^s\EE(V_j^2|\Fcal_{j-1})\right)&=\EE\left(\sum_{j=2}^s\EE(V_j^2|\Fcal_{j-1})\right)^2-{\sigma_{E}^4\over 16}\\
	&\le o\left(\left\|\Sigma_++\Sigma_{-}\right\|^4_F/s^4\right)\\
	&\le o\left(\sigma_{E}^{4}\right).
\end{align*}
Therefore, we can conclude
$$
\PP\left(\left|{\sum_{j=2}^s\EE(V_j^2|\Fcal_{j-1})\over \sigma_{E}^4}-{1\over 4}\right|\ge \epsilon\right)\le {\Var(\sum_{j=2}^s\EE(V_j^2|\Fcal_{j-1}))\over \epsilon^2\sigma_{E}^4}\to 0.
$$
So we prove \eqref{eq:cltcond2} and ${E(D)/\sigma_E}\to N(0,1)$ under null hypothesis, as $s\to \infty$. This immediately suggests the type I error can be controlled at $\epsilon/2$ level asymptotically if we reject the null hypothesis when $E(D)>z_\alpha \sigma_E$ and $\alpha=\epsilon/2$. 

Now, let's look at type II error. We adopt the same notation in the analysis for permutation test. With the same argument for the null hypothesis, we can show that
$$
{E_{11}+E_{21}-E_{31}\over \sigma_E}\to N(0,1),
$$
which leads to
$$
\PP\left(E_{11}+E_{21}-E_{31}<-C_\epsilon \sigma_E \right)\lesssim {\epsilon\over 8}.
$$
for some large enough constant $C_\epsilon$. Since
$$
E_{12}+E_{22}-E_{32}={4\over s_+}\sum_{Y_i=1}\left(X_i-{\mu\over 2}\right)^T\mu-{4\over s_-}\sum_{Y_i=-1}\left(X_i+{\mu\over 2}\right)^T\mu+4\mu^T\mu,
$$
the central limit theorem suggests that 
$$
\sqrt{s}{E_{12}+E_{22}-E_{32}-4\mu^T\mu \over 4\sqrt{2\mu^T(\Sigma_++\Sigma_{-})\mu}}\to N(0,1).
$$
So we can know that there exists a constant $C'_\epsilon$ such that
$$
\PP\left(E_{12}+E_{22}-E_{32}-4\mu^T\mu<-C'_\epsilon {\sqrt{\mu^T(\Sigma_++\Sigma_{-})\mu} \over \sqrt{s}}  \right)\lesssim {\epsilon\over 8}.
$$
We can control type II error at $\epsilon/2$ level asymptotically if 
$$
2\|\mu\|^2>C_\epsilon \sigma_E+C'_\epsilon {\sqrt{\mu^T(\Sigma_++\Sigma_{-})\mu} \over \sqrt{s}}.
$$ 
We now complete the proof.

\paragraph{Step 2: Mahalanobis distance}	
All the proof in the last two steps can be easily extend to Mahalanobis distance. The Mahalanobis distance $D_M$ can be seen as Euclidean distance after linear transformation. Specifically, if we write $M=BB^T$, $D_M(X_1,X_2)=\|B^T(X_1-X_2)\|^2$. Thus, to  control both type I and II error, we require 
$$
\|B^T\mu\|^2\ge {C\over s}\left\|B^T(\Sigma_++\Sigma_{-})B\right\|_F.
$$
If $M=UU^T$, then we need 
$$
\|U^T\mu\|^2\ge {C\over s}\left\|U^T(\Sigma_++\Sigma_{-})U\right\|_F.
$$
\paragraph{Step 3: covariance calculation}	
Now, we go back to our multi-view model to find the covariance matrix. By law of total variance, we can decompose the covariance structure of $X$ in two different ways. The first one is
\begin{align*}
	{\rm Var}(X)&={\rm Var}(\EE(X|Y))+\EE({\rm Var}(X|Y))\\
	&={\rm Var}(\EE(X|Y))+\EE({\rm Var}(\EE(X|Z,Y)|Y))+\EE({\rm Var}(X|Y,Z))\\
	&={\rm Var}(\EE(X|Y))+\EE({\rm Var}(\EE(X|Z)|Y))+\EE({\rm Var}(X|Z))
\end{align*}
The last step is due to $X\perp Y|Z$. The second one is 
$$
{\rm Var}(X)={\rm Var}(\EE(X|Z))+\EE({\rm Var}(X|Z)).
$$
If we compare these two decomposition, we can conclude that
$$
{\rm Var}(\EE(X|Z))={\rm Var}(\EE(X|Y))+\EE({\rm Var}(\EE(X|Z)|Y)).
$$
In the multi-view model, we know that
$$
{\rm Var}(\EE(X|Z))=BB^T\qquad {\rm and}\qquad {\rm Var}(\EE(X|Y))={\mu\mu^T\over 4}.
$$
So
$$
\EE({\rm Var}(\EE(X|Z)|Y))=BB^T-{\mu\mu^T\over 4}.
$$
This has two implications: $\|\mu/2\|^2\le \lambda_1$ and 
$$
\sum_{k=2}^K\lambda_k^2\le \|\EE({\rm Var}(\EE(X|Z)|Y))\|_F^2\le \sum_{k=1}^K\lambda_k^2.
$$ 
Then we can know that
\begin{align*}
	\left\|{1\over 2}\left(\Sigma_++\Sigma_{-}\right)\right\|_F^2&=\left\|\EE({\rm Var}(X|Y))\right\|_F^2\\
	&=\left\|\EE({\rm Var}(\EE(X|Z)|Y))+\EE({\rm Var}(X|Z))\right\|_F^2\\
	&\le \left\|BB^T+\Sigma\right\|_F^2
\end{align*}
Combing this with Step 1 immediately suggests that 
$$
r(\|\cdot\|^2,\epsilon)\le C{\|BB^T+\Sigma\|_F^{1/2}\over \sqrt{s}}.
$$

When it comes to Mahalanobis distance, we need to find bounds for $\|B^T\mu\|^2$ and $\|B^T(\Sigma_++\Sigma_{-})B\|_F^2$. First, we note that $\mu$ can be written as a linear combination of $b_1,\ldots, b_K$ since $X\perp Y|Z$ suggests 
$$
\EE(X|Y)=\EE(\EE(X|Y,Z)|Y)=\EE(\EE(X|Z)|Y)=B\EE(G(Z)|Y).
$$
Then, we have 
$$
\|B^T\mu\|^2\ge\lambda_K\|\mu\|^2
$$
and
\begin{align*}
	\|B^T(\Sigma_++\Sigma_{-})B\|_F^2&\le 2\|B^T(BB^T+\Sigma)B\|_F^2.
\end{align*}
Therefore, 
$$
r(D^\ast,\epsilon)\le C{\|B^T(BB^T+\Sigma)B\|_F^{1/2}\over \sqrt{s\lambda_K}}.
$$
Similarly, we can show that
$$
r(D^{\ast\ast},\epsilon)\le C{\|U^T(BB^T+\Sigma)U\|_F^{1/2}\over \sqrt{s}}.
$$
\subsubsection{Lower bound} 

We now work on the lower bound. The main idea of lower bound proof is to derive the asymptotic distribution of $E(D)$ under local alternative hypothesis. Recall that the local alternative hypothesis of interest is defined as 
$$
\tilde{H}_1(r)=\left\{\|\mu\|=r,\ \mu=ru_K\right\}.
$$

\paragraph{Step 1a: Euclidean distance and asymptotic distribution} We first show if $ \|\mu\|^2=o(\|\Sigma_++\Sigma_-\|_F/s)$, energy distance test based on asymptotic distribution has trivial power. 
Following the same analysis in step 1b of upper bound, we can know that 
$$
{E_{11}+E_{21}-E_{31}\over 2\sqrt{2}\|\Sigma_++\Sigma_{-}\|_F/s}\to N(0,1),
$$
and 
$$
{E_{12}+E_{22}-E_{32}-4\mu^T\mu \over 4\sqrt{2\mu^T(\Sigma_++\Sigma_{-})\mu/s}}\to N(0,1).
$$
Note that 
$$
\|\mu\|^2=o(\|\Sigma_++\Sigma_-\|_F/s)
$$
and
$$
\sqrt{\mu^T(\Sigma_++\Sigma_{-})\mu/s}\le \|\mu\| \sqrt{\|\Sigma_++\Sigma_-\|_F/s}=o(\|\Sigma_++\Sigma_-\|_F/s).
$$
It suggests that 
$$
E_{12}+E_{22}-E_{32}=o_p(\|\Sigma_++\Sigma_{-}\|_F/s)\quad {\rm and}\quad  \|\mu\|^2=o(\|\Sigma_++\Sigma_-\|_F/s). 
$$
This suggest that under the local alternative hypothesis, we have 
$$
{E(D)\over 2\sqrt{2}\|\Sigma_++\Sigma_{-}\|_F/s}\to N(0,1).
$$
This immediately suggest that 
$$
\PP(\phi_D=0|\tilde{H}_1(r))\to \alpha.
$$

\paragraph{Step 1b: Euclidean distance and permutation test} Similarly, we can show energy distance test (permutation test) has trivial power when $ \|\mu\|^2=o(\|\Sigma_++\Sigma_-\|_F/s)$. In step 1a, we have already shown that under local alternative hypothesis, ${E(D)/ (2\sqrt{2}\|\Sigma_++\Sigma_{-}\|_F/s)}\to N(0,1)$. With the same argument in step 1b of upper bound, we can show that 
$$
{\phi_1 E(D)\over 4\|\Sigma_\pm\|_F/s}\to N(0,1).
$$
Because $ \|\mu\|^2=o(\|\Sigma_++\Sigma_-\|_F/s)$, we can know that 
$$
{4\|\Sigma_\pm\|_F \over 2\sqrt{2}\|\Sigma_++\Sigma_{-}\|_F}\to 1.
$$
Therefore, we can know that 
$$
{\phi_1 E(D)\over 2\sqrt{2}\|\Sigma_++\Sigma_{-}\|_F/s}\to N(0,1).
$$
Now we complete the proof. 

\paragraph{Step 2: Mahalanobis distance}
All the proof in the last two steps can be generalized to Mahalanobis distance. More specifically, we can show that the test has trivial power when the distance is $D^\ast$ and $ \|B^T\mu\|^2=o(\|B^T(\Sigma_++\Sigma_-)B\|_F/s)$, and when the distance is $D^{\ast\ast}$ and $ \|U^T\mu\|^2=o(\|U^T(\Sigma_++\Sigma_-)U\|_F/s)$. 

\paragraph{Step 3: covariance calculation}	
From step 3 of upper bound, we can know that 
$$
{1\over 2}(\Sigma_++\Sigma_-)=BB^T+\Sigma-{\mu\mu^T\over 4}.
$$
Because $\mu=ru_K$, we can know that
$$
\|BB^T+\Sigma\|_F\le {K\over K-1}\left\|{1\over 2}(\Sigma_++\Sigma_-)\right\|_F.
$$
Therefore, we can know that $ \|\mu\|^2=o(\|BB^T+\Sigma\|_F/s)$ can lead to trivial power of energy distance test. 

We next work on $D^\ast$. Because $\mu=ru_K$, note that 
$$
\|B^T\mu\|^2=\lambda_K\|\mu\|^2\qquad {\rm and}\qquad \|B^T(BB^T+\Sigma)B\|_F\le {K\over K-1}\left\|{1\over 2}B^T(\Sigma_++\Sigma_-)B\right\|_F.
$$
Therefore, $\|\mu\|^2=o(\|B^T(BB^T+\Sigma)B\|_F/s\lambda_K)$ implies $ \|B^T\mu\|^2=o(\|B^T(\Sigma_++\Sigma_-)B\|_F/s)$ and thus trivial power of energy distance test. 

Similarly, we can show $\|\mu\|^2=o(\|U^T(BB^T+\Sigma)U\|_F/s)$ implies $ \|U^T\mu\|^2=o(\|U^T(\Sigma_++\Sigma_-)U\|_F/s)$. We now complete the proof. 

%%%%%%%%%%%%%%%%%%%%%%%%%%%%%%%%%%%%%%%%%%%%%%%
\subsection{Proof of Theorem~\ref{thm:metricupper}}
%%%%%%%%%%%%%%%%%%%%%%%%%%%%%%%%%%%%%%%%%%%%%%%
Without loss of generality, we assume $\EE(X_{i,j})=0$ and $\EE(Z_i)=0$. 

\subsubsection{Proof for $\hat{M}^\ast$} 
Recall $\hat{R}$ is defined as
\begin{align*}
	\hat{R}=-{1\over m(m-1)}\sum_{i\ne i'}\left(\bar{X}_{i}\bar{X}_{i'}^T+\bar{X}_{i'}\bar{X}_{i}^T\right)+{1\over mn(n-1)}\sum_{i,j\ne j'}\left(X_{i,j}X_{i,j'}^T+X_{i,j'}X_{i,j}^T\right)
\end{align*}
We write $E_i=\EE(X_{i,j}|Z_i)=BZ_i$.  We bound the above two terms in $\hat{R}$ separately. The first term can be decomposed as
\begin{align*}
	&{1\over m(m-1)}\sum_{i\ne i'}\left(\bar{X}_{i}\bar{X}_{i'}^T+\bar{X}_{i'}\bar{X}_{i}^T\right)\\
	=&{1\over m(m-1)}\sum_{i\ne i'}\left((\bar{X}_{i}-E_i)(\bar{X}_{i'}-E_{i'})^T+(\bar{X}_{i'}-E_{i'})(\bar{X}_{i}-E_i)^T\right)\\
	&+{2\over m(m-1)}\sum_{i\ne i'}\left(E_i(\bar{X}_{i'}-E_{i'})^T+(\bar{X}_{i'}-E_{i'})E_i^T\right)+{1\over m(m-1)}\sum_{i\ne i'}\left(E_iE_{i'}^T+E_{i'}E_i^T\right)\\
	=&A_1+A_2+A_3
\end{align*}
Because $m\gg \log(d+m)$, $\bar{X}_i-E_i$ is a sub-Gaussian vector with parameters $\sigma^2/n$ and $Z_i$ is a sub-Gaussian vector with parameters $1$, an application of Lemma~\ref{lm:ustat} on $A_1$ and $A_3$ suggests 
$$
\PP\left(\|A_1\|>C\sigma^2{\sqrt{d\log (d+m)}+\log (d+m)\over nm}\right)\le {1\over (d+m)^5}
$$
and
$$
\PP\left(\|A_3\|>C\|B\|^2{\sqrt{K\log (d+m)}+\log (d+m)\over m}\right)\le {1\over (d+m)^5}
$$
Similarly, we then apply Lemma~\ref{lm:crossstat} to bound $A_2$
$$
\PP\left(\|A_2\|>C\sigma\|B\|{\sqrt{K\log (d+m)}+\log (d+m)+\sqrt{d}\over \sqrt{n}m}\right)\le {1\over (d+m)^5}
$$
The second term in $\hat{R}$ can be decomposed as
\begin{align*}
	&{1\over mn(n-1)}\sum_{i,j\ne j'}\left(X_{i,j}X_{i,j'}^T+X_{i,j'}X_{i,j}^T\right)-BB^T\\
	=&{1\over mn(n-1)}\sum_{i,j\ne j'}\left((X_{i,j}-E_i)(X_{i,j'}-E_i)^T+(X_{i,j'}-E_i)(X_{i,j}-E_i)^T\right)\\
	&+{2\over mn}\sum_{i,j}\left(E_i(X_{i,j}-E_i)^T+(X_{i,j}-E_i)E_i^T\right)+{1\over m}\sum_{i}E_iE_i^T-BB^T\\
	=&{1\over mn(n-1)}\sum_{i,j\ne j'}\left((X_{i,j}-E_i)(X_{i,j'}-E_i)^T+(X_{i,j'}-E_i)(X_{i,j}-E_i)^T\right)\\
	&+{2\over m}\sum_{i}\left(E_i(\bar{X}_{i}-E_i)^T+(\bar{X}_{i}-E_i)E_i^T\right)+{1\over m}\sum_{i}E_iE_i^T-BB^T\\
	=&A_4+A_5+A_6
\end{align*}
To bound $A_4$, we now define 
$$
T_i={1\over n(n-1)}\sum_{j\ne j'}\left((X_{i,j}-E_i)(X_{i,j'}-E_i)^T+(X_{i,j'}-E_i)(X_{i,j}-E_i)^T\right)
$$
Clearly, $T_i$ is a $U$-statistic and Lemma~\ref{lm:ustat} can be applied
$$
\PP\left(\left\|T_i\right\|\ge C\sigma^2\left({\sqrt{4dt}+4t\over n}+{\sqrt{d\log d}\over n^{3/2}}+{\sqrt{d}(t+\log d)\over n^2}\right)\right)\le 2\exp(-t)
$$
If we choose 
$$
L=C\sigma^2\left({\sqrt{d\log(d+m)}+\log(d+m)\over n}+{\sqrt{d}\log(d+m)\over n^2}\right),
$$
we can have
$$
\PP\left(\left\|T_i\right\|\ge L \right)\le {1\over (d+m)^7}.
$$
Now we turn to truncated expectation
\begin{align*}
	\EE(\|T_i\|\bI(\|T_i\|\ge L) )&=L\PP(\|T_i\|>L)+\int_L^\infty \PP(\|T_i\|>x)dx\\
	&\le L\PP(\|T_i\|>L)+C\int_L^\infty \exp\left(-C\min\left({n^2x^2\over d\sigma^4},{nx\over \sigma^2},{n^2x \over \sqrt{d}\sigma^2}\right)\right)dx\\
	&\le L\PP(\|T_i\|>L)+C\max\left({\sigma^2\over n},{\sqrt{d}\sigma^2\over n^2}\right)\PP(\|T_i\|>L)\\
	&\le {L\over (d+m)^6}
\end{align*}
To apply Lemma~\ref{lm:trunmatrixbern}, we still need to determine the variance statistic
\begin{align*}
	\EE(T_i^2)&={1 \over n(n-1)}\EE\left((X_{i,j}-E_i)(X_{i,j'}-E_i)^T+(X_{i,j'}-E_i)(X_{i,j}-E_i)^T\right)^2\\
	&={2 \over n(n-1)}\left(\Sigma^2+{\rm Tr}(\Sigma)\Sigma\right)
\end{align*}
This leads to
$$
v_1^2=\left\|\sum_{i=1}^m\EE(T_i^2)\right\|= {2m\over n(n-1)}\left\|\Sigma^2+{\rm Tr}(\Sigma)\Sigma\right\|\le {4md\sigma^4\over n(n-1)}
$$
An application of Lemma~\ref{lm:trunmatrixbern} on $T_i$ yields 
$$
\PP\left(\left\|\sum_{i=1}^mT_i\over m\right\|\ge C\left(v_1{\sqrt{t-s+\log d}\over m}+L{t-s+\log d\over m}\right)\right)\le \exp(-t)+{1\over (d+m)^6},
$$
where $s={L/(d+m)^5}\le \log(d+m)/(d+m)^4$. Since $m\gg \log^2(d+m)$, taking $t=6\log(d+m)$ suggests 
$$
\PP\left(\|A_4\|>C\sigma^2{\sqrt{d\log (d+m)}\over n\sqrt{m}}\right)\le {1\over (d+m)^5}.
$$
To bound $A_5$, we apply Lemma~\ref{lm:crossconcen}
$$
\|\|E_i(\bar{X}_{i}-E_i)^T\|\|_{\psi_1}=\|(\bar{X}_{i}-E_i)^TBZ_i\|_{\psi_1}\le 2\sqrt{K}\sigma\|B\|/\sqrt{n}
$$
and observe
\begin{align*}
	v_2^2&=\max\left(\|\sum_{i=1}^m\EE\left(E_i(\bar{X}_{i}-E_i)^T(\bar{X}_{i}-E_i)E_i^T\right)\|,\|\sum_{i=1}^m\EE\left((\bar{X}_{i}-E_i)E_i^TE_i(\bar{X}_{i}-E_i)^T\right)\|\right)\\
	&\le md\sigma^2\|B\|^2/n
\end{align*}
We then can apply Lemma~\ref{lm:matrixbern} to bound $A_5$
$$
\PP\left(\|A_5\|>C\sigma\|B\|{\sqrt{d\log (d+m)}\over \sqrt{nm}}\right)\le {1\over (d+m)^5},
$$
where we use the conditon $m\gg \log^2(d+m)$. We now bound the $A_6$. $A_6$ can be rewritten as
$$
A_6={1\over m}\sum_{i}E_iE_i^T-BB^T=B\left({1\over m}\sum_{i}Z_iZ_i^T-I\right)B^T
$$
An application of Theorem 5.39 in \cite{vershynin2010introduction} yields
$$
\PP\left(\left\|{1\over m}\sum_{i}Z_iZ_i^T-I\right\|>C{\sqrt{K}+\sqrt{t}\over \sqrt{m}}+C{K+t\over m}\right)\le \exp(-t)
$$
Let $t=6\log(d+m)$ and then we can get
$$
\PP\left(\left\|{1\over m}\sum_{i}Z_iZ_i^T-I\right\|>C{\sqrt{K\log (d+m)}\over \sqrt{m}}\right)\le {1\over (d+m)^5}
$$
This leads to
$$
\PP\left(\|A_6\|>C\|B\|^2{\sqrt{K\log (d+m)}\over \sqrt{m}}\right)\le {1\over (d+m)^5}
$$
Now, we can put $A_1, \ldots, A_6$ together to obtain 
$$
\PP\left(\|\hat{R}-BB^T\|>C{\sqrt{\log(d+m)}\over \sqrt{m}}\left[\|B\|^2\sqrt{K}+\sigma\|B\|{\sqrt{d}\over \sqrt{n}}+\sigma^2{\sqrt{d}\over n}\right]\right)\le {6\over (d+m)^5}.
$$

We now define the event 
$$
\Acal=\{\|\hat{R}-BB^T\|\le \delta\},\qquad {\rm where}\quad \delta=C{\sqrt{\log(d+m)}\over \sqrt{m}}\left[\|B\|^2\sqrt{K}+\sigma\|B\|{\sqrt{d}\over \sqrt{n}}+\sigma^2{\sqrt{d}\over n}\right]
$$
then we know $\PP(\Acal)\ge 1- {6/(d+m)^5}$. The rest of analysis is conditioned on the event $\Acal$. We write $\tilde{R}=\sum_{k=1}^K\hat{\lambda}_k\hat{u}_k\hat{u}_k^T$. By Weyl’s inequality \citep{stewart1990matrix}, we can know that
$$
\lambda_k-\delta\le \hat{\lambda}_k\le \lambda_k+\delta,
$$
which leads to
$$
\|\hat{M}^\ast-M^\ast\|\le \delta,
$$
provided $\lambda_K/4>\delta$. 

\subsubsection{Proof for $\hat{M}^{\ast\ast}$} 
The proof for $\hat{M}^{\ast\ast}$ is slightly different from the proof for $\hat{M}^\ast$. We still use the same notation as in the proof for $\hat{M}^\ast$. We define 
$$
\hat{R}_E=A_3+A_6+BB^T={1\over m}\sum_{i}E_iE_i^T+{1\over m(m-1)}\sum_{i\ne i'}\left(E_iE_{i'}^T+E_{i'}E_i^T\right).
$$
Since $\hat{R}-\hat{R}_E=A_1+A_2+A_4+A_5$, we can know that 
$$
\PP\left(\|\hat{R}-\hat{R}_E\|>C{\sqrt{\log(d+m)}\over \sqrt{m}}\left[\sigma\|B\|{\sqrt{d}\over \sqrt{n}}+\sigma^2{\sqrt{d}\over n}\right]\right)\le {4\over (d+m)^5}.
$$
We now define the event 
$$
\Acal_1=\{\|\hat{R}-\hat{R}_E\|\le \delta'\},\qquad {\rm where}\quad \delta'=C{\sqrt{\log(d+m)}\over \sqrt{m}}\left[\sigma\|B\|{\sqrt{d}\over \sqrt{n}}+\sigma^2{\sqrt{d}\over n}\right],
$$
then we know $\PP(\Acal_1)\ge 1- {4/(d+m)^5}$. 

Next, we study the $K$st eigenvalue of $\hat{R}_E$. Because 
$$
\hat{R}_E=B\left({1\over m}\sum_{i}Z_iZ_i^T+{1\over m(m-1)}\sum_{i\ne i'}\left(Z_iZ_{i'}^T+Z_{i'}Z_i^T\right)\right)B^T,
$$
it is sufficient to study the smallest eigenvalue of $\hat{R}_Z$
$$
\hat{R}_Z={1\over m}\sum_{i}Z_iZ_i^T+{1\over m(m-1)}\sum_{i\ne i'}\left(Z_iZ_{i'}^T+Z_{i'}Z_i^T\right).
$$
From the proof for $\hat{M}^\ast$, we can know that 
$$
\PP\left(\left\|{1\over m}\sum_{i}Z_iZ_i^T-I\right\|>C{\sqrt{K\log (d+m)}\over \sqrt{m}}\right)\le {1\over (d+m)^5}.
$$
And an application of Lemma~\ref{lm:ustat} suggests that 
$$
\PP\left(\left\|{1\over m(m-1)}\sum_{i\ne i'}\left(Z_iZ_{i'}^T+Z_{i'}Z_i^T\right)\right\|\ge C{\sqrt{K\log (d+m)}+\log (d+m)\over m}\right)\le {1\over (d+m)^5}.
$$
Putting above two bound suggests
$$
\PP\left(\left\|\hat{R}_Z-I\right\|>C{\sqrt{K\log (d+m)}\over \sqrt{m}}\right)\le {2\over (d+m)^5}.
$$
We now define the event 
$$
\Acal_2=\{\|\hat{R}_Z-I\|\le \delta''\},\qquad {\rm where}\quad \delta''=C{\sqrt{K\log (d+m)}\over \sqrt{m}},
$$
then we know $\PP(\Acal_2)\ge 1- {2/(d+m)^5}$. On the event $\Acal_2$, we can know that the $K$st eigenvalue of $\hat{R}_E$, denoted by $\lambda_K(\hat{R}_E)$,
$$
\lambda_K(\hat{R}_E)\ge \lambda_K(1-\delta'').
$$

Note that the eigenvectors of $\hat{R}_E$ are $U$. We write $\tilde{R}'=\sum_{k=1}^K\hat{u}_k\hat{u}_k^T$. On the event $\Acal_1\cap\Acal_2$, we can apply Davis-Kahan theorem $\hat{R}=\hat{R}_E+(\hat{R}-\hat{R}_E)$ \citep[see, e.g., Corollary 2.8][]{chen2020spectral}
$$
\left\|\tilde{R}'-UU^T\right\|\le {C\delta'\over \delta''}.
$$
When $m\ge c\log(d+m)(K+d\kappa\sigma^2/n\lambda_K+d\sigma^4/n^2\lambda^2_K)$ and $\kappa$ is bounded, we can know that 
$$
{\delta'\over \delta''}\le C{\sqrt{\log(d+m)}\over \sqrt{m}}\left[\sigma{\sqrt{d}\over \sqrt{n\lambda_K}}+\sigma^2{\sqrt{d}\over n\lambda_K}\right].
$$
We now complete the proof.

%%%%%%%%%%%%%%%%%%%%%%%%%%%%%%%%%%%%%%%%%%%%%%%
\subsection{Proof of Theorem~\ref{thm:metriclower}}
%%%%%%%%%%%%%%%%%%%%%%%%%%%%%%%%%%%%%%%%%%%%%%%

\subsubsection{Proof for $\|\hat{M}^\ast-M^\ast\|$} 

To show the lower bound on $\|\hat{M}^\ast-M^\ast\|$, we only need to show the following two inequalities
\begin{equation}
	\label{eq:lowerR1}
	\inf_{\hat{M}^\ast}\sup_{B\in \Bcal(\nu)}\EE\left(\|M^\ast-\hat{M}^\ast\|\right)\ge {c\over \sqrt{m}}\left[\sigma{\sqrt{d\nu}\over \sqrt{n}}+\sigma^2{\sqrt{d}\over n}\right]
\end{equation}
and
\begin{equation}
	\label{eq:lowerR2}
	\inf_{\hat{M}^\ast}\sup_{B\in \Bcal(\nu)}\EE\left(\|M^\ast-\hat{M}^\ast\|\right)\ge {c\sqrt{K}\nu\over \sqrt{m}}.
\end{equation}
We first consider \eqref{eq:lowerR1}. Define a long vector $\bX_i=(X_{i,1}^T,\ldots,X_{1,n}^T)^T\in \RR^{nd\times 1}$, of which distribution is a mean-zero normal distribution with covariance matrix
$$
\Sigma(B)=\begin{pmatrix}
	BB^T+\sigma^2I & BB^T & \cdots & BB^T\\
	BB^T & BB^T+\sigma^2I  & \cdots & BB^T\\
	\cdots & \cdots & \cdots & \cdots \\
	BB^T & BB^T & \cdots & BB^T+\sigma^2I\\
\end{pmatrix}=\bB\bB^T+\sigma^2I_{nd\times nd},
$$
where $\bB=(B^T,\ldots,B^T)^T\in \RR^{nd\times K}$. Equivalently, our observation are independent identical copies $\bX_1, \bX_2,\ldots, \bX_m$.  The proof is then divided into three steps.
\paragraph{Step 1: hypothesis construction}
By the Varshamov-Gilbert bound \citep[][Lemma 4.7]{massart2007concentration}, we can find a collection of vectors $u_1,\ldots, u_N\in \{-1,1\}^{\lfloor d/2\rfloor}$ with $N\ge \exp(d/16)$ such that 
$$
\|u_{l_1}-u_{l_2}\|^2\ge \lfloor d/2\rfloor,\qquad l_1\ne l_2.
$$
Given $u_1,\ldots, u_N$, we can generate $w_1,\ldots, w_N$ such that $w_l=(\delta u_l/\sqrt{\lfloor d/2\rfloor},\sqrt{1-\delta^2})\in \RR^{\lfloor d/2\rfloor+1}$, where $0<\delta<1/2$ is a constant to be specified later. Based on $w_1,\ldots, w_N$, we can construct $B_1,\ldots, B_N$ in the following way
$$
B_l=\begin{pmatrix}
	w_l & 0\\
	0 & I_{(K-1)\times (K-1)}\\
	0 & 0
\end{pmatrix}\times \sqrt{\nu} I_{K\times K}.
$$
By the construction, it is clear that $B_l^TB_l=\nu I_{K\times K}$ so $B_l\in \Bcal(\nu)$. The definition suggest that $M_l$ is 
$$
M_l=\nu\times \begin{pmatrix}
	w_l & 0\\
	0 & I_{(K-1)\times (K-1)}\\
	0 & 0
\end{pmatrix} \times 
\begin{pmatrix}
	w_l^T & 0 &0\\
	0 & I_{(K-1)\times (K-1)} &0
\end{pmatrix}
$$
Then, we can know that, for any $l_1\ne l_2$,
\begin{align*}
	\|M_{l_1}-M_{l_2}\|&=
	\nu\sqrt{1-(w_{l_1}^Tw_{l_2})^2}\ge \nu\sqrt{1-\left(1-{\delta^2\over 2}\right)^2}\ge {\delta\nu\over 2}
\end{align*}
Here, we use the fact
$$
1-2\delta^2\le w_{l_1}^Tw_{l_2}\le 1-{\delta^2\over 2}
$$
because $\lfloor d/2\rfloor\le \|u_{l_1}-u_{l_2}\|^2\le 2d$ and 
$$
w_{l_1}^Tw_{l_2}=1-{1\over 2}\|w_{l_1}-w_{l_2}\|^2=1-{\delta^2\over 2\lfloor d/2\rfloor}\|u_{l_1}-u_{l_2}\|^2.
$$
If we $\mu_l$ corresponds the distribution of each hypothesis $\PP_l$, we can know these hypothesis are separated by $\delta/2$ in terms of the parameter of interest $M$.  

\paragraph{Step 2: bounding KL divergence} We now want bound the KL divergence ${\rm KL}(\PP_{l_1}||\PP_{l_2})$. Since each $\bX_i$ follow normal distribution, we then have
\begin{align*}
	{\rm KL}(\PP_{l_1}||\PP_{l_2})&=m{\rm KL}(N(0,\Sigma(B_{l_1}))||N(0,\Sigma(B_{l_2})))\\
	&={m \over 2}\left({\rm Tr}(\left(\Sigma(B_{l_2})\right)^{-1}\Sigma(B_{l_1}))-nd\right)
\end{align*}
By Woodbury matrix identity, we have
$$
\left(\Sigma(B)\right)^{-1}=(\sigma^2I+\bB\bB^T)^{-1}={1\over \sigma^{2}}I-{1\over n\nu\sigma^2+\sigma^4}\bB\bB^T
$$
which suggests 
$$
\left(\Sigma(B_{l_2})\right)^{-1}\Sigma(B_{l_1})=I+{1 \over \sigma^2}\bB_{l_1}\bB_{l_1}^T-{1\over n\nu+\sigma^2}\bB_{l_2}\bB_{l_2}^T-{1\over n\nu\sigma^2+\sigma^4}\bB_{l_2}\bB_{l_2}^T\bB_{l_1}\bB_{l_1}^T
$$
Since
$$
{\rm Tr}\left(\bB_{l_1}\bB_{l_1}^T\right)=Kn\nu,\qquad {\rm Tr}\left(\bB_{l_2}\bB_{l_2}^T\right)=Kn\nu
$$
and
$$
{\rm Tr}\left(\bB_{l_2}\bB_{l_2}^T\bB_{l_1}\bB_{l_1}^T\right)=(K-1+(w_{l_1}^Tw_{l_2})^2)n^2\nu^2,
$$
we then have
\begin{align*}
	{\rm Tr}(\left(\Sigma(B_{l_2})\right)^{-1}\Sigma(B_{l_1}))&=nd+{Kn\nu\over \sigma^2}-{Kn\nu\over n\nu+\sigma^2}-{(K-1+(w_{l_1}^Tw_{l_2})^2)n^2\nu^2\over n\nu\sigma^2+\sigma^4}\\
	&=nd+{(1-(w_{l_1}^Tw_{l_2})^2)n^2\nu^2\over n\nu\sigma^2+\sigma^4}\\
	&\le nd+{(1-(1-2\delta^2)^2)n^2\nu^2\over n\nu\sigma^2+\sigma^4}\\
	&\le nd+{4\delta^2n^2\nu^2\over n\nu\sigma^2+\sigma^4}
\end{align*}
Therefore, we have 
$$
{\rm KL}(\PP_{l_1}||\PP_{l_2})\le {2m\delta^2n^2\nu^2\over n\nu\sigma^2+\sigma^4}
$$

\paragraph{Step 3: application of Fano's lemma} In particular, we apply Lemma 3 in \cite{yu1997assouad}
\begin{align*}
	\inf_{\hat{M}^\ast}\sup_{B\in \Bcal(\nu)}\EE\left(\|M^\ast-\hat{M}^\ast\|\right)&\ge \inf_{\hat{M}^\ast}\sup_{B=B_l,l=1\ldots, N}\EE\left(\|M^\ast-\hat{M}^\ast\|\right)\\
	&\ge {\delta\nu\over 4}\left(1-\left({2m\delta^2n^2\nu^2\over n\nu\sigma^2+\sigma^4}+\log 2\right)\bigg/\log N\right)\\
	&\ge {\delta\nu\over 4}\left(1-16\left({2m\delta^2n^2\nu^2\over n\nu\sigma^2+\sigma^4}+\log 2\right)\bigg/d\right)
\end{align*}
If we choose 
$$
\delta=\sqrt{d(n\nu\sigma^2+\sigma^4)\over 256mn^2\nu^2}
$$
then we can finish the proof for \eqref{eq:lowerR1}.

Now, we turn to the proof for \eqref{eq:lowerR2}. To prove \eqref{eq:lowerR2}, we assume $n$ is infinity, i.e. $E_i=\EE(X_{i,j}|Z_i)=BZ_i$ is completely known and the entries of $B$ between $(K+1)$'s row and $d$'s row are zero. In other words, we work on the following simpler problem: we observe $E_1,\ldots, E_m\in \RR^K$ which follows normal distribution, i.e. $E_i\sim N(0,\Sigma)$, where $\Sigma\in \RR^{K\times K}$ and the goal is to estimate $M^\ast=\Sigma$. We consider a similar construction in \cite{cai2010optimal}. Without loss of generality, we assume $K$ is an even number. 

\paragraph{Step 1: hypothesis construction}
Specifically, for any $u\in \{0,1\}^{K/2}$, we define
$$
\Sigma(u)=2\nu I+{\tau \nu \over \sqrt{mK}}\sum_{l=1}^{K/2}u_lT_l.
$$
where $\tau<1/2$ is some small constant and $T_l=(t_{i,j,l})_{K\times K}$ is a matrix such that 
$$
t_{i,j,l}=I(i=l\ {\rm and}\ l+1\le j\le K)+I(j=l\ {\rm and}\ l+1\le i\le K).
$$
For any given $u$, we have 
$$
\left\|{\tau \over \sqrt{mK}}\sum_{l=1}^{K/2}u_lT_l\right\|\le \left\|{\tau \over \sqrt{mK}}\sum_{l=1}^{K/2}u_lT_l\right\|_F\le \sqrt{K^2\cdot {\tau^2\nu^2\over mK}}=\tau\nu\sqrt{ K\over m}
$$
Since $m>K$, we basically show that the eigenvalues of $\Sigma(u)$ are between $(2-\tau)\nu$ and $(2+\tau)\nu$. If we decompose $\Sigma(u)=B(u)B(u)^T$, then we can know that $B(u)\in \Bcal(\nu)$. For any $u\ne u'\in \{0,1\}^{K/2}$, 
\begin{align*}
	\Sigma(u)-\Sigma(u')&={\tau \nu \over \sqrt{mK}}\sum_{l=1}^{K/2}(u_l-u'_l)T_l.
\end{align*}
Let $z$ be a vector in $\RR^K$ such that $z_l=1$ when $K/2<l\le K$ and $z_l=0$ when $1\le l\le K/2$. We write $z'=(M(u)-M(u'))z$. If we write $H(u,u')$ as Hamming distance between $u$ and $u'$, there are at least $H(u,u')$ entries in $z'$ such that 
$$
|z'_l|\ge {\tau \sqrt{K}\nu\over 2\sqrt{m}}.
$$
As $\|z\|^2=K/2$, 
$$
\left\|M(u)-M(u')\right\|^2\ge {\|z'\|^2 \over \|z\|^2}\ge {H(u_1,u_2)\tau^2\nu^2\over 18m}.
$$

\paragraph{Step 2: application of Assouad’s lemma} The Lemma 6 in \cite{cai2010optimal} suggests that 
$$
\min_{H(u_1,u_2)=1}\|\PP_{u_1}\wedge \PP_{u_2}\|\ge c',
$$
where $\PP_{u}$ is the joint distribution of $G_1,\ldots, G_m$ and $c'$ is small constant. We are now ready to apply Lemma 4 in \cite{cai2010optimal} to obtain 
$$
\min_{\hat{M}^\ast}\max_{B\in B(\nu)} \left\|\hat{M}^\ast-M^\ast\right\|^2\ge \min_{\hat{M}^\ast}\max_{u\in \{0,1\}^{K/2}} \left\|\hat{M}^\ast-M(u)\right\|^2\ge c{K\nu^2\over m}
$$
We can then complete the proof for \eqref{eq:lowerR2}. 

\subsubsection{Proof for $\|\hat{M}^{\ast\ast}-M^{\ast\ast}\|$} 

The proof for $\hat{M}^{\ast\ast}$ is almost the same with the one for proving \eqref{eq:lowerR1}. We can use the same construction for $B_l$, $l=1,\ldots, N$, but define $M_l$ as
$$
M_l=\begin{pmatrix}
	w_l & 0\\
	0 & I_{(K-1)\times (K-1)}\\
	0 & 0
\end{pmatrix} \times 
\begin{pmatrix}
	w_l^T & 0 &0\\
	0 & I_{(K-1)\times (K-1)} &0
\end{pmatrix}
$$
Then, we can know that, for any $l_1\ne l_2$,
\begin{align*}
	\|M_{l_1}-M_{l_2}\|&=
	\sqrt{1-(w_{l_1}^Tw_{l_2})^2}\ge \sqrt{1-\left(1-{\delta^2\over 2}\right)^2}\ge {\delta\over 2}.
\end{align*}
With the same analysis, we can show that 
\begin{align*}
	\inf_{\hat{M}^{\ast\ast}}\sup_{B\in \Bcal(\nu)}\EE\left(\|M^{\ast\ast}-\hat{M}^{\ast\ast}\|\right)\ge  {\delta\over 4}\left(1-16\left({2m\delta^2n^2\nu^2\over n\nu\sigma^2+\sigma^4}+\log 2\right)\bigg/d\right)
\end{align*}
We can still choose 
$$
\delta=\sqrt{d(n\nu\sigma^2+\sigma^4)\over 256mn^2\nu^2}
$$
and obtain 
$$
\inf_{\hat{M}^{\ast\ast}}\sup_{B\in \Bcal(\nu)}\EE\left(\|M^{\ast\ast}-\hat{M}^{\ast\ast}\|\right)\ge {c\over \sqrt{m}}\left[\sigma{\sqrt{d}\over \sqrt{n\nu}}+\sigma^2{\sqrt{d}\over n\nu}\right].
$$

	\subsection{Proof of Theorem~\ref{thm:downstream}}
	%%%%%%%%%%%%%%%%%%%%%%%%%%%%%%%%%%%%%%%%%%%%%%%
	
	%%%%%%%%%%%%%%%%%%%%%%%%%%%%%%%%%%%%%%%%%%%%%%%
	\subsubsection{$k$-nearest neighbor classification}
	%%%%%%%%%%%%%%%%%%%%%%%%%%%%%%%%%%%%%%%%%%%%%%%
	
	\paragraph{Proof for $D_{\hat{M}^\ast}$}
	If we adopt the distance $D_{\hat{M}^\ast}$ in $k$-NN, the main difference is the shape of the ball $\Bcal_{D_{\hat{M}^\ast}}(x, r_{2k/s})=\{y:D_{\hat{M}^\ast}(x,y)\le r_{2k/s}\}$ is a bit different from that of the ball $\Bcal_{D^\ast}(x, r_{2k/s})=\{y:D^\ast(x,y)\le r_{2k/s}\}$, compared with the proof of Theorem~\ref{thm:knn}. We write the support of $X$ as $\Xcal\subset \RR^d$, which is a compact set, i.e., $\sup_{x\in \Xcal}\|x\|\le O$ for some constant $O$. Given a radius $r$ and $x\in \Xcal$, if $y\in \Bcal_{D_{\hat{M}^\ast}}(x, r)$, i.e., $(y-x)^T\hat{M}^\ast(y-x)\le r$, then we can know that 
	$$
	(y-x)^TM^\ast(y-x)=(y-x)^T\hat{M}^\ast(y-x)+(y-x)^T(M^\ast-\hat{M}^\ast)(y-x)\le r+O^2{\delta^\ast}.
	$$
	This means
	$$
	\Bcal_{D_{\hat{M}^\ast}}(x, r)\cap\Xcal\subset \Bcal_{D^\ast}(x, r+O^2\delta^\ast)\cap\Xcal.
	$$
	We can choose $r_{2k/s}\le C\lambda_K(\kappa^{K-1}k/s)^{1/K}$ to make sure $\Bcal_{D_{\hat{M}^\ast}}(x, r_{2k/s})=2k/s$. When $\delta^\ast\lesssim \lambda_K(\kappa^{K-1}k/s)^{1/K}$, then $O^2\delta^\ast\le C\lambda_K(\kappa^{K-1}k/s)^{1/K}$. So we can know that 
	$$
	\Bcal_{D_{\hat{M}^\ast}}(x, r_{2k/s})\cap\Xcal\subset \Bcal_{D^\ast}(x, C\lambda_K(\kappa^{K-1}k/s)^{1/K})\cap\Xcal.
	$$
	This leads to
	$$
	|\hat{\eta}^\ast(x)-\eta(x)|\le \sup_{y\in \Bcal_{D_{\hat{M}^\ast}}(x, r_{2k/s})}\left|\eta(y)-\eta(x)\right|\le L(C(\kappa^{K-1}k/s)^{1/K})^\alpha.
	$$
	Then we can follow the same analysis as in the proof of Theorem~\ref{thm:knn} and conclude
	$$
	r(D_{\hat{M}^\ast}) \le  C(s/\kappa^{K-1})^{-\alpha(1+\beta)/(2\alpha+K)}.
	$$
	
	\paragraph{Proof for $D_{\hat{M}^{\ast\ast}}$}
	
	We can adopt a similar strategy for $D_{\hat{M}^{\ast\ast}}$. In particular, we can show that when $\delta^{\ast\ast}\lesssim (k/s)^{1/K}$, 
	$$
	\Bcal_{D_{\hat{M}^{\ast\ast}}}(x, r_{2k/s})\cap\Xcal\subset \Bcal_{D^{\ast\ast}}(x, C(k/s)^{1/K})\cap\Xcal.
	$$
	Therefore, we can conclude
	$$
	r(D_{\hat{M}^{\ast\ast}}) \le  Cs^{-\alpha(1+\beta)/(2\alpha+K)}.
	$$

	%%%%%%%%%%%%%%%%%%%%%%%%%%%%%%%%%%%%%%%%%%%%%%%
	\subsubsection{Two-sample testing}
	%%%%%%%%%%%%%%%%%%%%%%%%%%%%%%%%%%%%%%%%%%%%%%%
	
	\paragraph{Proof for $D_{\hat{M}^\ast}$} If we follow the same analysis in proof of Theorem~\ref{thm:twosample}, we can show that the required signal is 
	$$
	\|\hat{B}^T\mu\|^2\ge {C\over s}\|\hat{B}^T(BB^T+\Sigma)\hat{B}\|_F,
	$$
	where $\hat{B}\hat{B}^T=\hat{M}^{\ast}$ and $\hat{B}=\argmin_{EE^T=\hat{M}^{\ast}}\|E-B\|$. By definition, we have
	$$
	\left|\mu^T\hat{B}\hat{B}^T\mu-\mu^TBB^T\mu\right|\le \|\hat{B}\hat{B}^T-BB^T\|\|\mu\|^2\le \delta^\ast\|\mu\|^2.
	$$
	Since $\delta^\ast=o(\lambda_K)$, we have $\|\hat{B}^T\mu\|^2\ge (1+o(1))\lambda_K\|\mu\|^2$. Observe that
	\begin{align*}
		\|\hat{B}^T(BB^T+\Sigma)\hat{B}\|_F&\le \|\hat{B}^T(BB^T+\Sigma)B\|_F+\|\hat{B}^T(BB^T+\Sigma)\|\|\hat{B}-B\|_F\\
		&\le \|\hat{B}^T(BB^T+\Sigma)B\|_F+C(\lambda_1+\delta^\ast)^{3/2}\sqrt{K}(\delta^\ast)^{1/2}\\
		&\le \|B^T(BB^T+\Sigma)B\|_F+2C(\lambda_1+\delta^\ast)^{3/2}\sqrt{K}(\delta^\ast)^{1/2}\\
		&\le \|B^T(BB^T+\Sigma)B\|_F+(2+o(1))C\sqrt{K}\lambda_1^{3/2}(\delta^\ast)^{1/2}
	\end{align*}
	Since $\|B^T(BB^T+\Sigma)B\|_F\ge \sqrt{K}\lambda_K^2$, we can know that 
	$$
	\|\hat{B}^T(BB^T+\Sigma)\hat{B}\|_F\le C\|B^T(BB^T+\Sigma)B\|_F.
	$$
	This suggests that a sufficient condition for $\|\hat{B}^T\mu\|^2\ge {C}\|\hat{B}^T(BB^T+\Sigma)\hat{B}\|_F/s$ is 
	$$
	\|\mu\|^2\ge {C'\over s}\|B^T(BB^T+\Sigma)B\|_F
	$$ 
	for a large enough constant $C'$. We complete the proof for 
	$$
	r(D_{\hat{M}^\ast},\epsilon)\lesssim {\|B^T(BB^T+\Sigma)B\|_F^{1/2}\over \sqrt{s}}.
	$$
	
	\paragraph{Proof for $D_{\hat{M}^{\ast\ast}}$} We work on $D_{\hat{M}^{\ast\ast}}$ and define $\hat{U}=\argmin_{EE^T=\hat{M}^{\ast\ast}}\|E-U\|$. So 
	$$
	\left|\mu^T\hat{U}\hat{U}^T\mu-\mu^TUU^T\mu\right|\le \|\hat{U}\hat{U}^T-UU^T\|\|\mu\|^2\le {\delta^{\ast\ast}}\|\mu\|^2.
	$$
	and 
	$$
	\|\hat{U}^T(BB^T+\Sigma)\hat{U}\|_F\le \|U^T(BB^T+\Sigma)U\|_F+C\sqrt{K}\lambda_1(\delta^{\ast\ast})^{1/2}
	$$
	Because $\|U^T(BB^T+\Sigma)U\|_F\ge \sqrt{K}\lambda_K$, we can conclude
	$$
	\|\hat{U}^T(BB^T+\Sigma)\hat{U}\|_F\le (1+C(\delta^{\ast\ast})^{1/2})\|U^T(BB^T+\Sigma)U\|_F.
	$$
	Since $\delta^{\ast\ast}=o(1)$, we can know
	$$
	r(D_{\hat{M}^{\ast\ast}},\epsilon)\lesssim {\|U^T(BB^T+\Sigma)U\|_F^{1/2}\over \sqrt{s}}.
	$$

	%%%%%%%%%%%%%%%%%%%%%%%%%%%%%%%%%%%%%%%%%%%%%%%
	\subsubsection{$k$-means clustering}
	%%%%%%%%%%%%%%%%%%%%%%%%%%%%%%%%%%%%%%%%%%%%%%%
	
	\paragraph{Proof for $D_{\hat{M}^\ast}$} 
	
	If we follow the same procedure in proof of Theorem~\ref{thm:kmeans}, we can know that 
	$$
	r(D_{\hat{M}^\ast})\le \PP\left((1+a_s)\|\hat{B}^T\mu\|^2\le 2\langle \hat{B}^T\xi_1,\hat{B}^T\mu+\hat{B}^T\xi_2-\hat{B}^T\xi_3\rangle\right),
	$$
	where $\xi_1\sim N(0,\Sigma_\pm)$, $\xi_2\sim N(0,\Sigma_\pm/s_+)$, $\xi_3\sim N(0,\Sigma_\pm/s_-)$ and $a_s\to 0$. Observe that
	$$
	\left|\|\hat{B}^T\mu\|^2-\|B^T\mu\|^2\right|\le \|\hat{B}\hat{B}^T-BB^T\|\|\mu\|^2\le \delta^\ast\|\mu\|^2\le {\delta^\ast\over \lambda_K}\|B^T\mu\|^2
	$$
	$\delta^\ast=o(\lambda_K)$ suggests $\|\hat{B}^T\mu\|^2=(1+o(1))\|B^T\mu\|^2$. We define the following notations
	$$
	\Delta_{B,1}=\xi_1^T(\hat{B}\hat{B}^T-BB^T)\mu\qquad {\rm and}\qquad \Delta_{B,2}=\xi_1^T(\hat{B}\hat{B}^T-BB^T)(\xi_2-\xi_3).
	$$
	Note that $\EE(\Delta_{B,1})=0$ and $\Var(\Delta_{B,1})\le (\delta^\ast)^2\|\Sigma_\pm\|\|\mu\|^2=o(\|B^T\mu\|^4)$. This suggests that $\Delta_{B,1}=o_p(\|B^T\mu\|^2)$. Similarly, we can show that $\EE(\Delta_{B,2})=0$ and 
	$$
	\Var(\Delta_{B,2})\le {C\over s}\|\Sigma_\pm^{1/2}(\hat{B}\hat{B}^T-BB^T)\Sigma_\pm^{1/2}\|_F^2\le {C\over s}\|\Sigma_\pm\|_F^2\|\hat{B}\hat{B}^T-BB^T\|^2\le {C(\delta^\ast)^2\|B^T\mu\|^4}
	$$
	Here, we use $\|B^T\mu\|^2\gg \Psi(\Lambda^2+B^T\Sigma_\pm B)$. So we can also conclude that $\Delta_{B,2}=o_p(\|B^T\mu\|^2)$. Putting all together suggest that 
	$$
	r(D_{\hat{M}^\ast})\le \PP\left((1+b_s)\|B^T\mu\|^2\le 2\langle B^T\xi_1,B^T\mu+B^T\xi_2-B^T\xi_3\rangle\right)
	$$
	where $b_s\to 0$ is different from $a_s$. Therefore, we can still obtain 
	$$
	r(D_{\hat{M}^\ast})\le \Gamma(1+o(1),B^T\mu,B^T\Sigma_\pm B).
	$$
	
	\paragraph{Proof for $D_{\hat{M}^{\ast\ast}}$} We use the similar strategy to work on $D_{\hat{M}^{\ast\ast}}$. For $D_{\hat{M}^{\ast\ast}}$, we have 
	$$
	r(D_{\hat{M}^{\ast\ast}})\le \PP\left((1+a_s)\|\hat{U}^T\mu\|^2\le 2\langle \hat{U}^T\xi_1,\hat{U}^T\mu+\hat{U}^T\xi_2-\hat{U}^T\xi_3\rangle\right).
	$$
	for some sequence $a_s\to 0$. If we apply the same analysis for $D_{\hat{M}^\ast}$, we can show that $\|\hat{U}^T\mu\|^2=(1+o(1))\|\mu\|^2$, $\Delta_{U,1}=o_p(\|\mu\|^2)$ and $\Delta_{U,2}=o_p(\|\mu\|^2)$, where 
	$$
	\Delta_{U,1}=\xi_1^T(\hat{U}\hat{U}^T-UU^T)\mu\qquad {\rm and}\qquad \Delta_{U,2}=\xi_1^T(\hat{U}\hat{U}^T-UU^T)(\xi_2-\xi_3).
	$$
	So we can conclude that 
	$$
	r(D_{\hat{M}^\ast})\le \PP\left((1+b_s)\|\mu\|^2\le 2\langle U^T\xi_1,\mu+U^T\xi_2-U^T\xi_3\rangle\right)
	$$
	for some $b_s\to 0$. This complete the proof. 
	
	%%%%%%%%%%%%%%%%%%%%%%%%%%%%%%%%%%%%%%%%%%%%%%%
	\subsubsection{Sample identification}
	%%%%%%%%%%%%%%%%%%%%%%%%%%%%%%%%%%%%%%%%%%%%%%%
	
	\paragraph{Proof for $D_{\hat{M}^\ast}$} Following the same strategy in proof of Theorem~\ref{thm:identification}, we can show that we can control type II error if 
	$$
	\|\hat{B}^TB(Z_1-Z_2)\|^2\ge C\|\hat{B}^T\Sigma \hat{B}\|_F.
	$$
	We observe that 
	$$
	|\|\hat{B}^TB(Z_1-Z_2)\|^2-\|\Lambda(Z_1-Z_2)\|^2|\le \|\hat{B}\hat{B}^T-BB^T\|\|B(Z_1-Z_2)\|^2\le \delta^\ast \|B(Z_1-Z_2)\|^2.
	$$
	Hence,
	$$
	\|\hat{B}^TB(Z_1-Z_2)\|^2\ge (1+o(1))\lambda_K^2 \|Z_1-Z_2\|^2.
	$$
	Note that
	\begin{align*}
		\|\hat{B}^T\Sigma \hat{B}\|_F&\le \|\hat{B}^T\Sigma B\|_F+ \|\hat{B}^T\Sigma\|\|\hat{B}-B\|_F\\
		&\le  \|\hat{B}^T\Sigma B\|_F+ C(\lambda_1+\delta^\ast)^{1/2}\|\Sigma\|\sqrt{K}(\delta^\ast)^{1/2}\\
		&\le  \|B^T\Sigma B\|_F+ 2C(\lambda_1+\delta^\ast)^{1/2}\|\Sigma\|\sqrt{K}(\delta^\ast)^{1/2}.
	\end{align*}
	Since $\lambda_d(\Sigma)\ge c\|\Sigma\|$, we can know $\|B^T\Sigma B\|_F\ge c\sqrt{K}\|\Sigma\|\lambda_K$, which leads to
	$$
	\|\hat{B}^T\Sigma \hat{B}\|_F\le (1+o(1))\|B^T\Sigma B\|_F.
	$$
	Therefore, we can control type II error if 
	$$
	\lambda_K^2\|Z_1-Z_2\|^2\ge C\|B^T\Sigma B\|_F.
	$$
	
	\paragraph{Proof for $D_{\hat{M}^{\ast\ast}}$} We now work on $D_{\hat{M}^{\ast\ast}}$ and the proof is similar to the one for $D_{\hat{M}^\ast}$. We still follow the proof of Theorem~\ref{thm:identification} and show that we can show that we can control type II error if 
	$$
	\|\hat{U}^TB(Z_1-Z_2)\|^2\ge C\|\hat{U}^T\Sigma \hat{U}\|_F.
	$$
	So, we can show that 
	$$
	\|\hat{U}^TB(Z_1-Z_2)\|^2\ge (1+o(1))\lambda_K \|Z_1-Z_2\|^2
	$$
	and 
	$$
	\|\hat{U}^T\Sigma \hat{U}\|_F\le (1+o(1))\|U^T\Sigma U\|_F.
	$$
	So, we now conclude that we can control type II error provided
	$$
	\lambda_K\|Z_1-Z_2\|^2\ge C\|U^T\Sigma U\|_F.
	$$

	%%%%%%%%%%%%%%%%%%%%%%%%%%%%%%%%%%%%%%%%%%%%%%%
	\subsection{Proof of Theorem~\ref{thm:identification}}
	%%%%%%%%%%%%%%%%%%%%%%%%%%%%%%%%%%%%%%%%%%%%%%%
	
	\subsubsection{Upper bound}
	As $T_D$ is chosen as upper $\epsilon/2$-quantile of $D(X_{i,1},X_{i,2})$, the we can control type I error at $\epsilon/2$ level no matter what distance is used. Now, we work on type II error. Because $\epsilon_i\sim N(0,\Sigma)$, under null hypothesis, we have 
	$$
	\|X_1-X_2\|^2=\|\epsilon_1-\epsilon_2\|^2\sim 2\sum_{l=1}^da_lZ_l^2,
	$$ 
	where $Z_l$ follow independent standard normal distribution and $a_1\ge \ldots \ge a_d$ are eigenvalues of $\Sigma$. By Lemma 1 in \cite{laurent2000adaptive}, we have 
	$$
	\PP\left(\sum_{l=1}^dZ_l^2>\sum_{l=1}^da_l+2\sqrt{(\sum_l a_l^2)t}+2a_1t\right)\le \exp(-t). 
	$$
	Thus, we can know $T_{\|\cdot\|^2}\ge 2(\sum_{l=1}^da_l+C\sqrt{\sum_l a_l^2})$. Under the alternative hypothesis, 
	$$
	\|X_1-X_2\|^2=\|B(Z_1-Z_2)\|^2+2(Z_1-Z_2)^TB^T(\epsilon_1-\epsilon_2)+\|\epsilon_1-\epsilon_2\|^2.
	$$ 
	Clearly, $\|B(Z_1-Z_2)\|^2\ge \lambda_K\|Z_1-Z_2\|^2$. As 
	$$
	(Z_1-Z_2)^TB^T(\epsilon_1-\epsilon_2)\sim N(0, V),
	$$ 
	where $V=(Z_1-Z_2)^TB^T\Sigma B(Z_1-Z_2)$, there exist a constant $C$ such that
	$$
	\PP\left(2(Z_1-Z_2)^TB^T(\epsilon_1-\epsilon_2)<-C\sqrt{V}\right)\le \epsilon/4.
	$$
	An application of Lemma 1 in \cite{laurent2000adaptive} suggests 
	$$
	\PP\left(\|\epsilon_1-\epsilon_2\|^2<2 \left(\sum_{l=1}^da_l-C\sqrt{\sum_l a_l^2}\right)\right)\le \epsilon/4.
	$$
	Because $V\le \|Z_1-Z_2\|^2\lambda_1a_1\le \|Z_1-Z_2\|^2\sqrt{\sum_l a_l^2}$, we can control type II error at $\epsilon/2$ level when 
	$$
	\lambda_K\|Z_1-Z_2\|^2\ge C\sqrt{\sum_l a_l^2}+C\|Z_1-Z_2\|\left(\sum_l a_l^2\right)^{1/4}.
	$$
	If we note $\sum_l a_l^2=\|\Sigma\|_F^2$, this finishes the proof for 
	$$
	r(\|\cdot\|^2,\epsilon)\le C{\|\Sigma\|_F^{1/2}\over \sqrt{\lambda_K}}.
	$$
	
	Next, we can show the results for $D^\ast$ in a similar way as Euclidean distance. Under the null hypothesis, we can show 
	$$
	T_{D^\ast}\ge 2\left({\rm Tr}(B^T\Sigma B) +C\|B^T\Sigma B\|_F\right).
	$$
	Similarly, under alternative hypothesis, 
	\begin{align*}
		&D^\ast(X_1,X_2)\\
		=&\|B^T(X_1-X_2)\|^2\\
		=&\|B^TB(Z_1-Z_2)\|^2+2(\epsilon_1-\epsilon_2)^TBB^TB(Z_1-Z_2)+\|B^T(\epsilon_1-\epsilon_2)\|^2.
	\end{align*}
	Thus, we can control type II error at $\epsilon/2$ level if 
	$$
	\lambda_K^2\|Z_1-Z_2\|^2\ge C\|B^T\Sigma B\|_F.
	$$
	for a large enough constant $C$. Thus, we can conclude that
	$$
	r(D^\ast,\epsilon)\le C{\|B^T\Sigma B\|_F^{1/2}\over \lambda_K}.
	$$
	Finally, we can work on $D^{\ast\ast}$ in a exact the same way to show
	$$
	r(D^{\ast\ast},\epsilon)\le C{\|U^T\Sigma U\|_F^{1/2}\over \sqrt{\lambda_K}}.
	$$
	
	\subsubsection{Lower bound} 
	
	We now work on the lower bound. The main idea of lower bound proof is to derive the type II error under local alternative hypothesis 
	$$
	\tilde{H}_1(r)=\left\{Z_1-Z_2=re_K\right\}.
	$$
	
	We first work on Euclidean distance when $r=o({\|\Sigma\|_F^{1/2}/\sqrt{\lambda_K}})$. We decompose $\|X_1-X_2\|^2$ into three parts
	$$
	\|X_1-X_2\|^2=\underbrace{\|B(Z_1-Z_2)\|^2}_{E_1}+\underbrace{2(Z_1-Z_2)^TB^T(\epsilon_1-\epsilon_2)}_{E_2}+\underbrace{\|\epsilon_1-\epsilon_2\|^2}_{E_3}.
	$$ 
	Because $Z_1-Z_2=re_K$, $E_1=\lambda_Kr^2$ and ${\rm Var}(E_2)=4\lambda_Kr^2u_K^T\Sigma u_K$. Because $r=o({\|\Sigma\|_F^{1/2}/\sqrt{\lambda_K}})$, we can know that $E_1+E_2=o_p(\|\Sigma\|_F)$. Since the proof in upper bound shows that 
	$$
	T_{\|\cdot\|^2}-2{\rm Tr}(\Sigma)\ge C\|\Sigma\|_F,
	$$
	we can conclude that
	$$
	\PP\left(\|X_1-X_2\|^2\ge T_{\|\cdot\|^2}\middle|\tilde{H}_1(r)\right)=\PP\left({E_3-2{\rm Tr}(\Sigma)\over \|\Sigma\|_F}\ge {T_{\|\cdot\|^2}-2{\rm Tr}(\Sigma)\over \|\Sigma\|_F}(1+o(1)) \middle|\tilde{H}_1(r)\right)\to \alpha. 
	$$
	
	We next work on $D^\ast$ when $r=o({\|B^T\Sigma B\|_F^{1/2}/\lambda_K})$. With a similar strategy, we also decompose $\|B^T(X_1-X_2)\|^2$ as $E_1$, $E_2$ and $E_3$
	$$
	\|B^T(X_1-X_2)\|^2=\underbrace{\|B^TB(Z_1-Z_2)\|^2}_{E_1}+\underbrace{2(\epsilon_1-\epsilon_2)^TBB^TB(Z_1-Z_2)}_{E_2}+\underbrace{\|B^T(\epsilon_1-\epsilon_2)\|^2}_{E_3}
	$$ 
	and show that $E_1+E_2=o_p(\|B^T\Sigma B\|_F)$. So we can also have
	$$
	\PP\left(D^\ast(X_1,X_2)\ge T_{D^\ast}\middle|\tilde{H}_1(r)\right)\to \alpha. 
	$$
	The results for $D^{\ast\ast}$ can be proved in a same way. 
	
	%%%%%%%%%%%%%%%%%%%%%%%%%%%%%%%%%%%%%%%%%%%%%%%
	\subsection{Proof of Theorem~\ref{thm:kmeans}}
	%%%%%%%%%%%%%%%%%%%%%%%%%%%%%%%%%%%%%%%%%%%%%%%
	
	Without loss of generality, we assume $\EE(X|Y=1)=\mu/2$ and $\EE(X|Y=-1)=-\mu/2$ and write $\epsilon_i=X_i-\EE(X_i|Y_i)$ rather than $\epsilon$ in  latent factor model in the proof. 
	
	\subsubsection{Upper bound}
	We first show the proof when the Euclidean distance is used. The analysis is similar with proof in \cite{lu2016statistical}.  Then, we discuss the case of Mahalanobis distance. In this proof, we define a generalized version of mis-clustering rate as
	$$
	h(\hat{Y},Y)=\max\left({ |\{i:\hat{Y}_i=-1, Y_i=1\}| \over \min(|\{i:\hat{Y}_i=-1\}|,|\{i:Y_i=1\}|)},{|\{i:\hat{Y}_i=1, Y_i=-1\}|\over \min(|\{i:\hat{Y}_i=1\}|,|\{i:Y_i=-1\}|)}\right).
	$$
	We also write $h^{(t)}=h(\hat{Y}^{(t)},Y)$. 
	In addition, we define the following events
	$$
	\Acal_{1,+}=\left\{\left\|\bar{\epsilon}_+-\bar{\epsilon}_+^\ast\right\|^2\le 3{h(\hat{Y},Y)\over \min(s_+,\tilde{s}_+)}\left({\rm Tr}(\Sigma_{\pm})+\sqrt{{\rm Tr}(\Sigma_{\pm}^2)s}+\|\Sigma_{\pm}\|s\right)\quad \forall \ \hat{Y}_i\right\},
	$$
	where $\bar{\epsilon}_+^\ast=\sum_{Y_i=1}\epsilon_i/s_{+}$ and $\bar{\epsilon}_+=\sum_{\hat{Y}_i=1}\epsilon_i/|\{i:\hat{Y}_i=1\}|$. Here, $\forall \ \hat{Y}_i$ means all possible realization of $\hat{Y}_i$. Similarly, we can define events $\Acal_{1,-}$ for $\bar{\epsilon}_+=\sum_{\hat{Y}_i=-1}\epsilon_i/|\{i:\hat{Y}_i=-1\}|$ and $\bar{\epsilon}_-^\ast=\sum_{Y_i=-1}\epsilon_i/s_{-}$. We also define 
	$$
	\Acal_{2}=\left\{s_+\|\bar{\epsilon}_+^\ast\|^2,s_-\|\bar{\epsilon}_-^\ast\|^2\le \left({\rm Tr}(\Sigma_{\pm})+\sqrt{{\rm Tr}(\Sigma_{\pm}^2)\log s}+\|\Sigma_{\pm}\|\log s \right)\right\},
	$$
	$$
	\Acal_{3}=\left\{\left\|{1\over s}\sum_i \epsilon_i\epsilon_i^T\right\|\le \|\Sigma_{\pm}\|\left(1+{d\over s}\right), \right\}
	$$
	and
	$$
	\Acal_{4}=\left\{s/2-\sqrt{8s\log s}\le s_+\le s/2+\sqrt{8s\log s}\right\}.
	$$
	By Lemma~\ref{lm:kmeansevent}, we can know that
	$$
	\PP\left(\Acal_{1,+}\cap\Acal_{1,-}\cap \Acal_{2}\cap \Acal_{3}\cap \Acal_{4}\right)\ge 1-s^{-5}.
	$$
	We conduct the analysis conditioning on these events. 
	
	%We also define the convergence rate of center estimation in this proof
	%$$
	%\eta^{(t)}=\max\left(D\left(\hat{\mu}_+^{(t)},{\mu \over 2}\right), D\left(\hat{\mu}_-^{(t)},-{\mu \over 2}\right)\right).
	%$$
	
	\paragraph{Step 1a: centroid of cluster} We start with one step analysis. At the $t$th step, the centroids of each group are 
	$$
	\hat{\mu}_+^{(t)}={1\over s_+^{(t)}}\sum_{\hat{Y}_i^{(t)}=1}X_i\qquad {\rm and}\qquad \hat{\mu}_-^{(t)}={1\over s_-^{(t)}}\sum_{\hat{Y}_i^{(t)}=-1}X_i,
	$$
	where $s_+^{(t)}=|\{i:\hat{Y}_i^{(t)}=1\}|$ and $s_-^{(t)}=|\{i:\hat{Y}_i^{(t)}=-1\}|$. We can then decompose the error of centroids into two parts: the uncertainty due to noise and mis-clustering
	\begin{align*}
		\hat{\mu}_+^{(t)}-{\mu \over 2}&={1\over s_+^{(t)}}\sum_{\hat{Y}_i^{(t)}=1,Y_i=1}\left(X_i-{\mu \over 2}\right)+{1\over s_+^{(t)}}\sum_{\hat{Y}_i^{(t)}=1, Y_i=-1}\left(X_i-{\mu \over 2}\right)\\
		&={1\over s_+^{(t)}}\sum_{\hat{Y}_i^{(t)}=1,Y_i=1}\epsilon_i+{1\over s_+^{(t)}}\sum_{\hat{Y}_i^{(t)}=1, Y_i=-1}\left(-\mu+\epsilon_i\right)\\
		&={1\over s_+^{(t)}}\sum_{\hat{Y}_i^{(t)}=1}\epsilon_i-\mu{s_{+-}^{(t)}\over s_+^{(t)}}\\
		&=\bar{\epsilon}_+^{(t)}-\mu{s_{+-}^{(t)}\over s_+^{(t)}}.
	\end{align*}
	Here, $s_{+-}^{(t)}=|\{i:\hat{Y}_i^{(t)}=1, Y_i=-1\}|$ and $\bar{\epsilon}_+^{(t)}=\sum_{\hat{Y}_i^{(t)}=1}\epsilon_i/s_{+}^{(t)}$. Similarly, if we write $s_{-+}^{(t)}=|\{i:\hat{Y}_i^{(t)}=-1, Y_i=1\}|$ and $\bar{\epsilon}_-^{(t)}=\sum_{\hat{Y}_i^{(t)}=-1}\epsilon_i/s_{-}^{(t)}$, 
	$$
	\hat{\mu}_-^{(t)}+{\mu \over 2}=\bar{\epsilon}_-^{(t)}+\mu{s_{-+}^{(t)}\over s_-^{(t)}}.
	$$
	This clearly suggests that
	$$
	\hat{\mu}_+^{(t)}-\hat{\mu}_-^{(t)}=\mu\left(1-{s_{+-}^{(t)}\over s_+^{(t)}}-{s_{-+}^{(t)}\over s_-^{(t)}}\right)+\bar{\epsilon}_+^{(t)}-\bar{\epsilon}_-^{(t)}
	$$
	and
	$$
	\hat{\mu}_+^{(t)}+\hat{\mu}_-^{(t)}=\mu\left({s_{-+}^{(t)}\over s_-^{(t)}}-{s_{+-}^{(t)}\over s_+^{(t)}}\right)+\bar{\epsilon}_+^{(t)}+\bar{\epsilon}_-^{(t)}.
	$$
	Note that 
	\begin{align*}
		\left\|\hat{\mu}_+^{(t)}-{\mu \over 2}\right\|&=\left\|\bar{\epsilon}_+^{(t)}-\mu{s_{+-}^{(t)}\over s_+^{(t)}}\right\|\\
		&\le \left\|\bar{\epsilon}_+^{(t)}-\bar{\epsilon}_+^\ast\right\|+\|\bar{\epsilon}_+^\ast\|+h^{(t)}\|\mu\|.
	\end{align*}
	Since the analysis is conditioned on $\Acal_{1,+}$, $\Acal_2$ and $\|\mu\|^2>v(\|\Sigma_{\pm}\|+{\rm Tr}(\Sigma_{\pm})/s)$, we have
	$$
	\left\|\bar{\epsilon}_+^{(t)}-\bar{\epsilon}_+^\ast\right\|^2\le 3{h^{(t)}\over \min(s_+,\tilde{s}_+)}\left({\rm Tr}(\Sigma_{\pm})+\sqrt{{\rm Tr}(\Sigma_{\pm}^2)s}+\|\Sigma_{\pm}\|s\right)\le  9{h^{(t)}\over \min(s_+,\tilde{s}_+)}{s\|\mu\|^2\over v}
	$$
	and 
	$$
	\|\bar{\epsilon}_+^\ast\|\le {{\rm Tr}(\Sigma_{\pm})+\sqrt{{\rm Tr}(\Sigma_{\pm}^2)\log s}+\|\Sigma_{\pm}\|\log s \over s_+}\le {3\over s_+}{s\|\mu\|^2\over v}
	$$
	Since $h^{(t)}<1/2$ and $s_+\ge s/2-\sqrt{s\log s}$, we can conclude that
	$$
	\left\|\hat{\mu}_+^{(t)}-{\mu \over 2}\right\|\le \left(\sqrt{C\over v}+h^{(t)}\right)\|\mu\|.
	$$
	Similarly, we can know that 
	$$
	\left\|\hat{\mu}_-^{(t)}+{\mu \over 2}\right\|\le \left(\sqrt{C\over v}+h^{(t)}\right)\|\mu\|.
	$$
	
	\paragraph{Step 1b:  sample assignment} 
	We now bound $s_{-+}^{(t+1)}$ by the results from step 1a
	\begin{align*}
		s_{-+}^{(t+1)}&=\sum_{i}\bI(\hat{Y}_i^{(t+1)}=-1, Y_i=1)\\
		&=\sum_{Y_i=1}\bI\left(\left\|{\mu\over 2}+\epsilon_i-\hat{\mu}_-^{(t)}\right\|^2\le \left\|{\mu\over 2}+\epsilon_i-\hat{\mu}_+^{(t)}\right\|^2\right)\\
		&=\sum_{Y_i=1}\bI\left(\left\|{\mu\over 2}-\hat{\mu}_-^{(t)}\right\|^2- \left\|{\mu\over 2}-\hat{\mu}_+^{(t)}\right\|^2\le 2\langle \epsilon_i,\hat{\mu}_-^{(t)}-\hat{\mu}_+^{(t)}\rangle \right)\\
		&=\sum_{Y_i=1}\bI\left(\left\|\mu\right\|^2\left(1-\sqrt{C\over v}-h^{(t)}\right) \le 2\langle \epsilon_i,\mu+ \bar{\epsilon}_+^\ast-\bar{\epsilon}_-^\ast\rangle+\Delta_{i}^{(t)} \right),
	\end{align*} 
	where $\Delta_{i}^{(t)}$ is defined as
	$$
	\Delta_{i}^{(t)}=2\left\langle \epsilon_i, -\mu\left({s_{+-}^{(t)}\over s_+^{(t)}}+{s_{-+}^{(t)}\over s_-^{(t)}}\right)+\left(\bar{\epsilon}_+^{(t)}-\bar{\epsilon}_+^\ast\right)-\left(\bar{\epsilon}_-^{(t)}-\bar{\epsilon}_-^\ast\right)\right\rangle.
	$$
	Since $\bI(x<y)\le y^2/x^2$, we have 
	\begin{align*}
		s_{-+}^{(t+1)}&\le \sum_{Y_i=1}\bI\left(\left\|\mu\right\|^2\left(1-\sqrt{C\over v}-h^{(t)}\right)\le 2\langle \epsilon_i,\mu+ \bar{\epsilon}_+^\ast-\bar{\epsilon}_-^\ast\rangle+\Delta_{i}^{(t)} \right)\\
		&\le \sum_{Y_i=1}\bI\left(\left\|\mu\right\|^2\left(1-Cv^{-1/4}-h^{(t)}\right) \le 2\langle \epsilon_i,\mu+ \bar{\epsilon}_+^\ast-\bar{\epsilon}_-^\ast\rangle\right)+\bI\left(\|\mu\|^2/v^{1/4} \le \Delta_{i}^{(t)}\right)\\
		&\le \sum_{Y_i=1}\bI\left(\left\|\mu\right\|^2\left(1-Cv^{-1/4}-h^{(t)}\right) \le 2\langle \epsilon_i,\mu+ \bar{\epsilon}_+^\ast-\bar{\epsilon}_-^\ast\rangle\right)+\sum_{Y_i=1}{{\Delta_{i}^{(t)}}^2\sqrt{v}\over  \|\mu\|^4}.
	\end{align*}
	We bound $\Delta_{i}^{(t)}$ term conditioned on event $\Acal_{1,+}$, $\Acal_{1,-}$, and $\Acal_3$
	\begin{align*}
		\sum_{Y_i=1}{ {\Delta_{i}^{(t)}}^2 \sqrt{v} \over \|\mu\|^4}&\le {v\over \|\mu\|^4}\sum_{Y_i=1} \left[\left\langle \epsilon_i, -\mu\left({s_{+-}^{(t)}\over s_+^{(t)}}+{s_{-+}^{(t)}\over s_-^{(t)}}\right)\right\rangle^2+\left\langle \epsilon_i, \left(\bar{\epsilon}_+^{(t)}-\bar{\epsilon}_+^\ast\right)-\left(\bar{\epsilon}_-^{(t)}-\bar{\epsilon}_-^\ast\right)\right\rangle^2\right] \\
		&\le {2\sqrt{v}{h^{(t)}}^2\over \|\mu\|^4}\sum_{Y_i=1} \left\langle \epsilon_i, \mu\right\rangle^2+{2\sqrt{v}\over \|\mu\|^4}\left\|\left(\bar{\epsilon}_+^{(t)}-\bar{\epsilon}_+^\ast\right)-\left(\bar{\epsilon}_-^{(t)}-\bar{\epsilon}_-^\ast\right)\right\|^2\left\|\sum_{Y_i=1} \epsilon_i\epsilon_i^T\right\|\\
		&\le \left[{2\sqrt{v}{h^{(t)}}^2\over \|\mu\|^2}+{2\sqrt{v}\over \|\mu\|^4}\left\|\left(\bar{\epsilon}_+^{(t)}-\bar{\epsilon}_+^\ast\right)-\left(\bar{\epsilon}_-^{(t)}-\bar{\epsilon}_-^\ast\right)\right\|^2\right]\left\|\sum_{Y_i=1} \epsilon_i\epsilon_i^T\right\|\\
		&\le \left[2\sqrt{v}{h^{(t)}}^2+{6\sqrt{v}h^{(t)}\over \|\mu\|^2\min(s_+,\tilde{s}_+)}\left({\rm Tr}(\Sigma_{\pm})+\sqrt{{\rm Tr}(\Sigma_{\pm}^2)s}+\|\Sigma_{\pm}\|s\right)\right]{\|\Sigma_{\pm}\|\over \|\mu\|^2}\left(s+d\right)\\
		&\le \left({C{h^{(t)}}\over \sqrt{v}}+{C{h^{(t)}}\over v}\right)s\\
		&\le {Cs{h^{(t)}}\over \sqrt{v}}
	\end{align*}
	Here, we use the fact $\|\mu\|^2>v(\|\Sigma_{\pm}\|+{\rm Tr}(\Sigma_{\pm})/s)$ and $h^{(t)}\le 1$. The bound for $\Delta_{i}^{(t)}$ help yield
	$$
	s_{-+}^{(t+1)}\le \sum_{Y_i=1}\bI\left(\left\|\mu\right\|^2\left(1-Cv^{-1/4}-h^{(t)}\right) \le 2\langle \epsilon_i,\mu+ \bar{\epsilon}_+^\ast-\bar{\epsilon}_-^\ast\rangle\right)+{Cs{h^{(t)}}\over \sqrt{v}}.
	$$
	We can bound $s_{+-}^{(t+1)}$ in a similar argument and show, when $v$ is large enough,
	$$
	h^{(t+1)}\le {1\over s}\sum_{i}\bI\left(\left\|\mu\right\|^2\left(1-Cv^{-1/4}-h^{(t)}\right) \le 2\langle \epsilon_i,\mu+ \bar{\epsilon}_+^\ast-\bar{\epsilon}_-^\ast\rangle\right)+{h^{(t)}\over 4}.
	$$

	\paragraph{Step 1c:  multiple iterations}
	We first show that for any $0<\delta<1/2$,
	\begin{equation}
		\label{eq:invariant}
		{1\over s}\sum_{i}\bI\left(\left\|\mu\right\|^2\left(1-\delta\right) \le 2\langle \epsilon_i,\mu+ \bar{\epsilon}_+^\ast-\bar{\epsilon}_-^\ast\rangle\right)\le \exp\left(-{(1-2\delta)^2\|\mu\|^4\over 8\Gamma(\mu,\Sigma_{\pm})}\right)
	\end{equation}
	with probability at least $1-\exp(-\sqrt{v}\|\mu\|)$. Here, $\Gamma(\mu,\Sigma_{\pm})=\max\left(\mu^T\Sigma_{\pm}\mu,\|\Sigma_{\pm}\|_F^2/s,\|\Sigma_{\pm}\|\|\mu\|^2/\sqrt{s}\right)$. We can note that 
	\begin{align*}
		&{1\over s}\sum_{i}\bI\left(\left\|\mu\right\|^2\left(1-\delta\right)/2 \le \langle \epsilon_i,\mu+ \bar{\epsilon}_+^\ast-\bar{\epsilon}_-^\ast\rangle\right)\\
		\le & {1\over s}\sum_{i}\bI\left(\left\|\mu\right\|^2\delta_1\le \langle \epsilon_i,\mu\rangle\right)+{1\over s}\sum_{i}\bI\left(\left\|\mu\right\|^2(\delta_2+\delta_3) \le \langle \epsilon_i, \bar{\epsilon}_+^\ast-\bar{\epsilon}_-^\ast\rangle\right),
	\end{align*}
	where $\delta_1+\delta_2+\delta_3=(1-\delta)/2$ and $\delta_1$, $\delta_2$ and $\delta_3$ will be specified later.
	For the first term, we have 
	$$
	\EE\left({1\over s}\sum_{i}\bI\left(\left\|\mu\right\|^2\delta_1 \le \langle \epsilon_i,\mu\rangle\right)\right)\le \exp\left(-{\delta_1^2\|\mu\|^4\over 2\mu^T\Sigma_{\pm}\mu}\right)
	$$
	because $\langle \epsilon_i,\mu\rangle$ follows the normal distribution $N(0,\mu^T\Sigma_{\pm}\mu)$. 
	For the second term, we have 
	\begin{align*}
		&\EE\left({1\over s}\sum_{i}\bI\left(\left\|\mu\right\|^2(\delta_2+\delta_3) \le \langle \epsilon_i, \bar{\epsilon}_+^\ast-\bar{\epsilon}_-^\ast\rangle\right)\right)\\
		\le &\EE\left({1\over s}\sum_{i}\bI\left(\left\|\mu\right\|^2\delta_2 \le \langle \epsilon_i, \tilde{\epsilon}_{+,i}^\ast-\tilde{\epsilon}_{-,i}^\ast\rangle\right)\right)+\EE\left({1\over s}\sum_{i}\bI\left(\left\|\mu\right\|^2\delta_3 \le \langle \epsilon_i, \epsilon_i/\min(s_+,s_-)\rangle\right)\right)\\
		\le & \PP\left(\left\|\mu\right\|^2\delta_2 \le \langle \epsilon_i, \tilde{\epsilon}_{+,i}^\ast-\tilde{\epsilon}_{-,i}^\ast\rangle\right)+\PP\left(\left\|\mu\right\|^2\delta_3 \le \langle \epsilon_i, \epsilon_i/\min(s_+,s_-)\rangle\right)\\
		\le & \exp\left(-\min\left({s\delta_2^2\|\mu\|^4\over 8\|\Sigma_{\pm}\|_F^2}, {\sqrt{s}\delta_2\|\mu\|^2\over 4\|\Sigma_{\pm}\|} \right)\right)+\exp\left(-\min\left({s^2\delta_3^2\|\mu\|^4\over 8\|\Sigma_{\pm}\|_F^2}, {s\delta_3\|\mu\|^2\over 8\|\Sigma_{\pm}\|} \right)\right)
	\end{align*}
	$\tilde{\epsilon}_{+,i}^\ast$ and $\tilde{\epsilon}_{-,i}^\ast$ are just $\bar{\epsilon}_+^\ast$ and $\bar{\epsilon}_-^\ast$ terms by removing $\epsilon_i$. Here, we apply Lemma 1 in \cite{laurent2000adaptive} and Lemma~\ref{lm:guassianprod} in the Section~\ref{sc:lemma}. If we select $\delta_1=\delta_2=(1-2\delta)/4$ and $\delta_3=\delta/2$, we prove
	$$
	\EE\left({1\over s}\sum_{i}\bI\left(\left\|\mu\right\|^2\left(1-\delta\right) \le \langle \epsilon_i,\mu+ \bar{\epsilon}_+^\ast-\bar{\epsilon}_-^\ast\rangle\right)\right)\le \exp\left(-{(1-2\delta)^2\|\mu\|^4\over 8\Gamma(\mu,\Sigma_{\pm})}\right).
	$$
	By Markov inequality, we show \eqref{eq:invariant}. Since $\|\mu\|^2>v(\|\Sigma_{\pm}\|+{\rm Tr}(\Sigma_{\pm})/s)$, with probability $1-\exp(-\sqrt{v}\|\mu\|)$, if $\delta<h$, then
	$$
	{1\over s}\sum_{i}\bI\left(\left\|\mu\right\|^2\left(1-\delta\right) \le 2\langle \epsilon_i,\mu+ \bar{\epsilon}_+^\ast-\bar{\epsilon}_-^\ast\rangle\right)\le e^{-v(1-2\delta)^2/8}\le {3h\over 8},
	$$
	where $h$ is the constant appears in condition (e) of Assumption~\ref{ap:kmeans}. This suggest that when $h^{(0)}<h$, with probability $1-\exp(-\sqrt{v}\|\mu\|)-s^{-5}$, we can always apply the results in step 1b, i.e.,
	\begin{align*}
		h^{(t+1)}&\le {1\over s}\sum_{i}\bI\left(\left\|\mu\right\|^2\left(1-Cv^{-1/4}-h^{(t)}\right) \le 2\langle \epsilon_i,\mu+ \bar{\epsilon}_+^\ast-\bar{\epsilon}_-^\ast\rangle\right)+{h^{(t)}\over 4}.
	\end{align*}
	We keep applying above results and \eqref{eq:invariant} to show 
	\begin{align*}
		h^{(t)}&\le {1\over s}\sum_{i}\bI\left(\left\|\mu\right\|^2\left(1+o(1)\right) \le 2\langle \epsilon_i,\mu+ \bar{\epsilon}_+^\ast-\bar{\epsilon}_-^\ast\rangle\right)\\
		&\le \Gamma(1+o(1),\mu,\Sigma_{\pm}).
	\end{align*}
	when $t>\log s$. 
	
	This shows that, when $\|\mu\|^2>v(\|\Sigma_{\pm}\|+{\rm Tr}(\Sigma_{\pm})/s)$, 
	$$
	r(\|\cdot\|^2) \le \Gamma(1+o(1),\mu,\Sigma_{\pm})
	$$
	with probability at least $1-s^{-5}-\exp(-\sqrt{v}\|\mu\|)$. 
	
	\paragraph{Step 2: Mahalanobis distance} With the similar argument, we can show that, when $\|B^T\mu\|^2>v(\|B^T\Sigma_{\pm} B\|+{\rm Tr}(B^T \Sigma_{\pm} B)/s)$, 
	$$
	r(D^\ast) \le \Gamma(1+o(1),B^T\mu,B^T\Sigma_{\pm}B)
	$$
	with probability at least $1-s^{-5}-\exp(-\sqrt{v}\|B^T\mu\|)$. Also, when $\|\mu\|^2>v(\|U^T\Sigma_{\pm} U\|+{\rm Tr}(U^T \Sigma_{\pm} U)/s)$, 
	$$
	r(D^{\ast\ast}) \le \Gamma(1+o(1),\mu,U^T\Sigma_{\pm}U)
	$$
	with probability at least $1-s^{-5}-\exp(-\sqrt{v}\|\mu\|)$.
	
	\subsubsection{Lower bound}
	
	We now work on the lower bound. The main idea of lower bound proof is that we reduce the clustering problem to a easier two point testing problem where all labels but one is known and then we predict the only one unknown label. We start with the case of Euclidean distance. We write $\bY=(Y_1,\ldots,Y_s)$ as the vector of true label and $\hat{\bY}=(\hat{Y}_1,\ldots,\hat{Y}_s)$ as the clustering result by $k$-means working with Euclidean distance. We also write the mis-clustering rate between $\bY$ and $\hat{\bY}$ as $r(\bY,\hat{\bY})$ when $\bY$ is given in this proof. Let $l=\lceil 15s/32\rceil$ and $\Ycal$ be a parameter space of $\bY$
	$$
	\Ycal=\left\{\bY:Y_1=\ldots=Y_l=1,Y_{l+1}=\ldots=Y_{2l}=-1,Y_{2l+1},\ldots,Y_s=1{\rm \ or\ }-1\right\}.
	$$
	In this space, we do not need to worry about label permutation and the mis-clustering rate $r(\bY,\hat{\bY})$ can be written as
	$$
	r(\bY,\hat{\bY})={1\over s}\sum_{i=1}^s \bI(\hat{Y}_i\ne Y_i).
	$$
	We are going to work on this parameter space. We observe that
	\begin{align*}
		\sup_{\bY\in \Ycal}\EE(r(\bY,\hat{\bY}))&\ge {1\over |\Ycal|}\sum_{\bY\in \Ycal}\EE_\bY\left({1\over s}\sum_{i=1}^s \bI(\hat{Y}_i\ne Y_i)\right)\\
		&\ge{1\over s|\Ycal|}\sum_{i=2l+1}^s\sum_{\bY\in \Ycal}\PP_\bY(\hat{Y}_i\ne Y_i)
	\end{align*}
	For each given $i=2l+1,\ldots,s$, we can split $\Ycal$ into two spaces $\Ycal_{i,+}$ and $\Ycal_{i,-}$ such that
	$$
	\Ycal_{i,+}=\{\bY\in\Ycal:Y_i=1\}\qquad {\rm and}\qquad \Ycal_{i,-}=\{\bY\in\Ycal:Y_i=-1\}.
	$$ 
	Clearly, we can define a one-to-one correspondence between $\Ycal_{i,+}$ and $\Ycal_{i,-}$, called $\pi_i$, such that only the $i$th entries of $\bY$ and $\pi_i(\bY)$ are different. So we have
	$$
	\sum_{\bY\in \Ycal}\PP_\bY(\hat{Y}_i\ne Y_i)=2\sum_{\bY\in \Ycal_{i,+}}\left({1\over 2}\PP_\bY(\hat{Y}_i\ne Y_i)+{1\over 2}\PP_{\pi_i(\bY)}(\hat{Y}_i\ne Y_i)\right).
	$$
	We now work on $\PP_\bY(\hat{Y}_i\ne Y_i)+\PP_{\pi_i(\bY)}(\hat{Y}_i\ne Y_i)$. This reduction suggests that we know all labels but the $i$th one, so we need to determine the last label $Y_{i}$. In $k$-means, we basically know 
	$$
	\hat{\mu}_+^{(t)}={1\over s_+}\sum_{Y_j=1,j\ne i}X_j \qquad {\rm and}\qquad  \hat{\mu}_-^{(t)}={1\over s_-}\sum_{Y_j=-1,j\ne i}X_j
	$$
	and the label $Y_{i}$ is determined by
	$$
	\hat{Y}_{i}^{(t+1)}=\argmin_{+1,-1} \left(\|X_{i}-\hat{\mu}_+^{(t)})\|,\|X_{i}-\hat{\mu}_-^{(t)}\|\right).
	$$
	If $\bY\in \Ycal_{i,+}$, $X_{i}\sim N(\mu/2,\Sigma_\pm)$, $\hat{\mu}_+^{(t)}\sim N(\mu/2,\Sigma_\pm/s_+)$ and $\hat{\mu}_-^{(t)}\sim N(-\mu/2,\Sigma_\pm/(s-1-s_+))$, then
	\begin{align*}
		\PP_{\bY}(\hat{Y}_{i}\ne Y_{i}) &=\PP_{\bY}\left(\|X_{i}-\hat{\mu}_-^{(t)}\|^2<\|X_{i}-\hat{\mu}_+^{(t)})\|^2\right)\\
		&=\PP_{\bY}\left(\|\mu+\epsilon_{i}-\bar{\epsilon}_-^\ast\|^2<\|\epsilon_{i}-\bar{\epsilon}_+^\ast\|^2\right)\\
		&=\PP_{\bY}\left(\|\mu\|^2+2\langle \mu,\epsilon_{i}-\bar{\epsilon}_-^\ast\rangle+ \|\epsilon_{i}-\bar{\epsilon}_-^\ast\|^2 < \|\epsilon_{i}-\bar{\epsilon}_+^\ast\|^2\right)\\
		&=\PP_{\bY}\left(\|\mu\|^2-2\langle \mu,\bar{\epsilon}_-^\ast\rangle +\|\bar{\epsilon}_-^\ast\|^2-\|\bar{\epsilon}_+^\ast\|^2 < 2\langle \bar{\epsilon}_-^\ast-\bar{\epsilon}_+^\ast-\mu, \epsilon_{i}\rangle \right).
	\end{align*}
	If $\|\Sigma_{\pm}\|+{\rm Tr}(\Sigma_{\pm})/s=o(\|\mu\|^2)$, then 
	\begin{align*}
		\PP_{\bY}(\hat{Y}_{i}\ne Y_{s})&\ge \PP_{\bY_1}\left(\|\mu\|^2(1+o(1))< 2\langle \bar{\epsilon}_-^\ast-\bar{\epsilon}_+^\ast-\mu, \epsilon_{i}\rangle \right)\\
		&=\Gamma(1+o(1),\mu,\Sigma_{\pm})
	\end{align*}
	Similarly, we can show 
	$$
	\PP_{\pi_i(\bY)}(\hat{Y}_{i}\ne Y_{i})\ge \Gamma(1+o(1),\mu,\Sigma_{\pm}).
	$$
	Hence, we have
	$$
	\sum_{\bY\in \Ycal}\PP_\bY(\hat{Y}_i\ne Y_i)\ge |\Ycal|\Gamma(1+o(1),\mu,\Sigma_{\pm}).
	$$
	Now, we can conclude that
	$$
	\sup_{\bY\in \Ycal}\EE(r(\bY,\hat{\bY}))\ge {s-2l\over s}\Gamma(1+o(1),\mu,\Sigma_{\pm})=\Gamma(1+o(1),\mu,\Sigma_{\pm})
	$$
	The last equality is due to $\|\Sigma_{\pm}\|+{\rm Tr}(\Sigma_{\pm})/s=o(\|\mu\|^2)$. The analysis for target distances in metric learning is similar, so we omit them here.

	%%%%%%%%%%%%%%%%%%%%%%%%%%%%%%%%%%%%%%%%%%%%%%%
	\subsection{Lemmas}
	\label{sc:lemma}
	%%%%%%%%%%%%%%%%%%%%%%%%%%%%%%%%%%%%%%%%%%%%%%%
	
	\begin{lemma}
		\label{lm:guassianprod}
		Suppose $\epsilon_1$ and $\epsilon_2$ are independent Gaussian random variables $N(0,\Sigma)$. Then, we have 
		$$
		\PP\left(\left|\langle \epsilon_1,\epsilon_2\rangle\right|>\|\Sigma\|t+\sqrt{2\|\Sigma\|_F^2t} \right)\le  2e^{-t}.
		$$
	\end{lemma}
	\begin{proof}
		We can write $\epsilon_1$ and $\epsilon_2$ as $\epsilon_1=\Sigma^{1/2}Z_1$ and $\epsilon_2=\Sigma^{1/2}Z_2$, where $Z_1$ and $Z_2$ are independent Gaussian random variables $N(0,I)$. Let $a_1\ge \ldots\ge a_d$ as the eigenvalues of $\Sigma$ and $y_1,\ldots, y_d$ follows standard normal distribution. For any $0<\lambda<a_1^{-1}$, 
		\begin{align*}
			\EE\left(e^{\lambda \langle \epsilon_1,\epsilon_2\rangle}\right)=\EE\left(\EE\left(e^{\lambda \langle Z_1,\Sigma Z_2\rangle}|Z_2\right)\right)=\EE\left(e^{\lambda^2 Z_2^T\Sigma^2 Z_2  /2}\right)=\prod_{i}\EE\left(e^{\lambda^2 a_i^2y_i^2  /2}\right)=\prod_{i}{1\over \sqrt{1-\lambda^2a_i^2}}.
		\end{align*}
		Therefore, 
		$$
		\log \left(\EE\left(e^{\lambda \langle \epsilon_1,\epsilon_2\rangle}\right)\right)=-{1\over 2}\sum_{i}\log\left(1-\lambda^2a_i^2\right)\le \sum_{i}{\lambda^2a_i^2\over 2(1-\lambda a_1)}.
		$$
		We can then use the same strategy in the proof of Lemma 1 in \cite{laurent2000adaptive} to obtain 
		$$
		\PP\left(\langle \epsilon_1,\epsilon_2\rangle>a_1t+\sqrt{2t\sum_i a_i^2}\right)\le e^{-t}.
		$$
		Since the distribution of $\langle \epsilon_1,\epsilon_2\rangle$ is symmetrical, we complete the proof. 
	\end{proof}
	
	\begin{lemma}
		\label{lm:ustat}
		Suppose $X_1,\ldots,X_n$ are independent copies of sub-Gaussian vector $X\in \RR^{d_1}$ with parameters $\sigma_1^2$ and $\EE(X)=0$. Let $A$ be a $d_1\times d_2$ matrix with $d_2>d_1$. If we define
		$$
		\tilde{M}={1\over n(n-1)}\sum_{j< j'}\left(AX_j X_{j'}^TA^T+AX_{j'}X_j^TA^T\right),
		$$
		then 
		$$
		\PP\left(\left\|\tilde{M}\right\|\ge C\sigma^2\left({\sqrt{4dt}+4t\over n}+{\sqrt{d(t+\log d)}\over n^{3/2}}+{\sqrt{d}(t+\log d)\over n^2}\right)\right)\le 2\exp(-t),
		$$
		where $\sigma^2=\sigma_1^2\|A\|^2$ and $d=d_1$.
	\end{lemma}
	\begin{proof}
		Since $\hat{M}$ is a $U$-statistic, an application of decoupling technique for $U$-statistic \citep{de2012decoupling} yields
		$$
		\PP(\|\hat{M}\|>t)\le 15 \PP(15\|\tilde{M}\|>t),
		$$
		where $\tilde{M}$ is defined as
		$$
		\tilde{M}={1\over n(n-1)}\sum_{j< j'}\left(AX_j Y_{j'}^TA^T+AY_{j'}X_j^TA^T\right)
		$$
		where $Y_1,\ldots,Y_n$ is an independent copy of $X_1,\ldots,X_n$. Therefore, we only need to focus on the bound for $\PP(\|\tilde{M}\|>t)$. We can obtain the results by applying Lemma~\ref{lm:crossstat}.
	\end{proof}
	
	\begin{lemma}
		\label{lm:crossstat}
		Suppose $X_1,\ldots,X_n$ are independent copies of sub-Gaussian vector $X\in \RR^{d_1}$ with parameters $\sigma_1^2$ and $\EE(X)=0$ and $Y_1,\ldots,Y_n$ are independent copies of sub-Gaussian vector $Y\in \RR^{d_2}$ with parameters $\sigma_2^2$ and $\EE(Y)=0$. Let $d_3>\max(d_2,d_1)$ and $A$ be a $d_1\times d_3$ matrix and $B$ be a $d_2\times d_3$ matrix. If we define
		$$
		\hat{M}={1\over n(n-1)}\sum_{j\ne j'}\left(AX_j Y_{j'}^TB^T+BY_{j'}X_j^TA^T\right),
		$$
		then 
		$$
		\PP\left(\left\|\tilde{M}\right\|\ge C\sigma^2\left({\sqrt{4dt}+4t\over n}+{\sqrt{d'(t+\log d')}\over n^{3/2}}+{\sqrt{d}\log(d/d')(t+\log d')\over n^2}\right)\right)\le 2\exp(-t),
		$$
		where $\sigma^2=\sigma_1\sigma_2\|A\|\|B\|$, $d=\min(d_1,d_2)$ and $d'=\max(d_1,d_2)$. If $d/d'\to 0$, then 
		$$
		\PP\left(\left\|\tilde{M}\right\|\ge C\sigma^2\left({\sqrt{4dt}+4t\over n}+{\sqrt{d'(t+\log d')}\over n^{3/2}}\right)\right)\le 2\exp(-t).
		$$
	\end{lemma}
	\begin{proof}
		By definition, $\hat{M}$ can be written as
		$$
		\tilde{M}={n\over n-1}(A\bar{X}\bar{Y}^TB^T+B\bar{Y}\bar{X}^TA^T)-{1\over n(n-1)}\sum_j (AX_j Y_{j}^TB^T+BY_{j}X_j^TA^T).
		$$
		We now bound the above four terms one by one. The definition of trace shows that 
		$$
		\|A\bar{X}\bar{Y}^TB^T\|={\rm Tr}(A\bar{X}\bar{Y}^TB^T)=\bar{Y}^TB^TA\bar{X}.
		$$
		By Lemma~\ref{lm:crossconcen}, we can show that 
		$$
		\PP\left(n\bar{Y}^TB^TA\bar{X}/\sigma^2>\sqrt{4dt}+4t\right)\le  \exp(-t).
		$$
		Therefore, we can know that
		$$
		\PP\left(\|A\bar{X}\bar{Y}^TB^T\|>\sigma^2{\sqrt{4dt}+4t\over n}\right)\le \exp(-t).
		$$
		With the same argument, we can show that
		$$
		\PP\left(\|B\bar{Y}\bar{X}^TA^T\|>\sigma^2{\sqrt{4dt}+4t\over n}\right)\le \exp(-t).
		$$
		We now move to the last two terms. Lemma~\ref{lm:crossconcen} suggests 
		$$
		\|\|AX_jY_j^TB^T\|\|_{\psi_1}=\|Y_j^TB^TAX_j\|_{\psi_1}\le 2\sqrt{d}\sigma^2.
		$$
		where $\|\cdot\|_{\psi_1}$ is $\psi$-Orlicz norm. Then we can apply matrix Bernstein inequalities in Lemma~\ref{lm:matrixbern}
		$$
		\PP\left(\left\|{1\over n}\sum_{j=1}^nAX_jY_j^TB^T\right\|\ge Cv{\sqrt{t+\log d'}\over n}+C\sqrt{d}\sigma^2\log(d/d'){t+\log d'\over n}\right)\le \exp(-t)
		$$
		where 
		\begin{align*}
			v^2&=\max\left(\left\|\sum_j \EE\left(AX_jY_j^TB^TBY_jX_j^TA^T\right)\right\|,\left\|\sum_j \EE\left(BY_jX_j^TA^TAX_jY_j^TB^T\right)\right\|\right)\\
			&= n\max\left(\left\| \EE\left(\|BY_j\|^2AX_jX_j^TA^T\right)\right\|,\left\| \EE\left(\|AX_j\|^2BY_jY_j^TB^T\right)\right\|\right)\\
			&\le nd'\sigma^4
		\end{align*}	
		Putting these two terms together yields
		$$
		\PP\left(\left\|\tilde{M}\right\|\ge C\sigma^2\left({\sqrt{4dt}+4t\over n}+{\sqrt{d'(t+\log d')}\over n^{3/2}}+{\sqrt{d}\log(d/d')(t+\log d')\over n^2}\right)\right)\le 2\exp(-t).
		$$
		If $d/d'\to 0$, then 
		$$
		\PP\left(\left\|\tilde{M}\right\|\ge C\sigma^2\left({\sqrt{4dt}+4t\over n}+{\sqrt{d'(t+\log d')}\over n^{3/2}}\right)\right)\le 2\exp(-t).
		$$
	\end{proof}
	
	\begin{lemma}
		\label{lm:crossconcen}
		Suppose $X$ and $Y$ are independent $d_1$ and $d_2$-dimensional random vector, i.e. $X\in \RR^{d_1}$ and $Y\in \RR^{d_2}$. $X$ and $Y$ are zero-mean sub-Gaussian vectors with parameters $\sigma_1^2$ and $\sigma_2^2$, i.e.
		$$
		\EE\left(e^{\langle a,X\rangle}\right)\le e^{\sigma_1^2\|a\|^2/2}\qquad {\rm and}\qquad \EE\left(e^{\langle b,Y\rangle}\right)\le e^{\sigma_2^2\|b\|^2/2}
		$$
		for any vector $a\in \RR^{d_1}$ and $b\in \RR^{d_2}$.
		Let $A$ a $d_1\times d_2$ matrix, i.e. $A\in \RR^{d_1\times d_2}$. Then,
		$$
		\EE(e^{\lambda X^TAY})\le \left(1\over 1-\lambda^2\sigma^4\right)^{d/2}\qquad {\rm and}\qquad \PP\left(X^TAY/\sigma^2>\sqrt{4dt}+4t\right)\le \exp(-t),
		$$
		where $\sigma^2=\sigma_1\sigma_2\|A\|$ and $d=\min(d_1,d_2)$.
	\end{lemma}
	\begin{proof}
		Without loss of generality, we assume $d_2\le d_1$ so $d=d_2$. 
		When $\sigma_1^2\|A\|^2\lambda^2<1/\sigma_2^2$, we have 
		$$
		\EE\left(e^{\lambda X^TAY}\right)=\EE\left(\EE(e^{\lambda X^TAY}|Y)\right)\le \EE\left(e^{\sigma_1^2\lambda^2\|AY\|^2/2}\right)\le \EE\left(e^{\sigma_1^2\lambda^2\|A\|^2\|Y\|^2/2}\right) \le \EE\left(e^{\sigma^4\lambda^2\|Z\|^2/2}\right),
		$$
		where $Z$ is a $d$-dimensional vector of independent standard Gaussian random variable. The last inequality is due to \cite{hsu2012tail}. Since $Z_i$ are independent from each other, 
		$$
		\EE\left(e^{\sigma^4\lambda^2\|Z\|^2/2}\right)=\prod_{i=1}^d \EE\left(e^{\sigma^4\lambda^2\|Z_i\|^2/2}\right)=\left(1\over 1-\lambda^2\sigma^4\right)^{d/2}
		$$
		An application of Chernoff bound suggests that
		$$
		\PP\left(X^TAY>t\right)\le {\EE(e^{\lambda X^TAY})\over e^{\lambda t}}\le \exp\left(-\lambda t+\lambda^2\sigma^4 d\right).
		$$
		By choosing $\lambda=\min(t/2\sigma^4d,1/\sigma^2)$, we can have 
		$$
		\PP\left(X^TAY>t\right)\le \exp\left(-\min\left({t^2\over 4\sigma^4d},{t\over 4\sigma^2}\right)\right).
		$$
	\end{proof}

	\begin{lemma}[\cite{tropp2012user,minsker2017some,xia2021statistically}]
		\label{lm:matrixbern}
		Let $\bX_1,\ldots, \bX_n\in \RR^{d_1\times d_2}$ be random matrix with zero mean. Suppose that $\max_i \|\|\bX_i\|\|_{\psi_\alpha}\le U_\alpha$ for some $\alpha>1$ where $\|\cdot\|_{\psi_\alpha}$ is $\psi$-Orlicz norm. Then there exists a universal constant $C>0$ such that 
		$$
		\PP\left(\left\|\sum_{i=1}^n\bX_i\over n\right\|\ge C\left(v{\sqrt{t+\log(d_1+d_2)}\over n}+U_\alpha\log\left(\sqrt{n}U_\alpha\over v\right){t+\log(d_1+d_2)\over n}\right)\right)\le \exp(-t),
		$$
		where
		$$
		v^2=\max\left(\|\sum_{i=1}^n\EE(\bX_i\bX_i^T)\|,\|\sum_{i=1}^n\EE(\bX_i^T\bX_i)\|\right).
		$$
	\end{lemma}
	
	\begin{lemma}[\cite{chen2020spectral}]
		\label{lm:trunmatrixbern}
		Let $\bX_1,\ldots, \bX_n\in \RR^{d_1\times d_2}$ be random matrix with zero mean. Suppose that there is a constant $L$ such that 
		$$
		\PP(\|\bX_i\|\ge L)\le q_0\quad {\rm and}\quad \|\EE(\bX_i\bI(\|\bX_i\|\ge L) )\|\le q_1
		$$
		for some number $q_0$ and $q_1$. If $v$ is defined in the same way as Lemma~\ref{lm:matrixbern}, for all $t\ge nq_1$, then
		$$
		\PP\left(\left\|\sum_{i=1}^n\bX_i\right\|\ge t\right)\le (d_1+d_2)\exp\left(-(t-nq_1)^2/2\over v+L(t-nq_1)/3\right)+nq_0.
		$$
	\end{lemma}
	
	\begin{lemma}
		\label{lm:knnlowerbayes}
		We follow the notations in Lemma~\ref{lm:knnlower}. Suppose $\eta(x)$ is drawn from a prior distribution $f_\eta$ which is defined on a collection of possible $\eta(x)$. Given distance $D$ and prior $f_\eta$, the Bayes high error sets are defined as 
		$$
		\Ecal_{f_\eta}^+(\delta,\nu)=\left\{x:\PP_{\eta\sim f_\eta}\left(\eta(x)\ge{1\over 2}+\delta; \eta(\Bcal_{D}(x,r))\le {1\over 2},\ \forall\ r_\nu\le r\le r_{2\nu} \right)>c_1 \right\}
		$$
		and 
		$$
		\Ecal_{f_\eta}^-(\delta,\nu)=\left\{x:\PP_{\eta\sim f_\eta}\left(\eta(x)\le {1\over 2}-\delta;\eta(\Bcal_{D}(x,r))\ge {1\over 2},\ \forall\ r_\nu\le r\le r_{2\nu}\right)>c_1 \right\}.
		$$
		The Bayes risk of $r(D)$ is then lower bounded by 
		$$
		\EE_{\eta\sim f_\eta}\left(r(D)\right)\ge {c_0\delta}\mu\left(\Ecal_{f_\eta}^+\left(\delta,{k\over s}\right)\cup \Ecal_{f_\eta}^-\left(\delta,{k\over s}\right)\right)
		$$
		when we observe $s$ samples and choose $k\ge 9$ in $k$-NN.
	\end{lemma}
	\begin{proof}
		In the Bayes risk, we can think of $\eta(x)$ as a random variable and have
		$$
		\EE_{\eta\sim f_\eta}\left(r(D)\right)=\EE_{X\sim\mu}\left(\EE_{\eta\sim f_\eta}\left(|2\eta(X)-1|\PP(\hat{f}_{D}(X)\ne f^\ast(X))\right)\right).
		$$
		If $x\in \Ecal_{f_\eta}^+\left(\delta,{k/s}\right)$, we can define a event 
		$$
		\Acal_x=\left\{\eta:\eta(x)\ge{1\over 2}+\delta; \eta(\Bcal_{D}(x,r))\le {1\over 2},\ \forall\ r_\nu\le r\le r_{2\nu} \right\}.
		$$
		If $x\in \Ecal_{f_\eta}^+\left(\delta,{k/s}\right)$ and $\eta\in \Acal_x$, we can follow exactly the same analysis in proof of Lemma~\ref{lm:knnlower} to obtain
		$$
		\PP(\hat{f}_{D}(x)\ne f^\ast(x))\ge c_2
		$$
		for some constant $c_2$. This immediately suggests 
		$$
		\EE_{\eta\sim f_\eta}\left(|2\eta(x)-1|\PP(\hat{f}_{D}(x)\ne f^\ast(x))\right)\ge 2c_1c_2\delta
		$$
		when $x\in \Ecal_{f_\eta}^+\left(\delta,{k/s}\right)$. We can prove a similar conclusion if $x\in \Ecal_{f_\eta}^-\left(\delta,{k/s}\right)$. Therefore, we can conclude that 
		$$
		\EE_{\eta\sim f_\eta}\left(r(D)\right)\ge {2c_1c_2\delta}\mu\left(\Ecal_{f_\eta}^+\left(\delta,{k\over s}\right)\cup \Ecal_{f_\eta}^-\left(\delta,{k\over s}\right)\right).
		$$
	\end{proof}
	
	\begin{lemma}
		\label{lm:knnlower}
		Give a set $A$ such that $\mu(A)>0$, then we define the average of function $\eta(x)$ within set $A$ as
		$$
		\eta(A)={1\over \mu(A)}\int_A \eta(x)d\mu(x).
		$$
		For a given distance $D$, we define the high error sets as 
		$$
		\Ecal^+(\delta_1,\delta_2,\nu)=\left\{x:\eta(x)\ge{1\over 2}+\delta_1; \eta(\Bcal_{D}(x,r))\le {1\over 2}+\delta_2,\ \forall\ r_\nu\le r\le r_{2\nu} \right\}
		$$
		and 
		$$
		\Ecal^-(\delta_1,\delta_2,\nu)=\left\{x:\eta(x)\le {1\over 2}-\delta_1;\eta(\Bcal_{D}(x,r))\ge {1\over 2}-\delta_2,\ \forall\ r_\nu\le r\le r_{2\nu} \right\},
		$$
		where $r_v$ is the radius such that $\mu(\Bcal_{D}(x,r_\nu))=\nu$. 
		If we observe $s$ samples and choose $k\ge 9$ in $k$-NN, then there exists $c_0$ such that
		$$
		r(D)\ge {2c_0\delta}\mu\left(\Ecal^+\left(\delta,{1\over \sqrt{k}},{k\over s}\right)\cup \Ecal^-\left(\delta,{1\over \sqrt{k}},{k\over s}\right)\right).
		$$
	\end{lemma}
	\begin{proof}
		The proof of this lemma can be seen as a generalization of proof in Theorem 3 of \cite{chaudhuri2014rates}. 
		If $x\in \Ecal^+\left(\delta,1/\sqrt{k},{k/s}\right)$, then an application of Chernoff bound suggests 
		$$
		\PP\left(X_{(k+1)}\in \Bcal_{D}(x,r_{2k/s}), X_{(k+1)}\notin \Bcal_{D}(x,r_{k/s}) \right)\ge {1\over 2}(1-e^{-k/4}). 
		$$
		According to Chapter 1.2 in \cite{biau2015lectures}, $\mu(\Bcal_D(x,D(x,X_{(1)}))),\ldots, \mu(\Bcal_D(x,D(x,X_{(s)})))$ follow the same distribution as $U_{(1)},\ldots, U_{(s)}$ which are the order statistics of i.i.d. $[0,1]$ uniform random variables $U_1,\ldots, U_s$.  Conditioned on $U_{(k+1)}$, $U_{(1)},\ldots, U_{(k)}$ can be seen as order statistics of i.i.d. $[0,U_{(k+1)}]$ uniform random variables \citep[Chapter 5,][]{ahsanullah2013introduction}. In other words, conditioned on $D(x,X_{(k+1)})=r$,  $X_{(1)},\ldots, X_{(k)}$ can be seen randomly drawn from $\Bcal_D(x,r)$. Thus, $k\hat{\eta}(x)$ follows the same distribution as binomial random variable with parameter $k$ and $\mu(\Bcal_D(x,r))$. When $r_{k/s}\le r\le r_{2k/s}$, we have $\mu(\Bcal_D(x,r))\le 1/2+1/\sqrt{k}$ and
		$$
		\PP\left(\hat{\eta}(x)<{1\over 2}\middle|D(x,X_{(k+1)})=r\right)\ge \PP\left({\rm Bin}(k, 1/2+1/\sqrt{k})<{k\over 2}\right),
		$$
		where ${\rm Bin}(k, 1/2+1/\sqrt{k})$ is a random variable following binomial distribution with parameter $k$ and $1/2+1/\sqrt{k}$. We then apply Theorem 2.1 in \cite{slud1977distribution} to obtain 
		$$
		\PP\left(\hat{\eta}(x)<{1\over 2}\middle|D(x,X_{(k+1)})=r\right)\ge 1-\Phi\left(1\over 1/4-1/k\right)
		$$
		where $\Phi$ is the cumulative distribution function of standard normal distribution. This suggests that if $x\in \Ecal^+\left(\delta,1/\sqrt{k},{k/s}\right)$, there exists a constant $c_0$ such that
		\begin{equation}
			\label{eq:lowerprob}
			\PP\left(\hat{\eta}(x)<{1\over 2}\right)\ge c_0.
		\end{equation}
		We can show the similar conclusion if $x\in \Ecal^-\left(\delta,1/\sqrt{k},{k/s}\right)$. By definition, we can show 
		\begin{align*}
			r(D)&=\EE\left(|2\eta(X)-1|\bI(\hat{f}_{D}(X)\ne f^\ast(X))\right)\\
			&\ge \int_{\Ecal^+\left(\delta,1/\sqrt{k},{k/s}\right)\cup \Ecal^-\left(\delta,1/\sqrt{k},{k/s}\right)} |2\eta(x)-1|\PP(\hat{f}_{D}(x)\ne f^\ast(x))d\mu(x)\\
			&\ge {2c_0\delta}\int_{\Ecal^+\left(\delta,1/\sqrt{k},{k/s}\right)\cup \Ecal^-\left(\delta,1/\sqrt{k},{k/s}\right)} d\mu(x).
		\end{align*}
		We now complete the proof. 
	\end{proof}
	
	\begin{lemma}
		\label{lm:kmeansevent}
		If we follow all the notations in the proof of Theorem~\ref{thm:kmeans}, then 
		$$
		\PP\left(\Acal_{1,+}\cap\Acal_{1,-}\cap \Acal_{2}\cap \Acal_{3}\cap \Acal_{4}\right)\ge 1-s^{-5}.
		$$
	\end{lemma}
	\begin{proof}
		We start the proof from events $\Acal_{1,+}$
		$$
		\Acal_{1,+}=\left\{\left\|\bar{\epsilon}_+-\bar{\epsilon}_+^\ast\right\|^2\le 3{h(\hat{Y},Y)\over \min(s_+,\tilde{s}_+)}\left({\rm Tr}(\Sigma_{\pm})+\sqrt{{\rm Tr}(\Sigma_{\pm}^2)s}+\|\Sigma_{\pm}\|s\right)\quad \forall \ \hat{Y}_i\right\}.
		$$
		Since $\bar{\epsilon}_+$ and $\bar{\epsilon}_+^\ast$ follow normal distribution, we look at the mean and variance of $\bar{\epsilon}_+-\bar{\epsilon}_+^\ast$. Clearly,
		$$
		\EE(\bar{\epsilon}_+-\bar{\epsilon}_+^\ast)=0\qquad {\rm and}\qquad {\rm Var}(\bar{\epsilon}_+-\bar{\epsilon}_+^\ast)=\gamma \Sigma_{\pm},
		$$
		where $\gamma\le 2{h(\hat{Y},Y)/\min(s_+,\tilde{s}_+)}$. We can then apply Lemma 1 in \cite{laurent2000adaptive} to obtain 
		$$
		\PP\left(\left\|\bar{\epsilon}_+-\bar{\epsilon}_+^\ast\right\|^2\ge 3{h(\hat{Y},Y)\over \min(s_+,\tilde{s}_+)}\left({\rm Tr}(\Sigma_{\pm})+\sqrt{{\rm Tr}(\Sigma_{\pm}^2)s}+\|\Sigma_{\pm}\|s\right)\right)\le e^{-3s/2}
		$$
		By union bound, we can show
		$$
		\PP(\Acal_{1,+})\ge 1-2^se^{-3s/2}\ge 1-s^{-6}.
		$$
		Similarly, we can also show $\PP(\Acal_{1,-})\ge 1-s^{-6}$. For $\Acal_2$, we only need to apply Lemma 1 in \cite{laurent2000adaptive} without union bound to show $\PP(\Acal_2)\ge 1-s^{-6}$. We to apply Theorem 5.39 in \cite{vershynin2010introduction} to show that $\PP(\Acal_3)\ge 1-s^{-6}$. Finally, we can apply Hoeffding's inequality to show $\PP(\Acal_4)\ge 1-s^{-6}$ since $s_+$ follow a binomial distribution. An application of union bound for $\Acal_{1,+}$, $\Acal_{1,-}$, $\Acal_2$, $\Acal_3$, and $\Acal_4$ can complete the proof.
	\end{proof}

\end{appendices}
\end{document}